\documentclass{article} 
\usepackage{authblk}
\usepackage[top=30truemm,bottom=30truemm,left=30truemm,right=30truemm]{geometry}

\usepackage{color}
\usepackage{bm}
\usepackage[hidelinks]{hyperref}
\usepackage{graphicx} 
\usepackage{amsmath,amssymb,amsthm}	
\usepackage{multicol}
\usepackage[hang,small,bf]{caption}
\usepackage[subrefformat=parens]{subcaption}
\usepackage{algorithm,tabularx}
\usepackage{algpseudocode}
\newtheorem{theorem}{Theorem}

\newtheorem{definition}[theorem]{Definition}
\newtheorem{lemma}[theorem]{Lemma}

\newtheorem{assumption}{Assumption}
\newtheorem{example}{Example}
\newtheorem{remark}{Remark}

\def\x{{\bm x}}
\def\y{{\bm y}}
\def\z{{\bm z}}
\def\e{{\bm e}}
\def\t{{\bm t}}
\def\Rbb{\mathbb{R}}
\def\rmd{\mathrm{d}}
\newcommand{\argmin}{\mathop{\mathrm{argmin}}}
\newcommand{\argmax}{\mathop{\mathrm{argmax}}}

\usepackage[round]{natbib}

\title{Scaling-based Data Augmentation for Generative Models \\ and its Theoretical Extension}

\author{
Yoshitaka Koike\thanks{Science Tokyo}, \
Takumi Nakagawa\thanks{Science Tokyo/RIKEN}, \ 
Hiroki Waida\thanks{Science Tokyo}, \ 
Takafumi Kanamori\thanks{Science Tokyo/RIKEN. (kanamori@c.titech.ac.jp)}
}

\date{}

\begin{document}
\maketitle



\begin{abstract}
This paper studies stable learning methods for generative models that enable high-quality data generation. Noise injection is commonly used to stabilize learning. However, selecting a suitable noise distribution is challenging. Diffusion-GAN, a recently developed method, addresses this by using the diffusion process with a timestep-dependent discriminator. We investigate Diffusion-GAN and reveal that data scaling is a key component for stable learning and high-quality data generation. Building on our findings, we propose a learning algorithm, Scale-GAN, that uses data scaling and variance-based regularization. Furthermore, we theoretically prove that data scaling controls the bias-variance trade-off of the estimation error bound. As a theoretical extension, we consider GAN with invertible data augmentations. Comparative evaluations on benchmark datasets demonstrate the effectiveness of our method in improving stability and accuracy. 
\end{abstract}

\section{Introduction}
\label{sec:Introduction}

Generative adversarial networks (GANs) and their numerous variants can generate high-quality samples, 
such as images, audio, and graphs~\citep{
goodfellow2014generative,arjovsky2017wasserstein,gulrajani2017improved,%
mao2017least,radford2015unsupervised,miyato2018spectral,brock2018large,%
zhang2019self,li2017mmd,karras2019style,karras2020analyzing,%
karras2021alias,sauer2021projected,wang2022diffusion,van2016wavenet,de2018molgan}. 
Although learning with diffusion models has received much attention in recent years~\citep{ho2020denoising,song2020score}, generative models properly trained with GANs are still superior in terms of the quality and computation efficiency of data generation~\citep{xiao2021tackling,sauer2022stylegan}.
 However, GANs have stability issues that have not yet been fully resolved, such as lack of convergence, mode collapse, and catastrophic forgetting~\citep{thanh2020catastrophic}, in addition to overfitting, a common problem in machine learning. 

Many studies have been conducted to alleviate these problems
 including modification of loss function, noise injection, 
adaptive instance normalization~\citep{huang2017arbitrary}, 
weight normalization~\citep{ioffe2015batch, ba2016layer, wu2018group}, 
Jacobian regularization~\citep{mescheder2017numerics, nagarajan2017gradient, nie2020towards}, 
and utilization of pre-trained diffusion-models~\citep{luo2024diff,pmlr-v235-xia24d}. 
To avoid overfitting, the regularization to the discriminator's 
gradient is widely exploited~\citep{arjovsky2017wasserstein,gulrajani2017improved,miyato2018spectral,kodali2017convergence,adler2018banach,petzka2018regularization,xu2021towards,mescheder2018training,zhou2019lipschitz,thanh2018improving}. 
Details of related works are summarized in 
Section~\ref{appendix:sec:Related_Works}. 

Noise injection is another popular approach for stabilizing GANs~\citep{roth2017stabilizing}. 
The instability occurs when the support of the distributions is far apart. 
By bringing the true training data and fake generated data closer together by noise injection, 
the discriminator's overfitting can be alleviated. 
As summarized in Section~\ref{appendix:relatedworks:fusion}, 
Diffusion-GAN~\citep{wang2022diffusion}, inspired by the Denoising Diffusion Probabilistic Model (DDPM)~\citep{ho2020denoising,song2020score}, showed that the addition of noise contributes to stable learning and improves accuracy for GAN. 
As DDPM does, Diffusion-GAN diffuses the data by scaling toward the origin and noise injection, i.e., diffusion process. 
Furthermore, the strength of the data scaling is adaptively updated according to the degree of overfitting of the discriminator. 
While this strategy is powerful as a learning stabilization, 
it is unclear which component in Diffusion-GAN contributes to improving the data generation performance. 
Besides the diffusion-GAN, the fusion of GANs and diffusion models have been considered 
to improve the efficiency, speed, and quality of generative models~\cite{zheng2022truncated,yin2024one,%
sauer2023adversarial,kim2023consistency,yin2024improved}.

This paper studies stable learning methods for generative models.
For that purpose, we deconstruct the components of Diffusion-GAN and analyze the significance of each. 
We conduct numerical simulations to identify critical factors for stable learning and high-quality sampling. 
We also investigate regularization to stabilize the learning process. 
Using the findings from the above discussion, we propose the learning algorithm called Scale-GAN, which uses data scaling and variance-based regularization. 
Our empirical studies on benchmark datasets in image generation demonstrate that the proposed method performs superior to existing methods. 
The contributions of this paper are summarized below.
\begin{itemize}
    \item 
 We reveal that data scaling contributes more significantly to stabilization than noise injection. 
	   Indeed, data scaling helps avoid mode collapse. Compared to data scaling without noise injection, 
	   the data diffusion with noise makes it harder for the discriminator to convey an effective gradient direction to the generator. 
\item In GAN-based learning methods, 
  the discrepancy between true training data and fake generated data, and the discriminator's gradient, is often incorporated into the loss function. 
  We introduce a regularization method that leverages the discriminator's variance as an approximation of the regularization for the gradient, offering a simpler yet effective approach. 
\item We analyze theoretical properties of the proposed method, including i) the invariance of gradient
  direction to the scaling intensity, ii) the influence of the variance-based regularization on the generator's distribution, 
  and iii) the relationship between the scaling strategy and the generalization performance. 
\end{itemize}

\section{A Brief Survey of Diffusion-GAN}
\label{sec:Diffusion-GAN}
Diffusion-GAN~\citep{wang2022diffusion} is a learning method for generative models inspired by the success of diffusion models and revisits the use of instance noise in GANs. 
Similar to vanilla GAN~\citep{goodfellow2014generative},
the learning of Diffusion-GAN is formulated as the following min-max optimization problem, 
\begin{align*}
\min_{\theta}\max_{\phi}   \mathbb{E}_{\x,t}\mathbb{E}_{\y} [\,\log(D_\phi(\boldsymbol{y}, t))\mid \x,t\,]
 + \mathbb{E}_{\z,t}\mathbb{E}_{\y} [\,\log (1-D_\phi(\boldsymbol{y}, t))\mid G_\theta(\z),t\,], 
\end{align*}
where $G_\theta$ is a generator and $D_\phi$ is a discriminator depending on the scaling intensity $t$. 
Differently from the vanilla GAN, 
both training and generated samples are diffused with a conditional multivariate normal distribution 
$\mathcal{N}\left(\boldsymbol{y} ; s_t \boldsymbol{x},(1-s_t^2)\sigma^2 \boldsymbol{I}\right)$ 
with mean $s_t\boldsymbol{x}$ and the variance-covariance matrix $(1-s_t^2)\sigma^2\boldsymbol{I}$, 
The scaling function $s_t$ determines the diffusion schedule, 
and various scheduling schemes have been devised for the diffusion model. 
As the distribution of the scaling intensity $t$, 
the uniform distribution $\pi(t)=1/T$, or the priority distribution $\pi(t)\propto t$ for $t\in\{0,1,\ldots,T\}$ is often used. 
In practice, a fixed $T$ is not recommended. 
Instead, \cite{karras2020training} and \cite{wang2022diffusion} proposed an adaptive update rule for $T$ 
according to how much the discriminator overfits to the training data. 

As shown above, Diffusion-GAN consists of 
 three main components: i) data scaling, ii) noise injection, and iii) scaling strategy consisting of $s_t$ and $\pi(t)$. 
In Section~\ref{sec:Effect_of_Data_Scaling_and_NoiseInjection_in_GANs}, 
we deconstruct each component of Diffusion-GAN and analyze how each contributes to learning stability and data-generation quality.

\section{Data Scaling and Noise Injection}
\label{sec:Effect_of_Data_Scaling_and_NoiseInjection_in_GANs}

The simulation setting is the following. 
Suppose that i.i.d. samples, ${\x}_1,\ldots,{\x}_{n}, n=80$, are generated from the two-dimensional Gaussian mixture distribution with eight components, $\sum_{k=1}^{8}\frac{1}{8}\mathcal{N}({\bm\mu}_k, \sigma^2 I_2)$, 
where ${\bm\mu}_k=(\cos(2\pi k/8),\sin(2\pi k/8)), k=1,\ldots,8$ and $\sigma^2=0.05$. 
Four-layer fully connected neural networks with LeakyReLU activation function are used for the discriminator and generator. 
The LeakyReLU is used in StyleGAN2~\citep{karras2020analyzing}. 
The findings in this section will also be confirmed for the standard benchmark datasets in Section~\ref{Numerical_Experiments}. 

\subsection{Data Scaling}
\label{sec:DataScaling}

We investigate the effect of the data scaling. Scale the data by $s=0.25, 0.5, 1, 1.5$ respectively, and train the scaled generator $s G_\theta$ and the discriminator $D_\phi$ having the learning parameters $\theta$ and $\phi$ using the scaled data $\{\y_i\}_{i=1}^{80}$ for $\y_i=s\x_i$. 
Since applying data scaling only to the training data causes a bias in the generator, it also applies to the generated data, which has been studied by~\cite{wang2022diffusion,jenni2019stabilizing,tran2021data,zhao2020differentiable,karras2020training}. 
The data scaling does not apply to $D_\phi$. 

In Fig~\ref{pr_rc_original_diff_scale}, (a) and (b) show precision and recall to training iterations for each scaling. 
Additional results with different seeds are reported in Section~\ref{appendix_DataScaling_NoiseInjection_GAN}.
The learning with $s=0.25$ shows that if the data distribution is well covered by the generator, i.e., high recall is attained in the early stages of learning, the precision gradually increases, and the learning is successful.
As shown in the lower panels of Fig~\ref{pr_rc_original_diff_scale}, 
however, it has not escaped from the initial learning failure.  
These results can be explained by the size of the gradient for the discriminator, as shown in Fig~\ref{pr_rc_original_diff_scale} (c). 
The gradient norm for $s=0.25$ is smaller than the other cases, meaning that abrupt changes in the discriminator rarely occur.  
Hence, the learning result will be significantly affected by the early stages of the learning process. 
On the other hand, for $s=1, 1.5$, high recall is achieved in the early stages of learning. 
However, it is observed that the recall drops significantly 
from the middle to the latter half of the learning period, causing mode collapse. 
Existing studies~\citep{gulrajani2017improved,kodali2017convergence,adler2018banach,petzka2018regularization,xu2021towards,mescheder2018training,zhou2019lipschitz,thanh2018improving} have shown that a large gradient of the discriminator tends to cause instability. 

Fig~\ref{1730_scale1.5_oscillation}, which visualizes the discriminator's output at $s=1.5$, 
illustrates an unstable situation as the mode oscillates to higher predicted values at each iteration. 
This phenomenon is a precursor to mode decay. 
This oscillation leads to mode decay, known as catastrophic forgetting. 
We observed that the mode collapses occur even for $s=1.0$. 

In summary, data scaling has a significant impact on learning dynamics. 
When the scale is small, the discriminator is less prone to sudden gradients and instability but is susceptible to early status in learning. 
Conversely, 
large scales make it easier to learn discriminant boundaries but are prone to instability before generating accurate samples.
These results suggest that learning data at several scales simultaneously can reflect the good characteristics of each.

\begin{figure}[t]
\centering 
 \begin{tabular}{ccc}  
 \includegraphics[height=45mm]{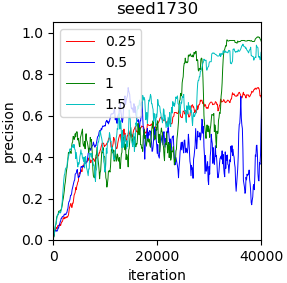} &
 \includegraphics[height=45mm]{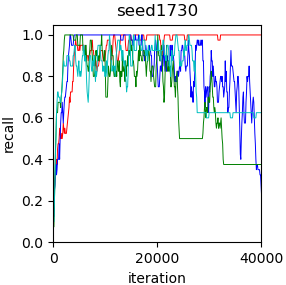}    &
 \includegraphics[height=45mm]{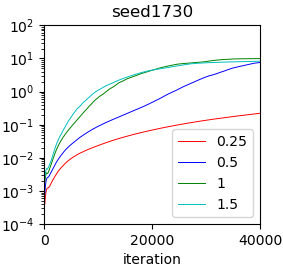}    \\
 \includegraphics[height=45mm]{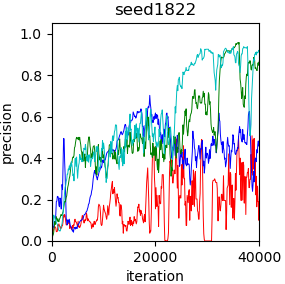} &
 \includegraphics[height=45mm]{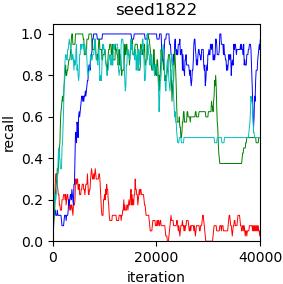}    &
 \includegraphics[height=45mm]{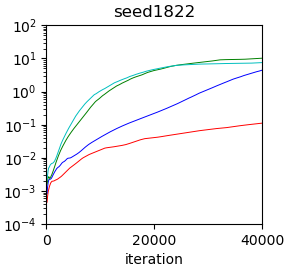}    \\
 {\small\begin{tabular}{l}(a) precision \end{tabular}} & 
 {\small\begin{tabular}{l}\ \ (b) recall \end{tabular}}  &  
 {\small\begin{tabular}{l}\hspace{-2mm}(c) gradient norm \end{tabular}}
 \end{tabular}
 \caption{For each data scaling, 
 (a)\,precision, (b)\,recall, and (c)\,averaged norm of discriminator's gradient 
 are depicted. 
 The upper and lower panels correspond to two different seeds. 
 }
 \label{pr_rc_original_diff_scale}
\end{figure}

\begin{figure}[t]
\centering 
\begin{tabular}{c}
 \includegraphics[height=40mm]{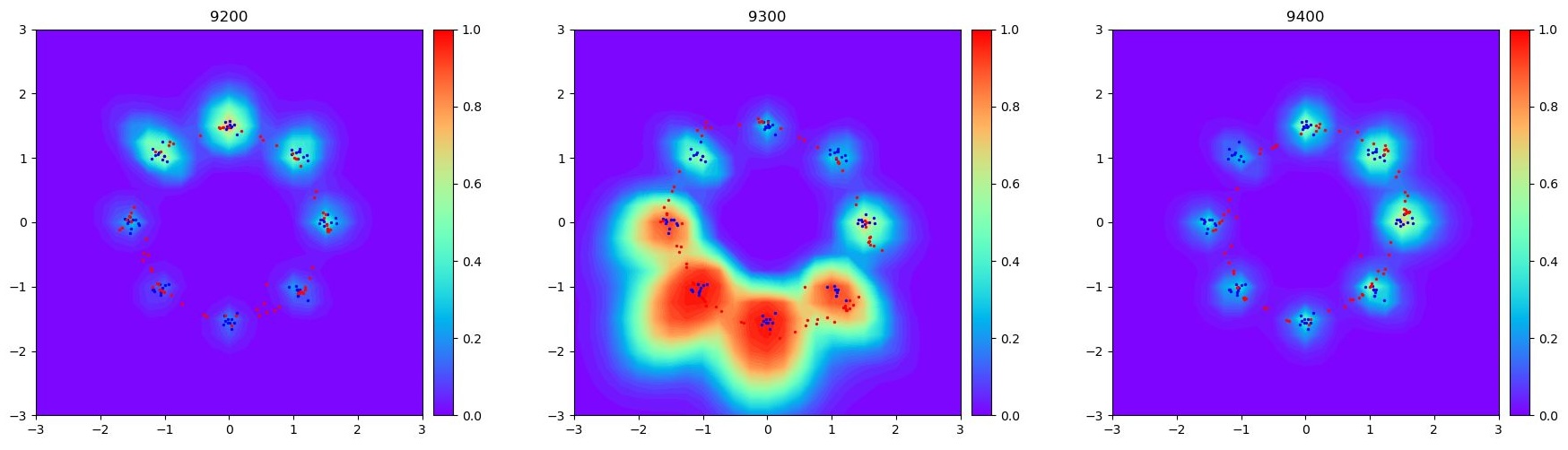}\\
 \includegraphics[height=40mm]{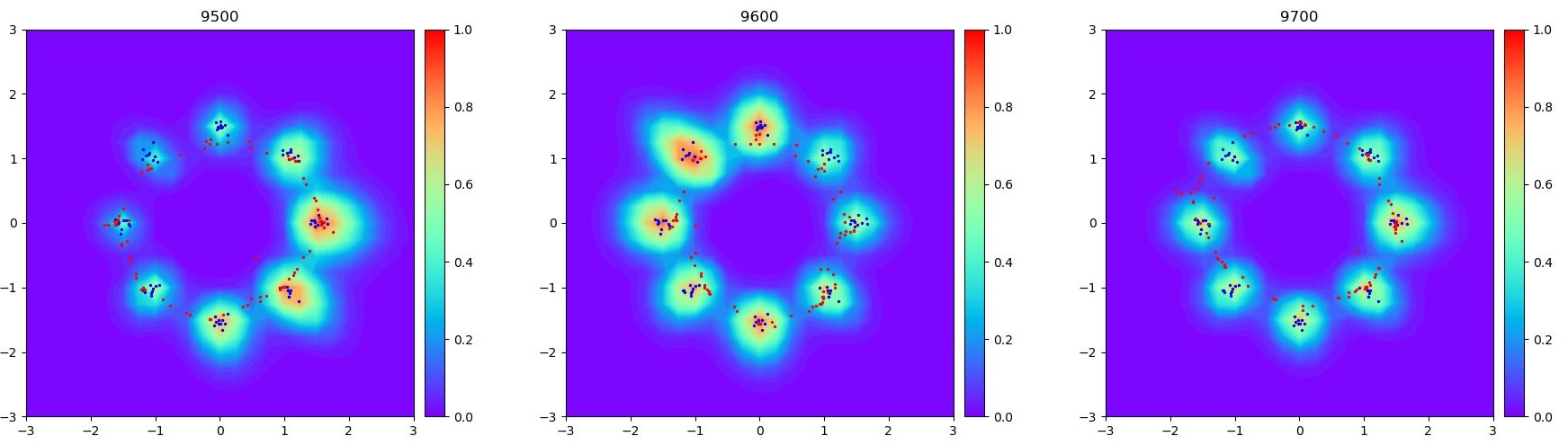}
\end{tabular}
  \caption{
  The discriminator's predictions for $s=1.5$. The number of repetitions ranges from 9200 to 9700 from top left to bottom right. 
  Blue (resp. Red) dots represent the data (resp. generated samples). 
  Modes with high values change in an oscillatory manner, and the oscillations become more intense, leading to mode collapse. 
  }
  \label{1730_scale1.5_oscillation} 
\end{figure}

\begin{figure}[t]
\centering 
\begin{tabular}{cc}
 \includegraphics[height=55mm]{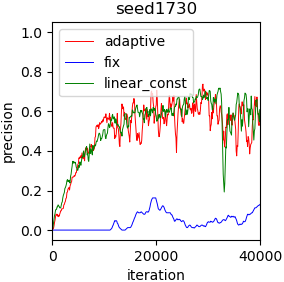} &
 \includegraphics[height=55mm]{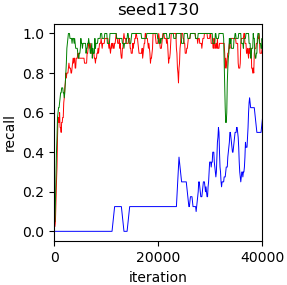}   \\
 (a) precision & (b) recall
\end{tabular}
 \caption{ (a)\,precision, and (b)\,recall 
 for each scaling strategy, ``fix'', ``linear const'', and ``adaptive''. }
 \label{pr_rc_diff_t_strategy}
\end{figure}

\subsection{Scaling Strategy}
\label{subsec:Scaling_Strategy}
Let us consider learning at multiple scales simultaneously. 
In a similar way as \cite{ho2020denoising}, we define the scaling function $s_t$ by 
$s_0=1$ and $s_t=s_{t-1}\sqrt{1-\beta_0(1-t/T)-\beta_T t/T}$ 
for $t\in\{1,\ldots,T\}$. 
The parameters $\beta_0$ and $\beta_T$ control the decay rate and the scale at $T$. 
The distribution $\pi(t)$ of $t$ is given by $\pi(t)=\frac{1}{2}\pi_0(t)+\frac{1}{2}\delta_0(t)$, 
where $\delta_0$ is the point mass distribution at $t=0$, 
and $\pi_0$ is a pre-defined distribution on $\{1,\ldots,T\}$. 
Determine the distribution $\pi_0$ in three different ways and compare the properties of each; 
\begin{itemize}
 \item[i)] $T$ is fixed to $T_{\max}$ and $\pi_0$ is the uniform distribution on $\{1,\ldots,T_{\max}\}$
	 (denoted by ``fix''). 
 \item[ii)] $T$ increases in proportion to the learning iteration $i$ and make it constant from the middle (denoted by ``linear const''), 
	  i.e., $\pi_0$ is the uniform distribution on $\{1,\ldots,T(i)\}$ with  $T(i) = \min\left(\frac{2T_{\text{max}}}{I}i,T_{\text{max}}\right)$. 
 \item[iii)]
	   $T$ adaptively changes by the following update rule, 
	   $T\leftarrow T+\operatorname{sign}(r_d-d_{\text {target }})$, 
	   for $r_d=\mathbb{E}[\operatorname{sign}(D(\y, t)-0.5)]$,
	   and $\pi_0$ is the uniform distribution on $\{1,\ldots,T\}$ (denoted by ``adaptive''). 
\end{itemize}
We use the discriminator $D(\boldsymbol{y}, t)$ that takes not only the data but also the scaling intensity $t$. 
The expectation in $r_d$ 
is taken for the joint distribution of $\y$ and $t$. 
In the adaptive strategy, $T$ is clipped to $[T_{\min}, T_{\max}]$. 
In Diffusion-GAN, the adaptive strategy with 
the priority distribution $\pi_0(t)\propto t$ is used. We train with $T_{\text{min}} = 0, T_{\text{max}} = 500, I=40000$ and $d_{\text{target}} =0.1$. 
The scaling function $s_t$ is determined by $\beta_0=0.0001$ and $\beta_T=0.02$. 
Some hyperparameters in Diffusion-GAN are summarized in Section~\ref{appendix:Hyperparameters_Image_Generation}. 

Panels (a) and (b) in Fig~\ref{pr_rc_diff_t_strategy} show the precision and recall for each scaling strategy. 
Additional results with different seeds are reported in Section~\ref{appendix_DataScaling_NoiseInjection_GAN}. 
The ``fix'' strategy has low precision and does not learn well compared to the others. 
Both the ``linear const'' and ``adaptive'' strategies can roughly estimate distributions from an early stage and learn well. 
We see that learning with multiple scales is stable. 
The ``adaptive'' strategy is less dependent on the hyperparameters and can be used more universally~\citep{wang2022diffusion}.

\subsection{Noise Injection and Data Scaling}
\label{sec:Noise_Injection_Data_Scaling}
We numerically investigate the effectiveness of the noise injection and its relation to data scaling. 
At the beginning, let us consider the noise injection defined by 
$\tilde{\x} = \x + \bm{\epsilon}, \bm{\epsilon} \sim \mathcal{N}(\bm{0}, \sigma_{\text{noise}}^2I)$. 
We compare the learning with $\sigma_{\text{noise}}=0$, i.e., the vanilla GAN, and the learning with positive $\sigma_{\text{noise}}$. 
Fig~\ref{1730colormap}~(a) 
shows the training data $\{\tilde{\x}\}_{i=1}^{80}$, 
generated samples, and the discriminator's outputs. 
As shown in the upper panels, the vanilla GAN learns faster, 
and the discriminator tends to have a large gradient. 
On the other hand, learning with noise injection yields a relatively smooth discriminator surface, meaning that mode collapse is mitigated. 
The precision and recall for each $\sigma_{\text{noise}}$ are shown in Fig~\ref{appendix:noise_scale_adding_noise_supp}.
We observe that the noise injection tends to stabilize the learning process compared to using the original training data. However, occasionally unstable behavior occurs even for noise injection. 

In Diffusion-GAN, both the data scaling and noise injection are incorporated, i.e., the training data $\boldsymbol{x}$ is transformed to 
$\tilde{\x} = s_t \x + \sqrt{1-s_t^2}\bm{\epsilon},\, \bm{\epsilon}\sim\mathcal{N}(\bm{0}, \sigma_{\text{noise}}^2I)$. 
The scaling function  $s_t$ is defined in the same way as that in Section~\ref{subsec:Scaling_Strategy}, and 
the ``adaptive'' strategy is employed. 
Diffusion-GAN is compared with the vanilla GAN with the scaled generator $s_t G_\theta({\z})$ trained by the scaled data $\tilde{\x} = s_t \x$ 
without noise injection. 
The precision and recall for each $\sigma_{\text{noise}}$ are shown in Fig~\ref{appendix:diff_noise_scale_diffusion_supp}. 
We see that overall the learning process is stable even in the case of $\sigma_{\mathrm{noise}}=0$, 
while Diffusion-GAN with $\sigma_{\mathrm{noise}}=0.05$ seems unstable. 
Fig~\ref{1730colormap}~(b) shows the training data, generated samples, and the discriminator's outputs. 
Unlike the results in Fig~\ref{1730colormap}~(a), 
the noise does not have the effect of smoothing the surface of the discriminator's outputs that much. 
This result suggests that scaling contributes significantly to stabilization and that the effect of noise injection is limited.

\begin{figure*}[!t]
 \begin{tabular}{cc}
  \includegraphics[width=75mm]{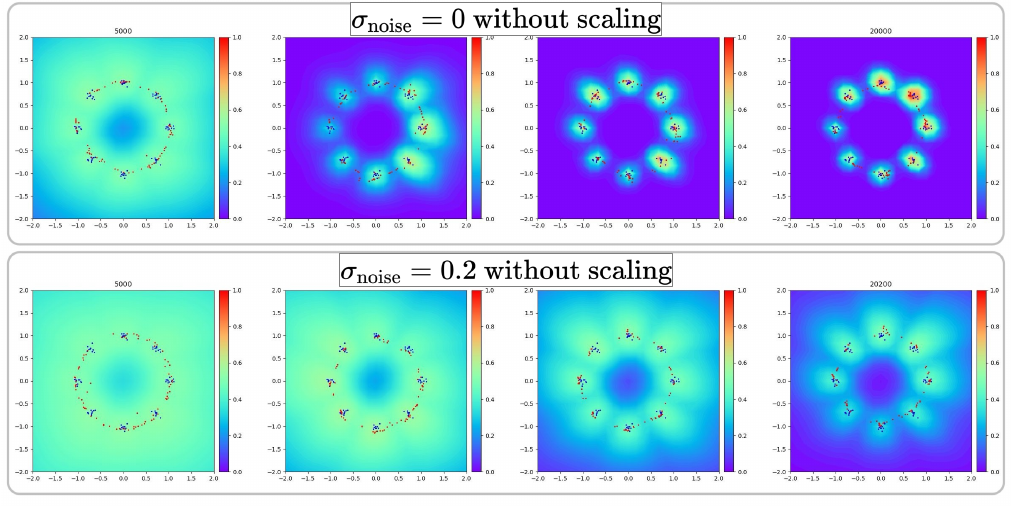}  &
  \includegraphics[width=75mm]{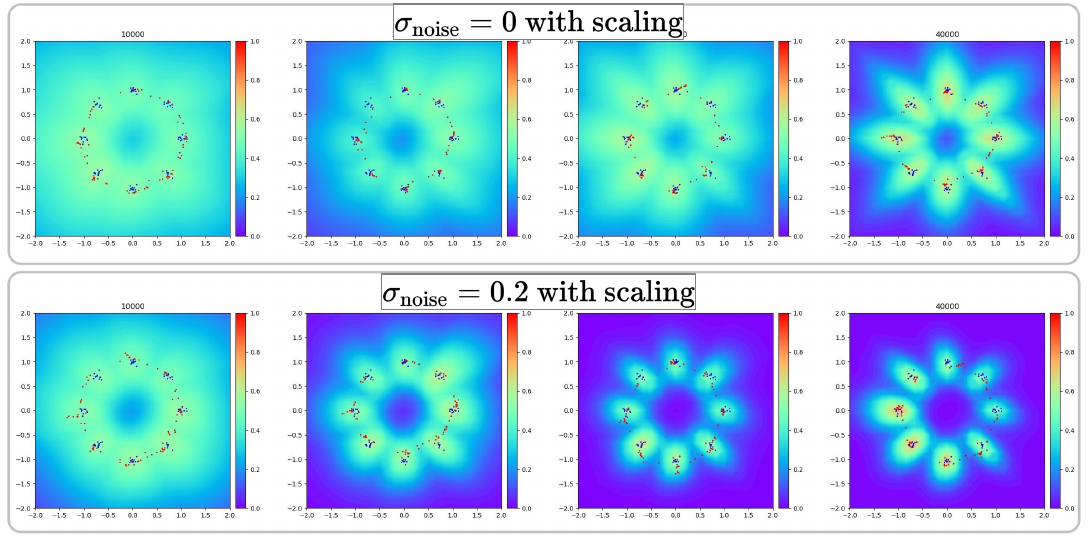} \\ 
  (a) Learning without scaling   &  (b) Learning with scaling  
 \end{tabular}
 \caption{
 (a) Learning without scaling: 
 the training data (red dots), generated data (blue dots), and the discriminator's outputs (heatmap) are depicted. 
 Upper panels: $\sigma_{\text{noise}}=0$. Lower panels: $\sigma_{\text{noise}}=0.2$. 
 The number of training iterations ranges from 5000 to 20000 at every 5000 from right to left. 
 (b) Learning with scaling: 
 the training data (red dots), generated data (blue dots), and the discriminator's outputs (heatmap) are depicted. 
 Upper panels: ``adaptive'' scaling with $\sigma_{\text{noise}}=0$. 
 Lower panels: Diffusion-GAN with ``adaptive'' scaling and $\sigma_{\text{noise}}=0.2$. 
 The number of training iterations ranges from 10000 to 40000 at every 10000 from right to left. 
 }
 \label{1730colormap}
 \end{figure*}

\section{Proposed Framework}
Based on the analysis in Section~\ref{sec:Effect_of_Data_Scaling_and_NoiseInjection_in_GANs}, 
we propose a learning algorithm called Scale-GAN. 
We use the following notations. 
Let $\mu$ be the Lebesgue measure on the Borel algebra of a subset in the Euclidean space. 
Let us define 
$\|f\|_1=\int|f|\rmd{\mu}$ and $\|f\|_\infty=\sup_{\x}|f(\x)|$, which is regarded as the essential supremum according to
the context. 
The function set ${\displaystyle L^{p}} (p=1,\infty)$ is defined as the set of all measurable functions 
for which $\|f\|_p<\infty$ holds. 
The expectation of the measurable function $f$ with respect to the probability density $p$ is 
denoted by $\mathbb{E}_{p}[f]$ or $\mathbb{E}_{\x\sim p}[f(\x)]$, meaning that 
$\int f(\x)p(\x)\rmd{\mu}(\x)$. 
The Lipschitz constant of $f$ is defined by 
$\|f\|_{\mathrm{Lip}}=\sup_{\x\neq \y}\frac{|f(\x)-f(\y)|}{\|\x-\y\|}$, 
where $\|\cdot\|$ is the Euclidean norm. 
For $\delta\in[0,1/2)$, let us define 
 $\mathcal{U}_\delta=\{f\in L^\infty\,|\,\delta<\inf{f},\,\sup{f} <1-\delta\}$. Let $[a]_+=\max\{a,0\}$ for $a\in\Rbb$.

\subsection{Learning Algorithm}
\label{subsec:Learning_Algorithms}
In Section~\ref{sec:Effect_of_Data_Scaling_and_NoiseInjection_in_GANs}, 
we find that data scaling contributes not only to counteracting overfitting but also to stabilization.
Adding a diffusion term suggested that data augmentation updates the generator in a different direction than it should. Based on this idea, we consider data augmentation with scaling alone.

In addition to the data augmentation, let us consider the regularization. 
As discussed in~\cite{mangalam2021overcoming}, 
the generator is prone to catastrophic forgetting and mode collapse 
when the discriminator's outputs differ significantly among modes. 
Thus, stabilization of the discriminator is critical. 
We use the variance-based regularization $\mathbb{E}_t[\mathbb{V}_{p_t}[\widetilde{D}]|t]$ to the discriminator $\widetilde{D}(y,t)$, where $p_t$ is the distribution of the scaled data. A simple calculation leads that this  regularization is an upper bound of the expectation of the approximate derivative,
$\mathbb{E}_t[\mathbb{E}_{p_t(x)\otimes p_t(y)}[\frac{|\widetilde{D}(x,t)-\widetilde{D}(y,t)|}{\|x-y\|}]^2]$,
up to a constant factor. 
Our variance regularization is similar to the NICE regularization~\citep{ni2024nice}, 
which is an empirical approximation of  
$\mathbb{E}_{x\sim p_0}[\mathbb{V}_s[D(x\circ s)]|x], s\sim N_d({\bm 1},\beta^2{\bm I})$, where $p_0$ is the data distribution. 
The paper's author proved that the NICE regularization approximates the gradient penalty for the discriminator. 
Compared to some popular gradient-based regularization \citep{arjovsky2017wasserstein,gulrajani2017improved,zhou2019lipschitz}, the variance-based regularization has an advantage for computation efficiency as well as NICE. 
In Section~\ref{subsec:Theoretical_Analysis}, 
we show that our variance regularization leads to the invariance of gradient direction under the data scaling. 
This is an important feature that mitigates the imbalance of the convergence speed between the discriminator and generator in the GAN learning process. 


We introduce the loss function for Scale-GAN. 
The learning algorithm is obtained through an empirical approximation of the loss function. 
Let $\z$ be the latent variable of the generator having the distribution $p_z$. 
The distribution of the scaling intensity, $t$, is denoted by $\pi(t)$ on the interval $[0, T]$. 
Here, $T$ is a fixed positive value, while $T$ can be variable in practical learning algorithms. 
The proposed framework for Scale-GAN is given by the min-max optimization problem, 
\begin{align}
 \min_{G}\max_{\widetilde{D}}
 \mathbb{E}_{t\sim\pi}\big[ L_{p_0}(G,\widetilde{D}) -\lambda \mathbb{V}_{p_0}[\widetilde{D}(s_t\x,t)] \big],
\label{eqn:Scale-GAN-loss}
\end{align}
with 
\begin{align*}
L_{p_0}(G,\widetilde{D}):= \mathbb{E}_{p_0}[\log \widetilde{D}(s_t\x,t)]+\mathbb{E}_{\z}[\log(1-\widetilde{D}(s_t G(\z),t))]
\end{align*}
for the data distribution $p_0(\x), \x\in\mathcal{X}$. 
Here, $\lambda\geq 0$ is the regularization parameter. 
The learning algorithm with an empirical approximation 
of \eqref{eqn:Scale-GAN-loss} is illustrated in Section~\ref{appendix:Learning_Algorithm_Scale-GAN}. 


\subsection{Theoretical Analysis}
\label{subsec:Theoretical_Analysis}

This section is devoted to revealing three theoretical properties of Scale-GAN; 
i) invariance property of the gradient direction for the generator's learning, 
ii) the bias induced by the variance regularization and 
iii) relation between the estimation error bound and the scaling strategy. 
The property i) means that the data scaling will not degrade the efficiency of the generator learning. Due to ii), we can quantitatively understand the effect of the regularization. The result iii) provides a guideline on how to design the scaling $s_t$ to balance the learning stability and data-generation quality. 

It is important to study the properties of the optimal discriminator 
for the inner maximization problem of~\eqref{eqn:Scale-GAN-loss}. 
In the below, the discriminator on $\mathcal{X}$ (resp. $\mathcal{X}\times[0,T]$) is denoted by $D$ (resp. $\widetilde{D}$). 
\begin{theorem} 
\label{thm:opt-discriminator}
 Suppose that $\mathcal{X}$ is $[0,1]^d$ or $\Rbb^d$. 
 Let $q(\y), \y\in\mathcal{X}$ be the probability density of the generated sample $\y=G(\z), \z\sim p_z$. 
 Suppose $\|p_0\|_\infty<\infty$ and 
 $\frac{p_0}{p_0+q}\in \mathcal{U}_\delta$ for a $\delta\in[0,1/2)$. 
 Let $s_t, t\in[0,T]$ be a strictly positive scaling function. 
 Then, the following maximization problem has 
 the unique optimal solution 
 $\{\widetilde{D}(\cdot,t)\}_{t\in[0,T]}\subset \mathcal{U}_0$, 
 \begin{align*}
  \max_{\substack{\widetilde{D}(\cdot,t)\in\mathcal{U}_0,\\ 0\leq t\leq T}}
  \mathbb{E}_{t\sim\pi}[L_{p_0}(G,\widetilde{D})-\lambda \mathbb{V}_{p_0}[\widetilde{D}(s_t\x,t)]
 \end{align*}
 Furthermore, the optimal discriminator satisfies $\widetilde{D}(s_t\x,t)=\widetilde{D}(\x,0)$ for $(\x,t)\in\mathcal{X}\times[0,T]$. 
\end{theorem}
The proof is deferred to Section~\ref{appendix:Theorem_scale_invariance}.
In the vanilla GAN, the explicit expression of the optimal discriminator is obtained. 
In our case, however, such an explicit expression is unavailable due to the regularization term. 
In the proof, we use the fact that the objective function is concave in the discriminator. The concavity ensures that 
the G\^{a}teaux differential~\citep{kurdila2005convex} of the objective function leads the condition on the global optimality. 
Then, we obtain a cubic equation of $\widetilde{D}(\x,t)$ at each $\x$ and $t$. 
Analyzing the cubic equation, we can prove the existence of the optimal solution. 
The uniqueness comes from the strict concavity of the objective function. 
\begin{remark}
 The assumption $\frac{p_0}{p_0+q}\in \mathcal{U}_\delta$ means that 
 $q/p_0$ is bounded away from zero and infinity on $\mathcal{X}$. 
\end{remark}

\subsubsection{Invariance of Gradient Direction}
\label{subsubsec:Invariance_Gradient_Direction}
For a fixed generator $G_\theta$, suppose that the optimal discriminator $\widetilde{D}(\y,t)$ is obtained. 
Theorem~\ref{thm:opt-discriminator} guarantees the equality $\widetilde{D}(s_t\x,t)=\widetilde{D}(\x,0)$. 
Then, we have $\nabla_{\y}\widetilde{D}(s_t\x,t)=\frac{1}{s_t}\nabla_{\y}\widetilde{D}(\x,0)$. 
At the optimal discriminator, the gradient vector of the objective function in~\eqref{eqn:Scale-GAN-loss} 
with respect to $\bm{\theta}$ is 
$
 \mathbb{E}_{\z}[ \frac{\nabla_{\theta}G_\theta(\z)\nabla_{\y}\widetilde{D}(G_\theta(\z),0) }{-1+\widetilde{D}(G_\theta(\z),0)}]$, 
which is independent of the distribution of the scaling intensity $\pi(t)$. 
Since $s_0=1$, the above gradient is nothing but the gradient for the vanilla GAN. 
Note that the NICE regularization~\citep{ni2024nice} does not induce such an invariance property. 

As illustrated in Section~\ref{sec:DataScaling}, 
the small scaling will make slower progress in learning the discriminator.
In contrast, the learning of the generator is not affected by the scaling that much. 
On the other hand, if the noise is also injected,
the generator's gradient direction is disturbed.
As a result, the convergence of the generator becomes slower.
Hence, the scaling without noise is thought to improve the balance of the learning progress for the discriminator and generator compared to Diffusion-GAN, in which 
the convergence of both the discriminator and generator becomes slower due to the noise injection. 
The above discussion is confirmed by numerical experiments in Section~\ref{subsec:Additional_Numerical_Studies}.

\subsubsection{Bias induced by Variance Regularization}
For $\lambda=0$, the probability distribution corresponding to the optimal generator is $p_0$~\citep{goodfellow2014generative}. 
Let us extend this to learning with regularization. 
Let $q_\lambda$ be the probability density corresponding to the optimal generator 
of \eqref{eqn:Scale-GAN-loss}. 
Under additional assumptions to Theorem~\ref{thm:opt-discriminator}, the upper bound of $\|p_0-q_\lambda\|_\infty$ is evaluated as follows. 
The detailed assumptions and the proof are found in Section~\ref{appendix:Theorem_bias_Scale-GAN}. 
\begin{theorem}
\label{thm:bias_Scale-GAN}
Let $\mathcal{X}$ be $[0,1]^d$ or $\Rbb^d$. 
Suppose $\|p_0\|_\infty<\infty$ and $\|p_0\|_{\mathrm{Lip}}<\infty$. 
Let $q_\lambda$ be the probability density corresponding to the optimal generator of~\eqref{eqn:Scale-GAN-loss}. 
Then, $\|p_0-q_\lambda\|_\infty=O(\lambda^{\frac{1}{d+3}})$ holds. 
\end{theorem}
Theorem~\ref{thm:bias_Scale-GAN} suggests that the variance regularization is close to the regularization with $ L_{\infty}$-norm 
on the set of probability densities. 


\subsubsection{Estimation Error Bound}
We apply the statistical analysis of GAN introduced by~\cite{puchkin24:_rates} and \cite{belomestny23:_simul}. 
Let us assume $\mathcal{X}=[0,1]^d$. 
For i.i.d. training data $\x_1,\ldots,\x_n\sim p_0$ and i.i.d. scaling intensity $t_1,\ldots,t_n\sim \pi$, 
the empirical approximation of \eqref{eqn:Scale-GAN-loss} is given by the min-max optimization problem, 
\begin{align}
 \label{eqn:empirical-scale-GAN}
 \min_{G\in\mathcal{G}}\max_{\widetilde{D}\in\widetilde{\mathcal{D}}}\ 
 \frac{1}{n}\sum_{i}\log \widetilde{D}(s_{t_i}\x_i, t_i)
 +
 \frac{1}{n}\sum_{i}\log(1-\widetilde{D}(s_{t_i}G(\z_i), t_i))
 -\lambda \widetilde{\mathbb{V}}[\widetilde{D}], 
\end{align} 
where $\widetilde{\mathbb{V}}[\widetilde{D}]$ is an empirical approximation of the variance regularization. 
Deep neural networks (DNNs) are used for $\mathcal{\widetilde{D}}$ and $\mathcal{G}$. 
In our theoretical analysis, we consider DNNs with the rectified quadratic activation function (ReQU), 
which is defined by the square of ReLU. 
Learning using ReQU-based DNNs enables us to simultaneously approximate function value and the derivative, which is the preferable property for our purpose; see \citep{puchkin24:_rates} for details. 

The set of probability densities corresponding to $\mathcal{G}$ 
is denoted by $\mathcal{Q}$, 
i.e., the set of the push-forward probability density of $p_z$ by $G\in\mathcal{G}$. 
Let $\widehat{q}_\lambda\in\mathcal{Q}$ be the probability density of the trained generator. 
The estimation accuracy is evaluated by the JS divergence, 
\begin{align*}
\mathrm{JS}(p_0,\widehat{q}_\lambda)=\frac{1}{2}\mathbb{E}_{p_0}[\log\frac{2p_0}{p_0+\widehat{q}_\lambda}]
+\frac{1}{2}\mathbb{E}_{\widehat{q}_\lambda}[\log\frac{2\widehat{q}_\lambda}{p_0+\widetilde{q}_\lambda}]. 
\end{align*}
As the function class of $p_0$ and $G_0$, 
let us consider the H\"{o}lder class of order $\beta$, i.e., $\mathcal{H}^\beta(\mathcal{X})$ endowed the norm
$\|\cdot\|_{\mathcal{H}^\beta}$. 
Let us define
\begin{align*}
&\mathcal{H}^\beta(\mathcal{X},B)=\{f\in\mathcal{H}^\beta(\mathcal{X})|\|f\|_{\mathcal{H}^\beta}\leq B\}, \\
\text{and}\ &\ 
\mathcal{H}_\Lambda^\beta(\mathcal{X},B)=\{f\in\mathcal{H}^\beta(\mathcal{X},B)|\Lambda^{-2}I_{d\times d}\preceq\nabla_x f^T\nabla_xf\preceq\Lambda^{2}I_{d\times d}\}
\end{align*}
accordig to \cite{puchkin24:_rates}. 
Note that $G_0\in\mathcal{H}^{1+\beta}_\Lambda(\mathcal{X})$ leads to $p_0\in\mathcal{H}^{\beta}(\mathcal{X})$. 
Without loss of generality, we assume that $\pi(t)$ is the uniform distribution on $[0,1]$ for the continuous scaling intensity. 
Suppose  $\sup_{\widetilde{D}\in\widetilde{\mathcal{D}}}\|\widetilde{D}\|_{\mathrm{Lip}}\leq L$ for a constant $L>0$. 

\begin{theorem}
\label{thm:estimation_error} 
 Suppose that the generator $G_0$ of the data distribution $p_0\in\mathcal{H}^{\beta}(\mathcal{X})$ 
 satisfies $G_0\in\mathcal{H}_{\Lambda}^{1+\beta}([0,1]^d, H_0)$ for $\beta>2, H_0>0$ and $\Lambda>1$. 
 For the scaling function $t\mapsto s_t$, 
 suppose 
 $1/s\in\mathcal{H}^\alpha([0,1])$ and $L\|s\|_{\mathcal{H}^1} \lesssim n^{c'}$ for  constants $\alpha >2$ and $c'>0$. 
 Then, there exist ReQU-based DNNs, $\mathcal{G}$ and $\widetilde{\mathcal{D}}$, 
 such that the following holds with high probability greater than $1-\delta$, 
\begin{align}
 \mathrm{JS}(p_0,\widehat{q}_\lambda)
& \lesssim
 \bigg[\max_{q\in\mathcal{Q}}\|\frac{p_0}{p_0+q}\|_{\mathcal{H}^\beta} -
 \frac{1}{8\sqrt{d}}\cdot\frac{L}{\|1/s\|_{\mathcal{H}^\alpha}}
 \bigg]_+^2
 + \bigg\{\bigg(\frac{L}{\|1/s\|_{\mathcal{H}^\alpha}}\bigg)^2 +c\bigg\}
 \bigg(\frac{\log{n}}{n}\bigg)^{\frac{2\beta}{2\beta+d}}
 +\frac{\log(1/\delta)}{n} + \lambda. 
\label{eqn:estimation-accuracy-DataScaling}
\end{align}
 In the above, $c$ is a positive constant depending only on $d, \beta, \Lambda, c'$ and $H_0$. 
\end{theorem}
The proof is shown in Section~\ref{appendix:Theorem_estimation_error}. 
Theorem~\ref{thm:estimation_error} indicates that 
$s_t$ controls the bias(1st term)-variance(2nd term) trade-off in the estimation error bound. 
The order of the variance term coincides with the min-max optimal rate for the class of probability densities in $\mathcal{H}^{\beta}(\mathcal{X})$. In this case, however, the class of $p_0$ is slightly restricted by the push-forward 
with $G_0$. 
The estimation error for the vanilla GAN is recovered from $\lambda=0$ and $s_t=1$~\citep{puchkin24:_rates}. 
The scaling with a large $\|1/s\|_{\mathcal{H}^\alpha}$, such as $s_t=e^{-Mt}$ with $1\ll M\leq n^{c'}/L$, leads to a small variance and a large bias. 
When $L$ is a large constant, and the first term of the upper bound vanishes, 
the data augmentation with scaling will improve the estimation accuracy. 


\subsubsection{Invertible Data Augmentation} 

Let us consider an extension of data scaling. 
A remarkable property of data scaling is that it is an invertible transformation.
For invertible data augmentations (DAs), some properties similar to those of Scale-GAN hold, including 
the expression of the optimal discriminator in Theorem~\ref{thm:opt-discriminator}, 
the effect of variance regularization in Theorem~\ref{thm:bias_Scale-GAN}, 
and the invariance of gradient direction shown in Section~\ref{subsubsec:Invariance_Gradient_Direction}. 
Theorem~\ref{thm:estimation_error} concerning estimation accuracy also holds with minor modifications.

The invertible data augmentation we consider here is formulated as follows. 
Let $S_t:\mathcal{X}\rightarrow\mathcal{X}$ be a transformation on $\mathcal{X}$ with the parameter $t\in{T}$ such that 
$S_t^{-1}(S_t(\x))=\x$ holds for $\x\in\mathcal{X}$ with inverse map $S_t^{-1}$. 
In what follows, parentheses for the transformation are omitted; that is, we use the notation 
$S_t\x$ or $S_t^{-1}\x$ for simplicity. 
In the learning process, the data scaling is replaced with the data transformation with $S_t$
meaning that the pair $(S_t\x,t)$ is substituted into the discriminator in \eqref{eqn:empirical-scale-GAN}. 
Such a learning procedure is referred to as \emph{generalized Scale-GAN}. 
Detailed analysis of the generalized Scale-GAN is presented in Section~\ref{appendix:Invertible_Data_Augmentation}. 

The estimation accuracy of the generalized Scale-GAN is roughly given by 
\begin{align}
 \mathrm{JS}(p_0,\widehat{q}_\lambda)
& \lesssim
 \bigg[\max_{q\in\mathcal{Q}}\|\frac{p_0}{p_0+q}\|_{\mathcal{H}^\beta} -
 \frac{1}{8\sqrt{d}}\cdot\frac{L}{\mathcal{J}_{\mathrm{DA}}}
 \bigg]_+^2
 + \bigg\{\bigg(\frac{L}{\mathcal{J}_{\mathrm{DA}}}\bigg)^2 +c\bigg\}
 \bigg(\frac{\log{n}}{n}\bigg)^{\frac{2\beta}{2\beta+d}}
\label{eqn:estimation-accuracy-InvertibleDA}
\end{align}
with high probability, 
where $\mathcal{J}_{\mathrm{DA}}$ is the coefficient depending on the data augmentation employed the learing
algorithm. 
For the invertible DA $S_t, t\in{T}$, let us define $\phi_{S^{-1}}(\x,t)={S_t}^{-1}\x\in\mathcal{X}$. 
For the single DA, $\{S_t\,:\,t\in{T}\}$, suppose
$\phi_{S^{-1}}\in\mathcal{H}^{\alpha}(\mathcal{X}\times{T}), \alpha>2$. Then, we have 
\begin{align*}
\mathcal{J}_{\mathrm{DA}}=\frac{1+\|\phi_{S^{-1}}\|_{\mathcal{H}^\alpha}}{2}. 
\end{align*}
Furthermore, suppose that $B$ invertible data augmentations, $S_{i,t_i},t_i\in{T}_i, i=1,\ldots,B$ such that 
$\phi_{S_{i}^{-1}}\in\mathcal{H}^{\alpha_i}(\mathcal{X}\times{T_i}), \alpha_i>2$, are used. 
In this case, the augmented data, $S_{i,t_i}\x,\, i\in[B], t_i\in{T}_i$, is fed into the GAN algorithm, where
the index $i$ and the parameter $t_i$ are randomly selected. 
Then, we have 
\begin{align*}
\mathcal{J}_{\mathrm{DA}}=\frac{\sqrt{B}+\max_{i\in[B]}\|\phi_{{S_i}^{-1}}\|_{\mathcal{H}^{\alpha_i}}}{2}. 
\end{align*}

\begin{example}[Scaling]
For the data-scaling, $\x\mapsto S_t\x=s_t\x$, the map $\phi_{S^{-1}}$ is defined by
$\phi_{S^{-1}}(\x,t)=\x/s_t$. Hence, the H\"{o}lder norm is
 $\|\phi_{S^{-1}}\|_{\mathcal{H}^{\alpha}}=\|1/s\|_{\mathcal{H}^{\alpha}}$, 
which leads to 
$\|1/s\|_{\mathcal{H}^{\alpha}}/2\leq 
\mathcal{J}_{\mathrm{DA}}=(1+\|1/s\|_{\mathcal{H}^{\alpha}})/2
\leq \|1/s\|_{\mathcal{H}^{\alpha}}$. 
\end{example}

\begin{example}[$\pi/2$ rotation]
 The $\pi/2$ rotation is realized by the permutation on pixels. 
 Suppose the pixel of the image is indexed by $(i,j)$ such that
 $(i,j)\in\mathbb{Z}\times\mathbb{Z}, |i|,|j|\leq v$. 
 Then, $\pi/2$ rotation of the image 
 maps $(x_{i,j})$ to $(x_{i,j}')=(x_{-j,i})$, 
 which is simply a permutation. Thus, the $\pi/2$ rotation is invertible. 
 Similarly, $\pi$ and $3\pi/2$ rotations are also invertible transformations. 
 The rotation with a fixed angle $\pi/2$ is represented by $S_t$ with the singleton $T=\{\pi/2\}$. 
 Then, the H\"{o}lder norm of $\phi_{S^{-1}}$ is $\|\phi_{S^{-1}}\|_{\mathcal{H}^{\alpha}}=1$. 
 When rotations with $0, \pi/2, \pi$, and $3\pi/2$ angles are used, 
 we have $\mathcal{J}_{\mathrm{DA}}=(\sqrt{4}+1)/2=1.5$. 
\end{example}

\begin{example}[Saturation]
 Each pixel of the RGB image, $(r,g,b)\in[-1,1]^3$, is mapped to $(r',g',b')$ by the saturation transformation 
 $(r,g,b,1)\mapsto (r',g',b',\ast)=(t I_4+(1-t)\bm{v}\bm{v}^T)(r,g,b,1)^T$, where $t$ is a positive parameter
 sampled from $\pi(t)$ and $\bm{v}\in\Rbb^4$ is a fixed unit vector~\citep{karras2020training}. 
 The log-normal distribution often applies as $\pi(t)$. 
 Here, we assume $t$ is bounded below by a positive constant, i.e., $t\geq t_{\min}>0$. 
 At each pixel, the inverse map is defined by the linear transformation with the matrix
 $t^{-1}I_4+(1-t^{-1})\bm{v}\bm{v}^T$. Then, the H\"{o}lder norm of $\phi_{S^{-1}}(\x,t)$ is 
 of the order $\|\phi_{S^{-1}}\|_{\mathcal{H}^{\alpha}}=O((1/t_{\min})^{\alpha+1})$. 
 The coefficient $\mathcal{J}_{\mathrm{DA}}$ is also $O((1/t_{\min})^{\alpha+1})$. 
\end{example}

The data augmentation using only pixel blitting, such as the rotation, does not significantly impact on the estimation accuracy. 
On the other hand, the color formation such as data-scaling or saturation can greatly change the balance between the bias and the variance of the estimation error bound.

\section{Numerical Experiments}
\label{Numerical_Experiments}

We first investigate the effects of the variance regularization using the synthetic dataset in Section~\ref{sec:Effect_of_Data_Scaling_and_NoiseInjection_in_GANs}. 
The results are shown in Section~\ref{appendix:Scale-wise_Variance_Reg}. 
Overall, the regularized Scale-GAN with an appropriate $\lambda$, such as $\lambda=0.1$ to $0.5$, outperforms non-regularized Scale-GAN with $\lambda=0$. 
However, the recall for the learning with strong regularization, such as $\lambda=5, 10$, is degraded, meaning that the generator misses some modes. 

Next, in Sections~\ref{subsec:Image_Generation} and
\ref{subsec:Additional_Numerical_Studies}, 
we examine the effectiveness of the proposed method using some benchmark datasets. 
Based on the above result, the regularization parameter $\lambda$ is set to around $0.1$. 
Though the benchmark data is much larger than the above synthetic data, 
we show that the Scale-GAN with the regularization parameter in the above properly works.

\subsection{Image Generation}
\label{subsec:Image_Generation}
The effectiveness of the proposed method is confirmed in image generation, which is the standard task in GAN.
The three datasets, {\tt CIFAR-10}, {\tt STL-10}, and {\tt LSUN-Bedroom} are used. 
{\tt CIFAR-10}~\citep{krizhevsky2009learning} consists of 
50k images of the size $32\times32\times3$ in 10 classes. 
{\tt STL-10}~\citep{coates2011analysis} consists of 100k images of the size $64\times 64\times3$ in 10 classes.  
While the original resolution of {\tt STL-10} is $96\times96\times3$, we resized it to compare with past studies. 
{\tt LSUN-Bedroom}~\citep{yu16:_lsun} includes 200k images of bedrooms of the size $256\times 256\times3$. 
We assess the image quality in terms of sample fidelity (FID) and sample diversity (recall), 
which are used in~\cite{wang2022diffusion}. 

We compare our method with StyleGAN2~\citep{karras2020analyzing}, 
StyleGAN2+DiffAug~\citep{zhao2020differentiable}, 
StyleGAN2+ADA~\citep{karras2020training}, and Diffusion-GAN with StyleGAN2~\citep{karras2020analyzing}. 
Also we use StyleGAN2 in Scale-GAN. 

Hyperparameters for existing methods are defined as the same as those in~\cite{wang2022diffusion}. 
Also, our method uses almost the same hyperparameters, while some of them are adjusted based on preliminary experiments.  
The detailed hyperparameters are summarized in Section~\ref{appendix:Hyperparameters_Image_Generation}. 
The distribution of the scaling intensity, $\pi_0$, on $\{1,\ldots,T\}$ has two choices, ``uniform'' and ``priority''.
The ``uniform'' distribution is defined by $\pi_0(t)=1/T$ used in Section~\ref{subsec:Scaling_Strategy}. 
Note that the end point $T$ can change from iteration to iteration according to the scaling strategy. 
The ``priority'' distribution is defined by $\pi_0(t)\propto t$, which is recommended in~\cite{wang2022diffusion}. 
Hence, the priority distribution is used in Diffusion-GAN. 
On the other hand, in our method, the uniform distribution is used for {\tt CIFAR-10} and {\tt STL-10}, 
and the priority distribution is used for {\tt LSUN-Bedroom}. 
The hyperparameter $d_{\mathrm{target}}$ in the adaptive strategy of $T$ is determined according to~\cite{wang2022diffusion}. 
Furthermore, $T$ is clipped to the interval $[T_{\min}, T_{\max}]$. 
The scaling function $s_t$ is determined by $\beta_0$ and $\beta_T$; see Section~\ref{subsec:Scaling_Strategy}. 
In Diffusion-GAN, $\sigma_{\mathrm{noise}}=0.05$, $\beta_0=0.0001$ and $\beta_T=0.02$ are used. 
In our method, a larger $\beta_T$ properly works compared to Diffusion-GAN. 

Table~\ref{numerical_result_all} shows the results of image generation on each benchmark dataset. 
The upper four methods are cited from the numerical results in \cite{wang2022diffusion},
and the lower two methods show the numerical results conducted using our computation environment. 
In terms of the FID score, our method outperforms the other methods by a clear margin.  
Also, the recall of our method attains higher values than the others. 

The computational cost per iteration for Scale-GAN is approximately equivalent to that of Diffusion-GAN. For the same computation time, Scale-GAN outperforms Diffusion-GAN in data-generation quality. 
For instance, on the LUSN-Bedroom dataset, Scale-GAN achieves the optimal FID score for Diffusion-GAN in roughly one-fourth of the computational time.

\begin{table*}[t!]
 \caption{Image generation results on benchmark datasets: CIFAR-10, STL-10, and LSUN-Bedroom. 
 We highlight the best and second best results in each column with bold and underline, respectively.
 Lower FIDs indicate better fidelity, while higher Recalls indicate better diversity. 
 } \label{numerical_result_all} \vspace*{-1mm}
 \centering\hspace*{-7mm} 
  \begin{tabular}{p{18em}cccccccc}
                     & \multicolumn{2}{c}{CIFAR-10}         && \multicolumn{2}{c}{STL-10}            && \multicolumn{2}{c}{LSUN-Bedroom}\\
                     & \multicolumn{2}{c}{(32 $\times$ 32)} && \multicolumn{2}{c}{(64 $\times$ 64)}  && \multicolumn{2}{c}{(256 $\times$ 256)}\\
                      \cline{2-3}  \cline{5-6}  \cline{8-9} 
   \multicolumn{1}{c}{{\bf Methods}} & FID $\downarrow$ &Recall $\uparrow$ && FID $\downarrow$ &Recall $\uparrow$  && FID $\downarrow$ &Recall $\uparrow$  \\
    \hline 
 StyleGAN2
           & 8.32 & 0.41   && 11.70 & \underline{0.44}   && 3.98 & \underline{0.32} \\
 StyleGAN2+DiffAug
   & 5.79 & 0.42   && 12.97 & 0.39   && 4.25 & 0.19 \\
 StyleGAN2+ADA
   &\underline{2.92}  & 0.49	 && 13.72 & 0.36   && 7.89 & 0.05 \\
 Diffusion-GAN
   & 3.19 & \underline{0.58}	 && \underline{11.43} & {\bf 0.45}   && 3.65 & \underline{0.32} \\
    \hline							 				         
 Diffusion-GAN(our re-experiment)                    & 3.39 & 0.57   && 11.53 & {\bf 0.45}   && \underline{3.53} & 0.25 \\
 Scale-GAN(proposal)                            & {\bf 2.87} & {\bf 0.60}   && {\bf 11.37} & {\bf 0.45}   &&  {\bf 1.90} & {\bf 0.36} \\  
   \hline
  \end{tabular}
\end{table*}

\subsection{Additional Experiments on CIFAR-10}
\label{subsec:Additional_Numerical_Studies}

\paragraph{Gradient Direction.}
As shown in Section~\ref{subsec:Theoretical_Analysis}, 
if the optimal discriminator is obtained at each learning iteration, 
the gradient direction passed to the generator is independent of the scaling. 
This property holds when the expectation in the objective function is exactly calculated. 
In the beginning, let us compare the gradient direction of the proposed method with that of Diffusion-GAN 
in the empirical situation. 
The effect of data augmentation, $\x\mapsto\widetilde{\x}$, at the scaling intensity $t$ 
on the gradient direction is evaluated by the averaged cosine similarity
between $\nabla{}\widetilde{D}(\widetilde{\x},t)$ and $\nabla{}\widetilde{D}(\x,0)$ for the distribution of $(\x, \widetilde{\x}, t)$. 
Fig~\ref{mean_cosine_similarity_cifar10} shows the result on CIFAR-10. 
For Diffusion-GAN, the intensity of diffusion is set to $\sigma_{\mathrm{noise}}=0.05$. 
In both learning methods, $\pi_0$ is the uniform distribution on $\{1,\ldots,T\}$, 
and the adaptive strategy is used for the scaling. 
Overall, Scale-GAN provides a gradient closer to the gradient of the original data than Diffusion-GAN does. 
This indicates that the discriminator provides a gradient direction in which the generator's learning will likely progress even under data augmentation. 
Therefore, Scale-GAN is expected to mitigate the imbalance of the convergence speed between the discriminator and generator
more efficiently than Diffusion-GAN. 

\paragraph{Setting: Effect of Each Component.} 
Table~\ref{ablation_cifar10_result} shows the result of numerical experiments for some modified learning methods; 
  i) StyleGAN2 with noise injection by $\widetilde{\x}=\x+\sqrt{1 - s_t^2}\sigma_{\mathrm{noise}} \bm{\epsilon}$, 
 ii) Diffusion-GAN with $\beta_T=0.028$ that is larger than $\beta_T$ used in Table~\ref{numerical_result_all}, 
iii) Diffusion-GAN with the variance regularization ($\beta_T=0.02, \lambda=0.1$), 
 iv) Scale-GAN with a modified variance regularization ($\beta_T=0.028, \lambda=0.1$), and 
  v) Scale-GAN without variance-based regularization ($\beta_T=0.028, \lambda=0$). 
The modified variance regularization is the regularization using second order moment with a fixed mean, $\lambda \mathbb{E}[(\widetilde{D}-1/2)^2]$. 
In the vanilla GAN, 
the optimal solution of the discriminator is the constant function $1/2$ when the generator is correctly specified. 
This fact motivates the usage of the modified variance regularization. 
Even using the modified variance regularization, the same theoretical properties in Section~\ref{subsec:Theoretical_Analysis} hold for Scale-GAN. 
Examined methods i), ii), and iii), include noise injection.

\paragraph{Noise Injection and Data Scaling.} 
For StyleGAN2, the noise injection efficiently works. 
Indeed, the usefulness of the noise injection has been intensively studied in~\cite{feng2021understanding} and other works. 
However, its performance is not as high as 
Diffusion-GAN and Scale-GAN. 
Diffusion-GAN with a stronger scaling yields no improvement in accuracy. 

\paragraph{Regularization.} 
When Diffusion-GAN with the variance regularization and the original Diffusion-GAN are compared, 
the regularization does not improve the accuracy as much as Scale-GAN. 
As shown in v) of Table~\ref{ablation_cifar10_result}, 
Scale-GAN without regularization efficiently works compared to the existing methods. 
Furthermore, we see that the effect of the regularization on the FID score is significant 
when Scale-GAN w/o Var-reg and Scale-GAN are compared. 

\paragraph{Modified Regularization.} 
The modified regularization's effect is expected to be similar to the variance regularization in Scale-GAN. 
Table~\ref{ablation_cifar10_result} shows that the accuracy of Scale-GAN with modified Var-reg is close to that of Scale-GAN. 
However, the modified regularization requires a longer training time to achieve high accuracy; 
see Section~\ref{appendix:subsec:modified_reg}. 
The slow learning is thought to be due to the effect of forcing the discriminator's output to be close to $1/2$. 
This fact especially affects the learning behavior when the adaptive strategy 
is employed for data scaling. 
In this case, $r_d$ tends to become smaller even in overfitting, and $T$, which determines the scaling magnitude, does not become larger. 
As a result, the modified regularization does not sufficiently address the overfitting in the middle stage of learning, which is interpreted as a lack of good gradients and progress in learning.

\begin{figure}[t!]
 \centering 
 \includegraphics[height=70mm]{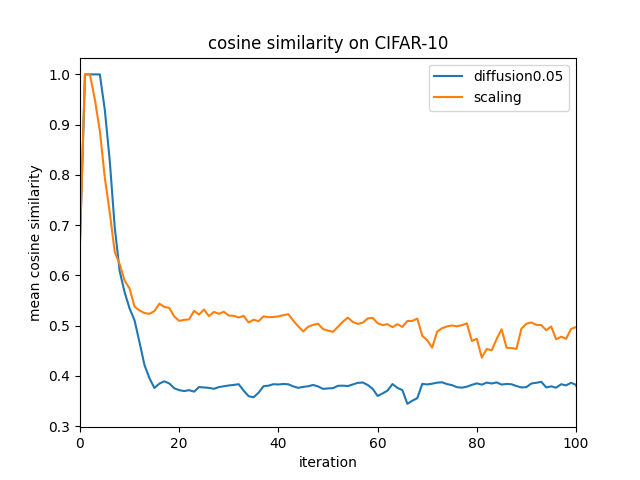}\vspace*{-2mm}
   \caption{Averaged cosine similarity between $\nabla{D}(\widetilde{\x},t)$ and $\nabla{D}(\x,0)$ at each iteration on CIFAR-10. 
 Blue solid line is Diffusion-GAN with $\sigma_{\mathrm{noise}}=0.05$, and orange solid line is Scale-GAN. }
   \label{mean_cosine_similarity_cifar10}
\end{figure}

  \hfill
\begin{table}
   \caption{
 Image generation results on CIFAR-10 for various learning methods. 
 The FID and recall values for StyleGAN2 and Diffusion-GAN in the lower part are identical to those in 
 Table~\ref{numerical_result_all}, as cited from the numerical results in \cite{wang2022diffusion}. 
 The result for Scale-GAN from Table~\ref{numerical_result_all} is also included. ``Var-reg'' denotes variance
 regularization, while ``modified Var-reg'' refers to its modified version.
  } \label{ablation_cifar10_result}
  \centering\vspace*{2mm}
  \begin{tabular}{clcc}
   \multicolumn{2}{c}{{\bf Methods}}    & FID $\downarrow$ & Recall $\uparrow$\\ \hline 
 i)  & StyleGAN2 w/ noise injection     & 3.87 & 0.56 \\
 ii) & Diffusion-GAN: $\beta_T=0.028$   & 3.39 & 0.57 \\
 iii)& Diffusion-GAN w/ Var-reg.        & 3.18 & 0.57 \\ 
 iv) & Scale-GAN w/  modified Var-reg.  & 2.89 & 0.57 \\ 
 v)  & Scale-GAN w/o Var-reg.           & 3.11 & 0.59 \\ \hline 
 &StyleGAN2
	   & 8.32 & 0.41 \\
 &Diffusion-GAN:
       $\beta_T=0.02$ & 3.19 & 0.58 \\
 &Scale-GAN                                              & 2.87 & 0.60 \\ \hline
  \end{tabular}
\end{table}

\section{Conclusions}
\label{Conclusions}
This paper demonstrates, both theoretically and experimentally, that data scaling enhances learning stability in generative models. We show that data scaling regulates the bias-variance trade-off, improving estimation accuracy. Our findings offer valuable insights for designing effective scaling strategies to train generative models.

\appendix




\section{Related Works}
\label{appendix:sec:Related_Works}

\renewcommand{\theequation}{A.\arabic{equation}}
\renewcommand{\thefigure}{A.\arabic{figure}}
\renewcommand{\thetable}{A.\arabic{table}}
\renewcommand{\thetheorem}{A.\arabic{theorem}}
\setcounter{equation}{0}
\setcounter{figure}{0}
\setcounter{table}{0}

While GAN can generate high-quality images, learning instability and mode collapse have long been problems. 
Numerous methods have been proposed to solve these problems.

\subsection{Loss functions, Regularization and Data Augmentation for GANs}
\label{appendix:relatedworks_loss}
Modifying loss functions is a possible approach to overcoming the instability problems. 
The loss function of the vanilla GAN~\citep{goodfellow2014generative} is the Jensen-Shannon (JS) divergence, which often fails to learn a generator if the training data and generated data are easily separated. 
Toward the remedy of the instability issue, discrepancy measures such as Wasserstein distance~\citep{arjovsky2017wasserstein}, squared loss~\citep{mao2017least}, hinge loss~\citep{miyato2018spectral,brock2018large,zhang2019self} and maximum mean discrepancy~\citep{li2017mmd} have been studied and shown to contribute to stabilization. 
As shown in Wasserstein-GAN~\citep{arjovsky2017wasserstein}, 
Lipschitz continuity of the discriminator plays a central role in stabilizing the learning process. 
By adding a penalty constraining the gradient of the discriminator to the objective function, Lipschitz continuity is prone to be satisfied~\citep{gulrajani2017improved,kodali2017convergence,adler2018banach,petzka2018regularization,xu2021towards,mescheder2018training,zhou2019lipschitz,thanh2018improving}. 
In addition, \cite{miyato2018spectral} proposed spectral normalization, 
which guarantees Lipschitz continuity by constraining the spectral norm of the weights. 
As an other approach, Adaptive Instance Normalization~\citep{huang2017arbitrary}, proposed in the context of style transformation, is used in styleGAN~\citep{karras2019style} and attracted much attention for its significant contribution to improving image generation performance. 
The instability of GANs in the early stages of learning is considered to be partly due to the lack of overlap between the support of the generated samples and that of the training data. To address the problem, \cite{roth2017stabilizing} proposed a method for extending the support of the distribution by adding noise. Noise injection as a data augmentation is expected to bring support closer together and suppress the overfitting of the discriminator. 
However, suitable noise distribution heavily depends on the data domain, and thus, the noise injection is hard to implement in practice, as pointed out by~\cite{arjovsky2017towards}.

\subsection{Fusion of GAN and Diffusion Models}
\label{appendix:relatedworks:fusion}
In recent years, the diffusion model has also attracted considerable attention as a model that generates natural high-resolution images~\citep{ho2020denoising,song19:_gener_model_estim_gradien_data_distr}. 
The diffusion model approximates the inverse process of the diffusion process, 
in which noise is added to the data, with a neural network, and generates data by repeated denoising. 
Since the learning of the diffusion model is formulated as the minimization problem of the loss function, such as the denoising score matching loss~\citep{song19:_gener_model_estim_gradien_data_distr}, it can be trained more stably than GANs, which makes it easier to apply to large data sets. 
However, in the sampling phase, hundreds to thousands of DNN's feedforward passes are required for a single image. 
Recent developments have enabled us to reduce 
the number of DNN's computations to about ten 
 iterations by using distillation~\citep{salimans2021progressive} and higher-order differential equation approximation methods~\citep{lu2022dpm, zheng2023dpm}. 
On the other hand, GANs usually need only one feed-forward pass for data generation. 
To combine the advantages of GANs and diffusion models, some frameworks have been recently developed. 
The Diffusion-GAN~\citep{wang2022diffusion}, inspired by the diffusion model, injects noise similar to the DDPM~\citep{ho2020denoising} into GAN's training process to achieve stable computation and high-quality image generation. 
\cite{xiao2021tackling} proposed Denoising Diffusion GAN that uses GAN's discriminator in the diffusion model to achieve both high accuracy and fast sampling. 
Besides the diffusion-GAN, the fusion of GANs and diffusion models have been considered by
\cite{zheng2022truncated,yin2024one,sauer2023adversarial,kim2023consistency,yin2024improved} to improve the
efficiency, speed, and quality of generative models, mainly focusing on diffusion models and distillation
techniques to reduce computational costs. 
and distillation techniques to reduce computational costs.


\section{Additional Numerical Studies to Section~\ref{sec:Effect_of_Data_Scaling_and_NoiseInjection_in_GANs}}
\label{appendix_DataScaling_NoiseInjection_GAN}

\renewcommand{\theequation}{B.\arabic{equation}}
\renewcommand{\thefigure}{B.\arabic{figure}}
\renewcommand{\thetable}{B.\arabic{table}}
\renewcommand{\thetheorem}{B.\arabic{theorem}}
\setcounter{equation}{0}
\setcounter{figure}{0}
\setcounter{table}{0}

Numerical experiments in Section~\ref{sec:Effect_of_Data_Scaling_and_NoiseInjection_in_GANs} are conducted with more seeds. 
The results are shown in Fig~\ref{appendix:seed-wise-simulations}. Each row corresponds to each seed. 
The panels (a), (b), and (c) show the results for four types of data scaling, and 
panels (d) and (e) show the results for three scaling strategies. 

For the scaling, $s=0.25, 0.5, 1, 1.5$, 
the discriminator $D_\phi$ and the scaled generator $s G_\theta$ are trained by the original GAN~\citep{goodfellow2014generative} 
using $\{\y_i\}_{i=1}^{80},\,\y_i=s\x_i$ as true data. 
In Fig~\ref{appendix:seed-wise-simulations}, the panels (a) and (b) show precision and recall at each scaling for different seeds. 
The learning with $s=0.25$ shows that if the distribution is well covered (high recall) in the early stages of learning, 
the precision gradually increases, and the learning is successful.
However, it has not escaped from the initial learning failure. 
These results can be explained by the size of the gradient, as shown in Fig~\ref{appendix:seed-wise-simulations} (c). 
The gradient's norm for $s=0.25$ is smaller than the other cases. 
The result indicates that abrupt changes in the discriminator are considered difficult to occur. 
Hence, the learning result will be significantly affected by the early stages of the learning process. 
On the other hand, for $s=1, 1.5$, high recall is achieved in the early stages of learning. 
However, it is observed that the recall drops significantly 
from the middle to the latter half of the learning period, causing mode collapse. 
Also as shown in Fig~\ref{appendix:seed-wise-simulations}~(c), 
instability is manifested as mode collapse due to large gradients after the middle of the training phase. 

Panels (d) and (e) in Fig~\ref{appendix:seed-wise-simulations} show the precision and recall for each scaling strategy. 
The ``fix'' strategy has low precision and does not learn well compared to the others. 
Both the ``linear const'' and ``adaptive'' strategies can estimate rough distributions from an early stage and learn well. 
It is found that learning data at multiple scales simultaneously is stable. 
The ``adaptive'' strategy is preferable 
since it is less dependent on the hyperparameters and can be used more universally~\citep{wang2022diffusion}.

Fig~\ref{appendix:noise_scale_adding_noise_supp} shows supplementary numerical results in Section~\ref{sec:Noise_Injection_Data_Scaling}. 
We consider the noise injection, 
$\tilde{\x} = \x + \bm{\epsilon}, \bm{\epsilon} \sim \mathcal{N}(\bm{0}, \sigma_{\text{noise}}^2I)$ for 
$\sigma_{\mathrm{noise}}\in\{0, 0.05, 0.1, 0.15, 0.2\}$. 
Numerical experiments conducted with several seeds are depicted. 
Each gray box corresponding to each seed includes two panels, the precision and recall to the learning iteration. 
The noise of any magnitude examined in the experiment can cause instability and may not contribute significantly to stabilization. 

In the diffusion-GAN, both the data scaling and noise injection are incorporated, i.e., 
the training data $\boldsymbol{x}$ is transformed to 
$\tilde{\x} = s_t \x + \sqrt{1-s_t^2}\bm{\epsilon},\, \bm{\epsilon}\sim\mathcal{N}(\bm{0}, \sigma_{\text{noise}}^2I)$. 
The scaling function $s_t$ is defined in the same way as those in Section~\ref{subsec:Scaling_Strategy}, and 
the ``adaptive'' strategy is employed. 
The diffusion-GAN is compared with the GAN with the scaled generator 
$s_t G_\theta({\z})$ trained by the scaled data $\tilde{\x} = s_t \x$ without diffusion. 
Fig~\ref{appendix:diff_noise_scale_diffusion_supp} shows 
the precision and recall for the learning with several  $\sigma_{\text{noise}}$.
Each gray box includes two panels, the precision, and recall to the learning iteration, for each seed when the noise level 
$\sigma_{\mathrm{noise}}$ varies $0, 0.05, 0.1, 0.15$, and $0.2$. 
We see that overall the learning process is stable, while the diffusion-GAN with $\sigma_{\mathrm{sigma}}=0.05$ seems unstable. 
The results in Fig~\ref{appendix:diff_noise_scale_diffusion_supp} 
suggests that scaling contributes significantly to stabilization and that the effect of noise imposition is limited. 
\begin{figure}[tp]
\centering
\begin{tabular}{p{6.5em}p{6.5em}p{6.5em}p{6.5em}p{6.5em}}
 \includegraphics[height=25mm]{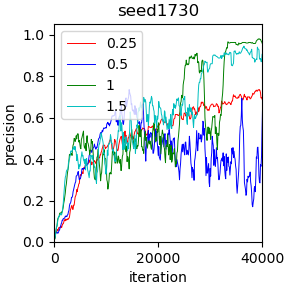} &
 \includegraphics[height=25mm]{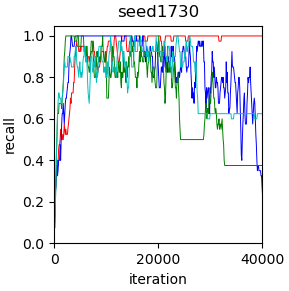}    &
 \includegraphics[height=25mm]{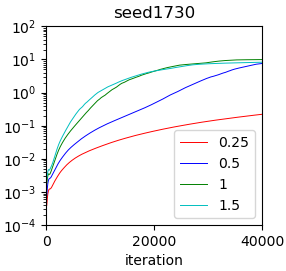}        &
 \includegraphics[height=25mm]{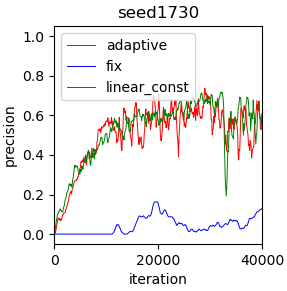}     &
 \includegraphics[height=25mm]{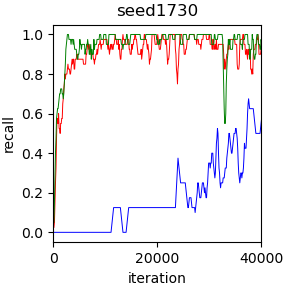} \\
 \includegraphics[height=25mm]{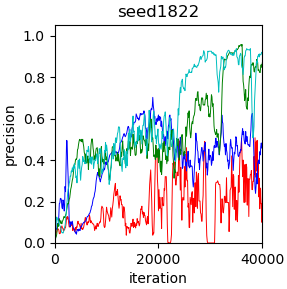} &
 \includegraphics[height=25mm]{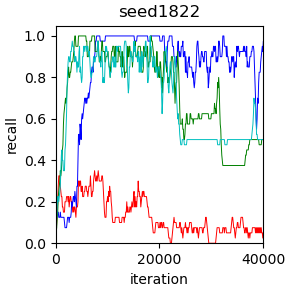}    &
 \includegraphics[height=25mm]{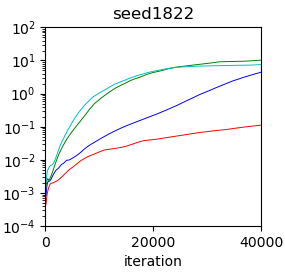}        &
 \includegraphics[height=25mm]{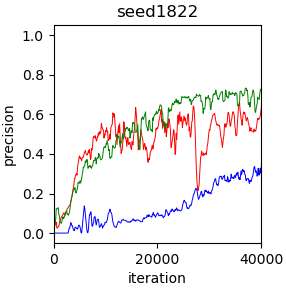}     &
 \includegraphics[height=25mm]{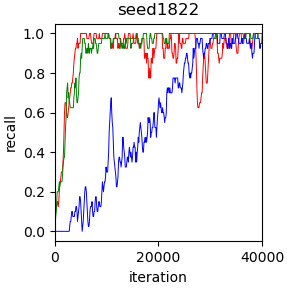} \\
 \includegraphics[height=25mm]{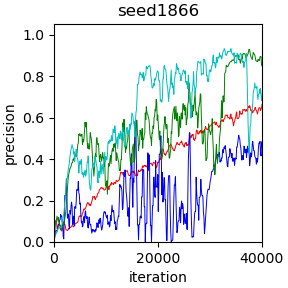} &
 \includegraphics[height=25mm]{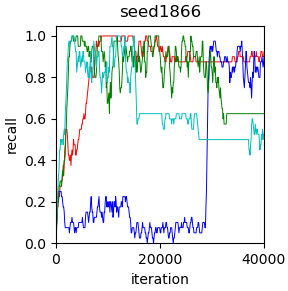}    &
 \includegraphics[height=25mm]{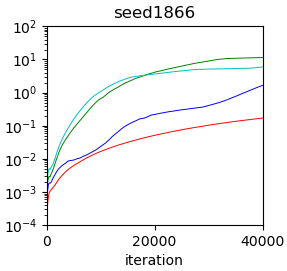}        &
 \includegraphics[height=25mm]{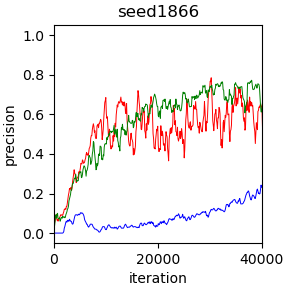}     &
 \includegraphics[height=25mm]{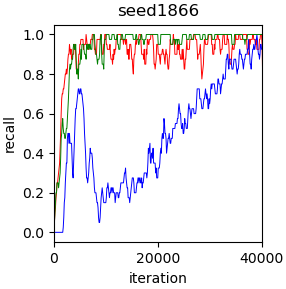} \\
 \includegraphics[height=25mm]{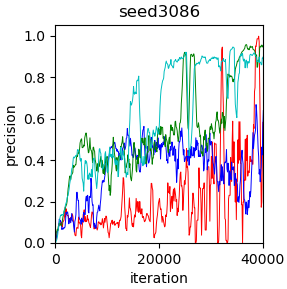} &
 \includegraphics[height=25mm]{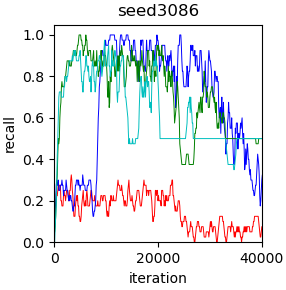}    &
 \includegraphics[height=25mm]{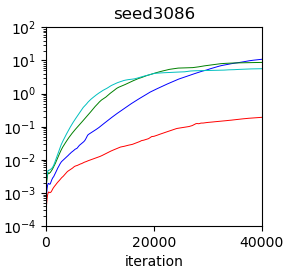}        &
 \includegraphics[height=25mm]{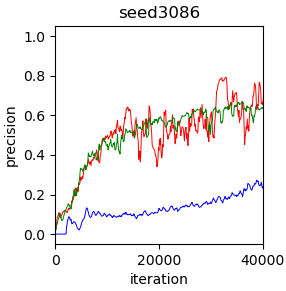}     &
 \includegraphics[height=25mm]{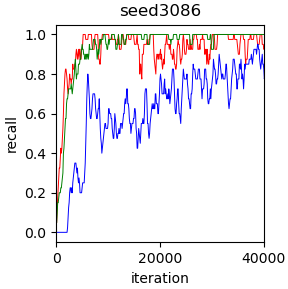} \\
 \includegraphics[height=25mm]{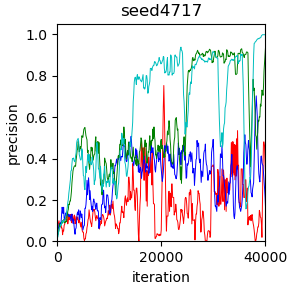} &
 \includegraphics[height=25mm]{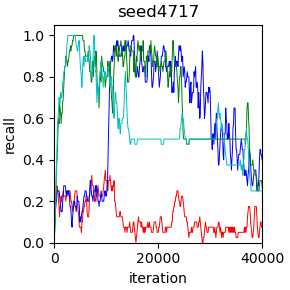}    &
 \includegraphics[height=25mm]{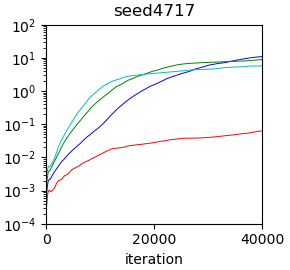}        &
 \includegraphics[height=25mm]{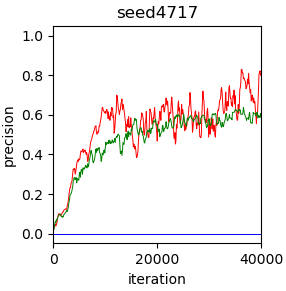}     &
 \includegraphics[height=25mm]{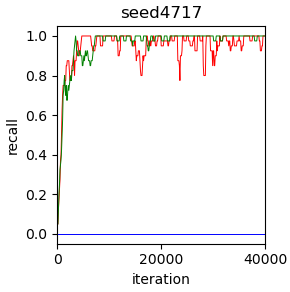} \\
 \includegraphics[height=25mm]{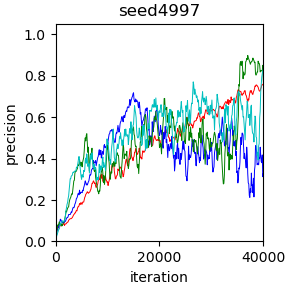} &
 \includegraphics[height=25mm]{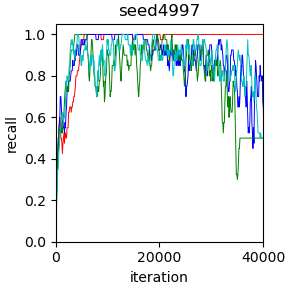}    &
 \includegraphics[height=25mm]{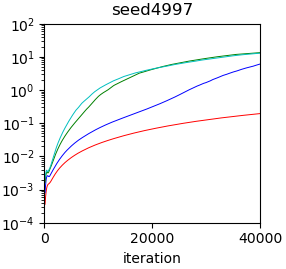}        &
 \includegraphics[height=25mm]{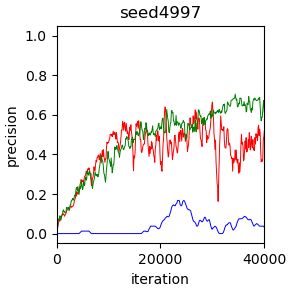}     &
 \includegraphics[height=25mm]{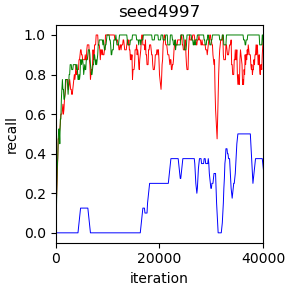} \\
 \includegraphics[height=25mm]{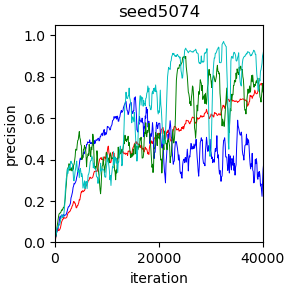} &
 \includegraphics[height=25mm]{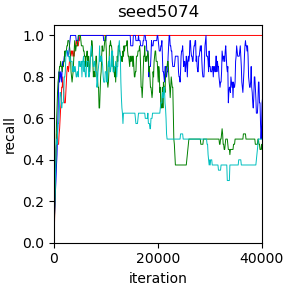}    &
 \includegraphics[height=25mm]{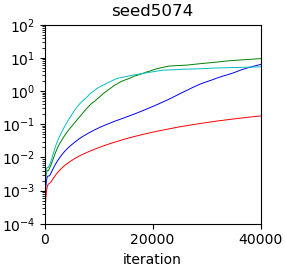}        &
 \includegraphics[height=25mm]{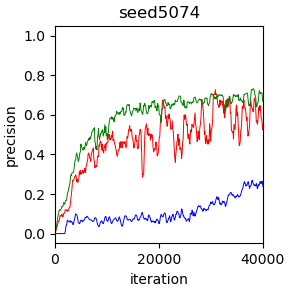}     &
 \includegraphics[height=25mm]{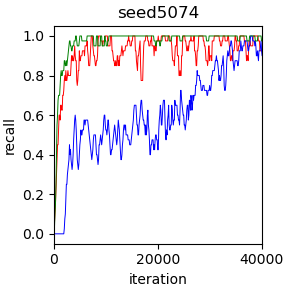} \\
 \includegraphics[height=25mm]{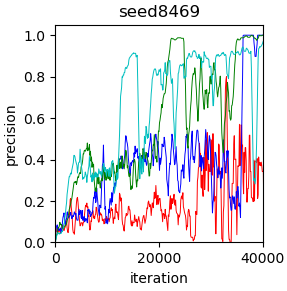} &
 \includegraphics[height=25mm]{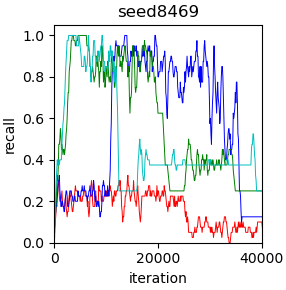}    &
 \includegraphics[height=25mm]{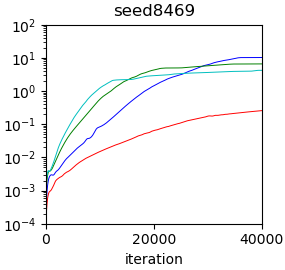}        &
 \includegraphics[height=25mm]{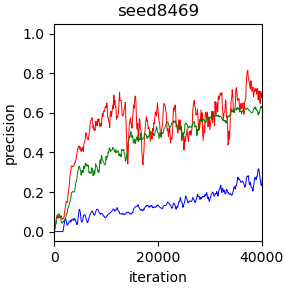}     &
 \includegraphics[height=25mm]{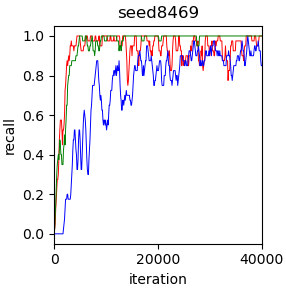} \\ 
{\small\begin{tabular}{l}(a) data scaling:\\ \phantom{(a) }precision \end{tabular}} & 
{\small\begin{tabular}{l}(b) data scaling:\\ \phantom{(b) }recall \end{tabular}} &  
{\small\begin{tabular}{l}(c) data scaling:\\ \phantom{(c) }norm of gradient \end{tabular}} &
{\small\begin{tabular}{l}(d) scaling strategy:\\ \phantom{(d)} precision \end{tabular}} &
{\small\begin{tabular}{l}(e) scaling strategy:\\ \phantom{(e)} recall \end{tabular}}
 \end{tabular}
\caption{
 (a)\,precision for each scale, (b)\,recall for each scale, 
 (c)\,averaged norm of discriminator's gradient, 
 (d)\,precision for each scaling strategy, and (e)\,recall for each scaling strategy 
 are depicted.  Each row corresponds to each seed. 
 The scale $s$ is set to $0.25, 0.5, 1,$ and $1.5$ for (a), (b), and (c). 
 As a scaling strategy,  ``fix'', ``linear const'', or ``adaptive'' is used in (d) and (e). 
 }
\label{appendix:seed-wise-simulations}
\end{figure}

\begin{figure}[tp]
 \centering
 \includegraphics[width=160mm]{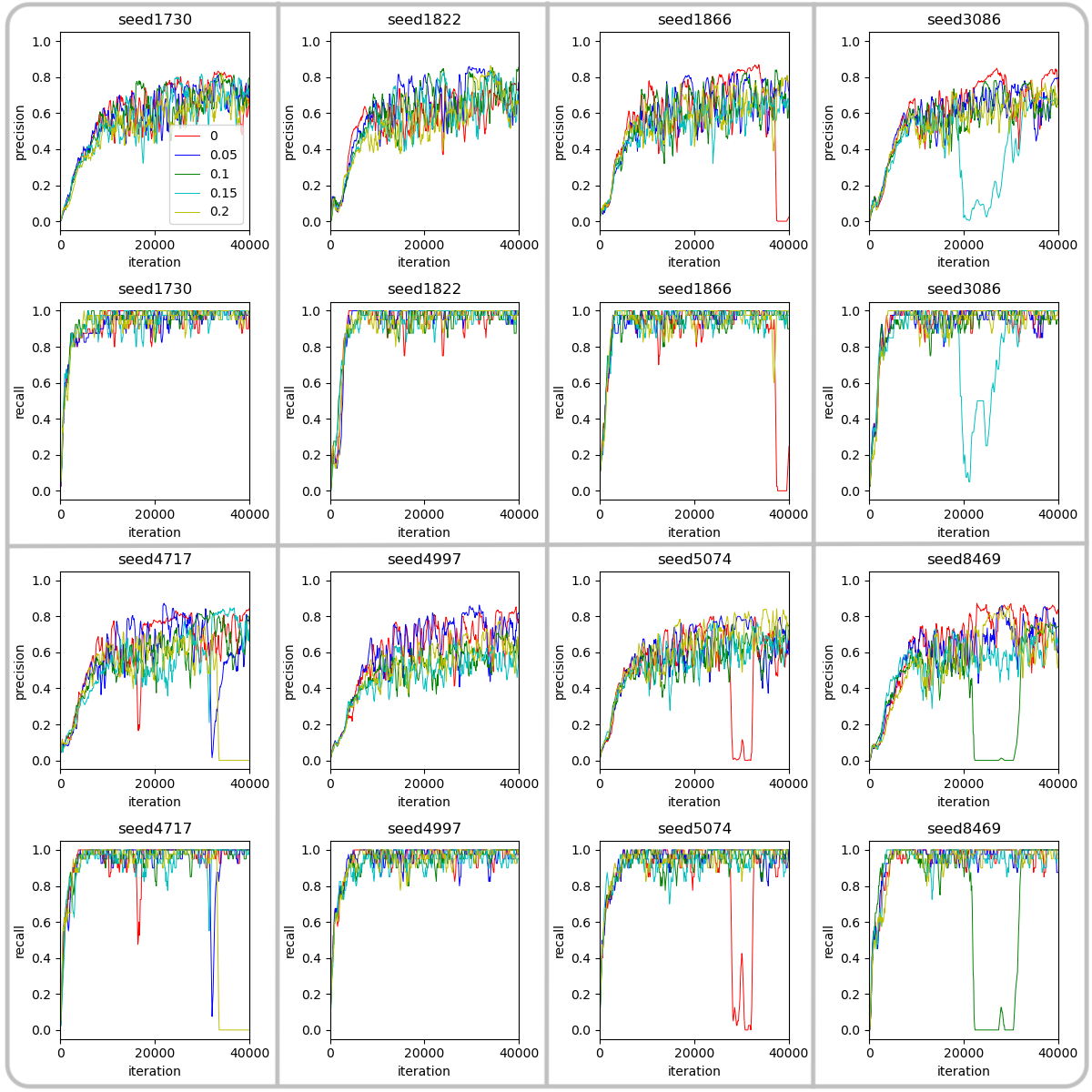}
 \caption{Each gray box corresponds to each seed. 
 Precision and recall for the learning with noise injection, 
 $\widetilde{\x}=\x+\sigma_{\mathrm{noise}}\bm{\epsilon}, \bm{\epsilon} \sim \mathcal{N}(\bm{0}, \sigma_{\text{noise}}^2I)$, 
 are depicted at each learning iteration. 
 The noise level $\sigma_{\mathrm{noise}}$ is set to $0, 0.05, 0.1, 0.15$, and $0.2$. 
 }
 \label{appendix:noise_scale_adding_noise_supp}
\end{figure}

\begin{figure}[tp]
 \centering
 \includegraphics[width=160mm]{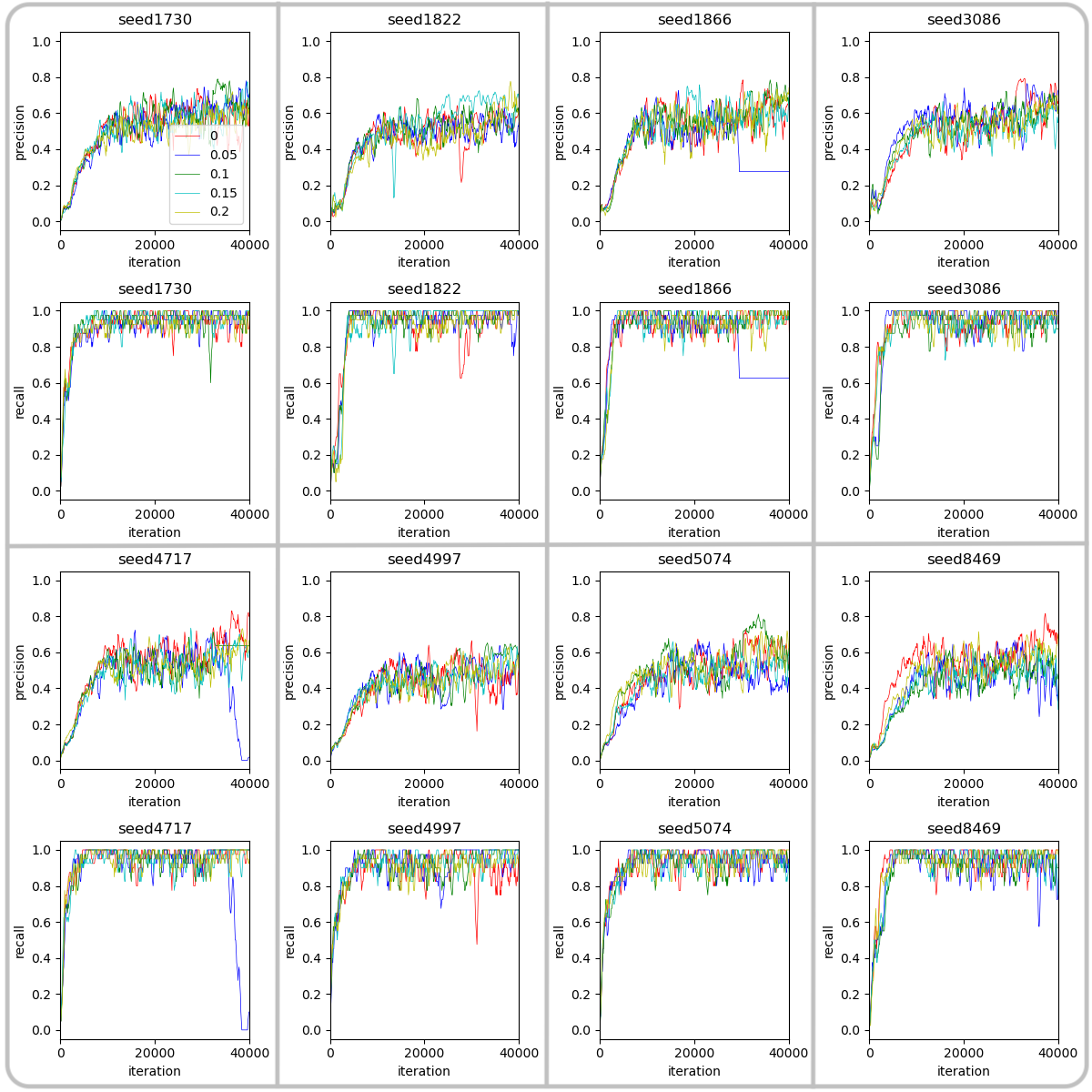}
 \caption{Each gray box corresponds to each seed. 
 Precision and recall for the learning with data scaling and noise injection,  
 $\widetilde{\x}=s_t\x+\sqrt{1-s_t^2}\sigma_{\mathrm{noise}}\bm{\epsilon}, \bm{\epsilon} \sim \mathcal{N}(\bm{0}, \sigma_{\text{noise}}^2I)$ 
 are depicted at each learning iteration
 The noise level $\sigma_{\mathrm{noise}}$ is set to $0, 0.05, 0.1, 0.15$, and $0.2$. 
 The scaling function $s_t$ is determined by the uniform distribution with adaptive strategy. }
 \label{appendix:diff_noise_scale_diffusion_supp}
\end{figure}

\section{Additional Numerical Studies to Section~\ref{Numerical_Experiments}}

\renewcommand{\theequation}{C.\arabic{equation}}
\renewcommand{\thefigure}{C.\arabic{figure}}
\renewcommand{\thetable}{C.\arabic{table}}
\renewcommand{\thetheorem}{C.\arabic{theorem}}
\setcounter{equation}{0}
\setcounter{figure}{0}
\setcounter{table}{0}


All experiments in this paper are conducted on one or two of the following GPUs: NVIDIA RTX A5000(24GB memory), RTX A6000(48GB memory), and Tesla V100S(32GB memory). 
When two GPUs are used in parallel, the same type of GPU is used. The training time is almost the same as Diffusion-GAN, 
as the computational cost of the scale-wise variance regularization has little impact. 
As for the libraries, 
the environment was built according to Diffusion-GAN's {\tt environment.yml.}

\subsection{Stability of Scale-wise Variance Regularization}
\label{appendix:Scale-wise_Variance_Reg}
We investigate the effects of scale-wise variance regularization using the same synthetic dataset as 
in Section~\ref{sec:Effect_of_Data_Scaling_and_NoiseInjection_in_GANs}. 
Fig~\ref{appendix:fig:pr_rc_diff_reg_original} 
shows the precision and recall when the scale-wise variance regularization is used for Scale-GAN. 
We observe that overall the regularized Scale-GAN with a positive $\lambda$ 
outperforms non-regularized Scale-GAN, i.e., $\lambda=0$ when the regularization parameter is appropriately
determined, such as $\lambda=0.1$ or $0.5$. 
However, too strong regularization like $\lambda=5, 10$ leads to instability. 
Indeed, the recall for the learning with $\lambda=5, 10$ is degraded, meaning that the generator misses some modes. 
\begin{figure}[tp]
 \centering
 \includegraphics[width=160mm]{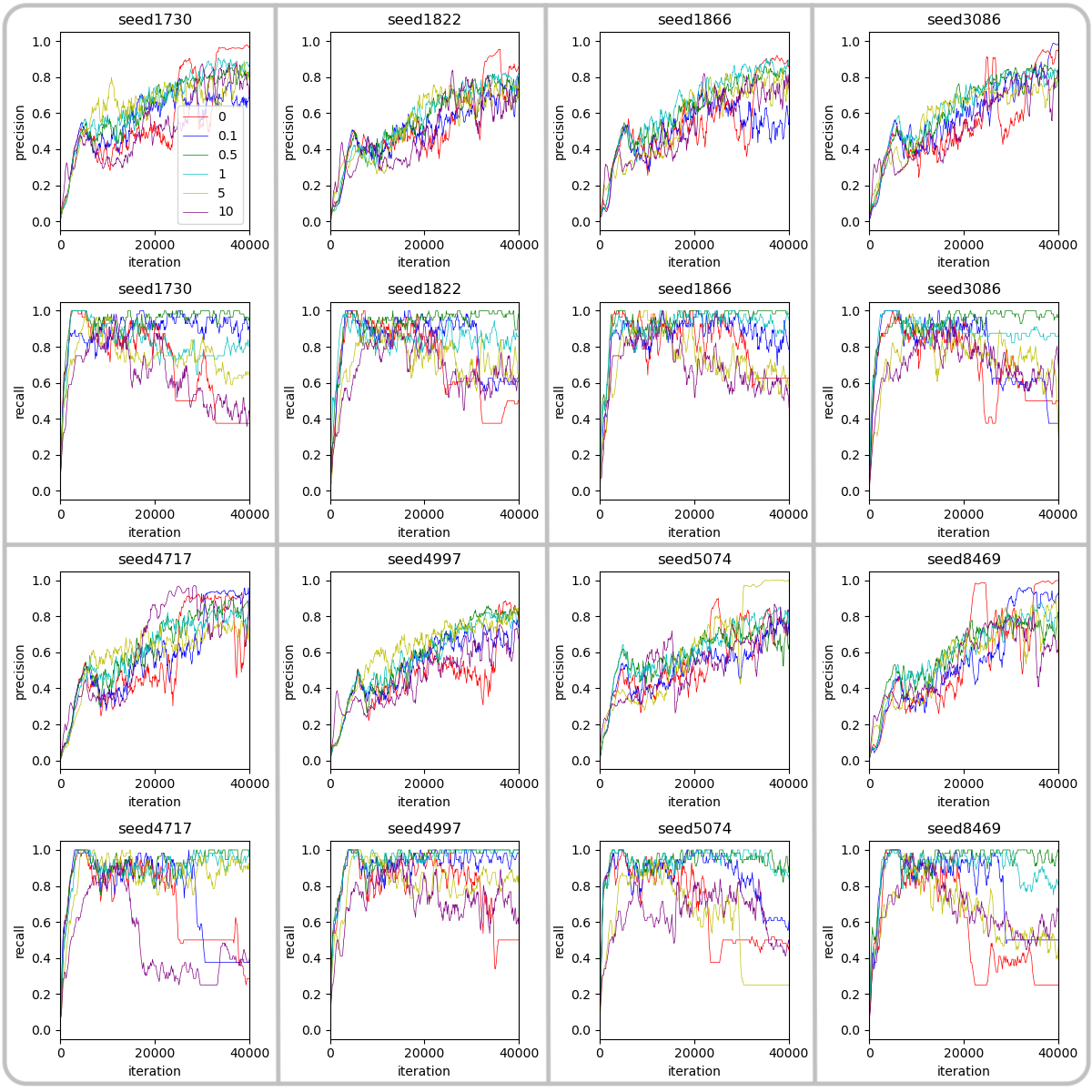}
 \caption{Each gray box corresponds to each seed. 
 Precision and recall for Scale-GAN with several regularization parameters are depicted at each learning iteration. 
 The regularization parameter $\lambda$ is set to $0, 0.1, 0.5, 1, 5$, and $10$. 
 The scaling function $s_t$ is determined by the uniform distribution with adaptive strategy. }
 \label{appendix:fig:pr_rc_diff_reg_original}
\end{figure}

\subsection{Hyperparameters in Section~\ref{subsec:Image_Generation}}
\label{appendix:Hyperparameters_Image_Generation}

Some hyper-parameters in Diffusion-GAN are defined as follows. 
\begin{itemize}
 \item $T$: The number of scales, $s_1,\ldots,s_T$, used in the learning process of Diffusion-GAN. 
 \item $T_{\max}, T_{\min}$: 
       In Diffusion-GAN, $T$ can vary during the learning process, and its range is from  
       $T_{\min}$ to $T_{\max}$ to  which are predefined constants. 
 \item $I$: 
       For the strategy of "linear const", $T$ at $i$-th step is determined by 
       $\\min\{i\ast 2T_{\max}/I, T_{\max}\}$. 
       The parameter $I$ controls the increasing speed of $T$. 
 \item $r_d$: An estimate of how much the discriminator overfits the data. 
 \item $d_{\mathrm{target}}$: The threshold to determine when the parameter $T$ increases in the "adaptive"
       strategy. 
\end{itemize}

Hyperparameters used in Diffusion-GAN and Scale-GAN are summarized in the following. 
\begin{itemize}
 \item Diffusion-GAN: for all three datasets, the following parameters are used: 
       \begin{itemize}
	\item priority distribution $\pi_0$ with adaptive strategy
	\item $\beta_0=0.0001, \beta_T=0.02, T_{\min}=5, T_{\max}=1000, \sigma_{\mathrm{noise}}=0.05$. 
       \end{itemize}
 \item Scale-GAN(proposed)
       \begin{itemize}
	\item CIFAR-10: 
	      \begin{itemize}
	       \item uniform distribution $\pi_0$ with adaptive strategy
	       \item $\beta_0=0.0001, \beta_T=0.028, T_{\min}=5, T_{\max}=1000, \lambda=0.1$. 
	      \end{itemize}

	\item STL-10 
	      \begin{itemize}
	       \item uniform distribution $\pi_0$ with adaptive strategy
	       \item $\beta_0=0.0001, \beta_T=0.03, T_{\min}=5, T_{\max}=1500, \lambda=0.05$. 
	      \end{itemize}

	\item LSUN-Bedroom 
	      \begin{itemize}
	       \item priority distribution $\pi_0$ with adaptive strategy
	       \item $\beta_0=0.0001, \beta_T=0.028, T_{\min}=5, T_{\max}=1000, \lambda=0.1$. 
	      \end{itemize}

       \end{itemize}
\end{itemize}
In Diffusion-GAN, the parameters are the same as those in~\citep{wang2022diffusion}. 
The parameter $d_{\mathrm{target}}$ in the adaptive strategy is set to $0.6$. 

The STL-10 dataset is known as a dataset with high variance, 
which is thought to be a reason that the larger scaling $\beta_T=0.03, T_{\text{max}}=1500$ works effectively. 

Fig.~\ref{CIFAR10_generated_images}, and \ref{STL10_generated_images} shows the generated sample images by Scale-GAN.

\begin{figure}[t]
  \begin{center}
    \includegraphics[width=120mm]{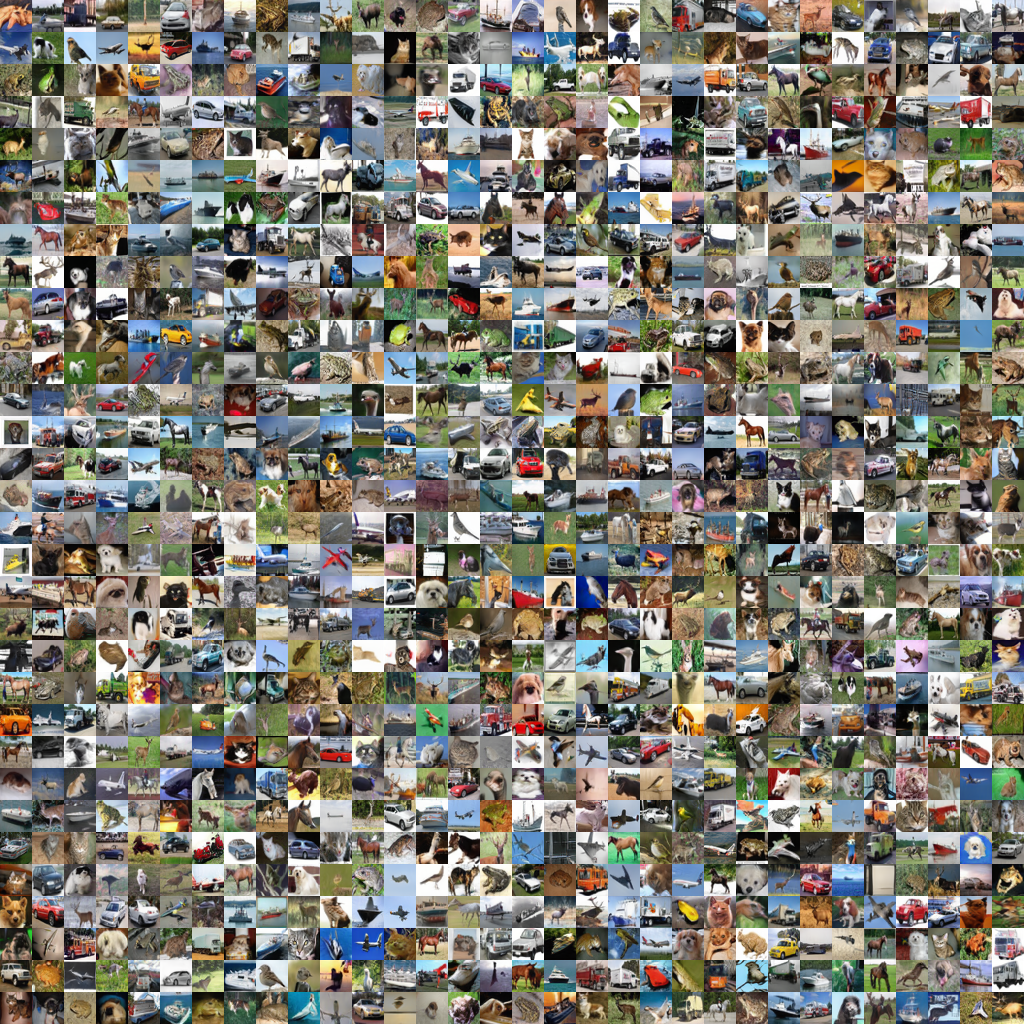}
    \caption{Generated sample images by Scale-GAN trained using CIFAR-10. }
   \label{CIFAR10_generated_images}
  \end{center}
\end{figure}

\begin{figure}[t]
  \begin{center}
    \includegraphics[width=120mm]{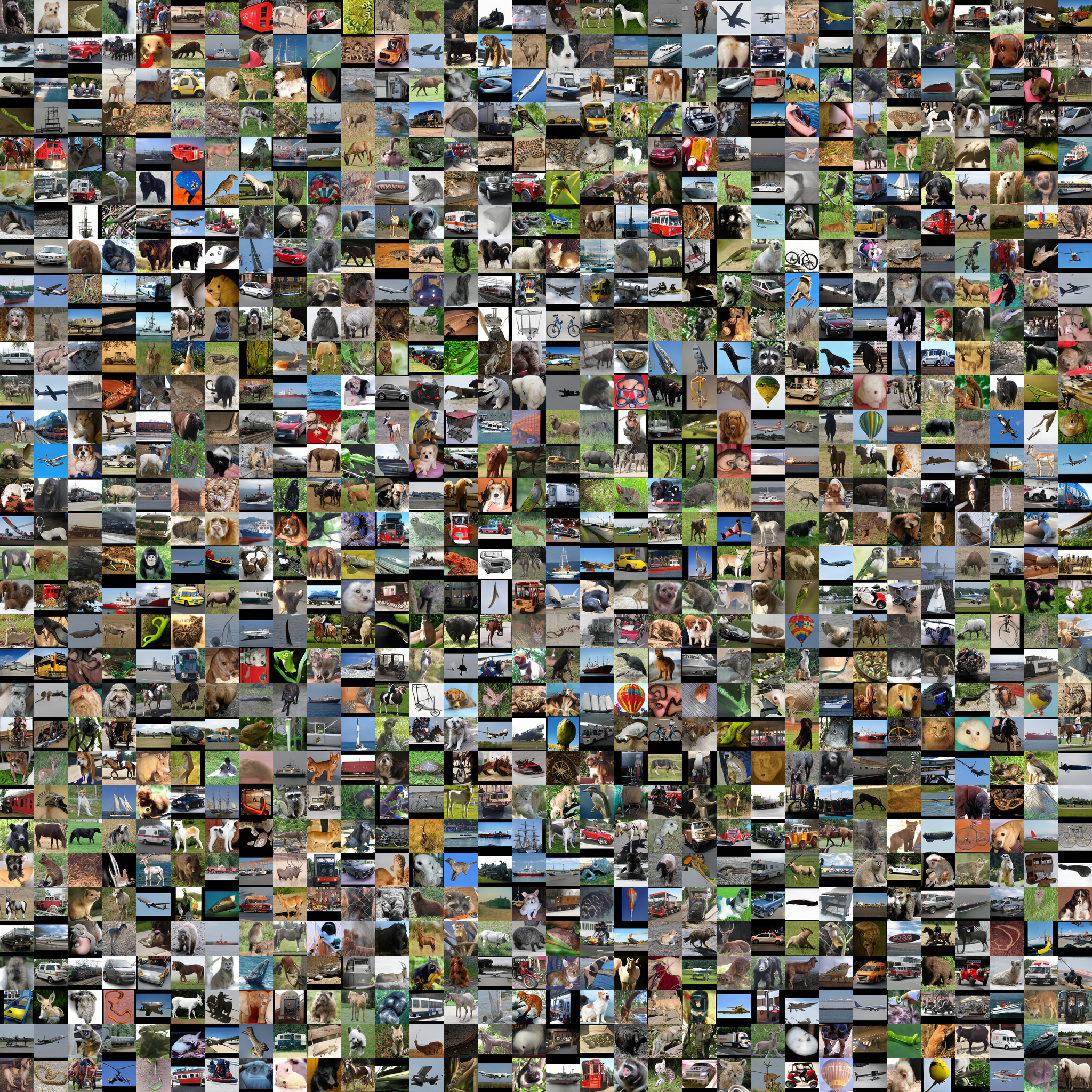}
   \caption{Generated sample images by Scale-GAN trained using STL-10. }
   \label{STL10_generated_images}
  \end{center}
\end{figure}

\subsection{Scale-GAN with Modified Regularization on CIFAR-10}
\label{appendix:subsec:modified_reg}

Let us investigate the effect of several regularization terms for Scale-GAN. 
The uniform distribution $\pi_0$ with the adaptive strategy is employed to determine the scaling function $s_t$. 

We compare 
Scale-GAN with scale-wise variance regularization ($\beta_T=0.02, \lambda=0.1$), i.e., the proposed method, 
Scale-GAN without regularization ($\beta_T=0.02, \lambda=0$), and 
Scale-GAN with modified regularization ($\beta_T=0.02, \lambda=0.1$). 
The scale-wise variance regularization is defined by the empirical approximation of $\lambda\mathbb{V}[D(\widetilde{\x},t)]$
for the joint distribution of the augmented data $\widetilde{\x}$ and the scale intensity $t$. 
On the other hand, the modified regularization is defined by the empirical approximation of $\lambda\mathbb{E}[(D(\widetilde{\x},t)-1/2)^2]$. 

Fig~\ref{appendix:fig:training_speed} shows the FID score at each learning iteration. 
The result indicates that the regularization efficiently attains a good FID score after a sufficient iteration. 
However, the convergence speed of Scale-GAN with modified regularization is slower than 
Scale-GAN with scale-wise variance regularization. 
Modified regularization is thought to make the adaptive strategy inefficient. 
More precisely, the expectation in $r_d=\mathbb{E}[\operatorname{sign}(\widetilde{D}(\y, t)-0.5)]$ 
for the adaptive strategy
\begin{align}
\label{eqn:adaptive_strategy} 
 T\leftarrow T+\operatorname{sign}(r_d-d_{\text {target }})
\end{align}
tends to be small, 
and the parameter $T$ does not take a large number, resulting in the scale intensity $t$ taking a small number.

\begin{figure}[tp]
 \centering 
 \includegraphics[height=80mm]{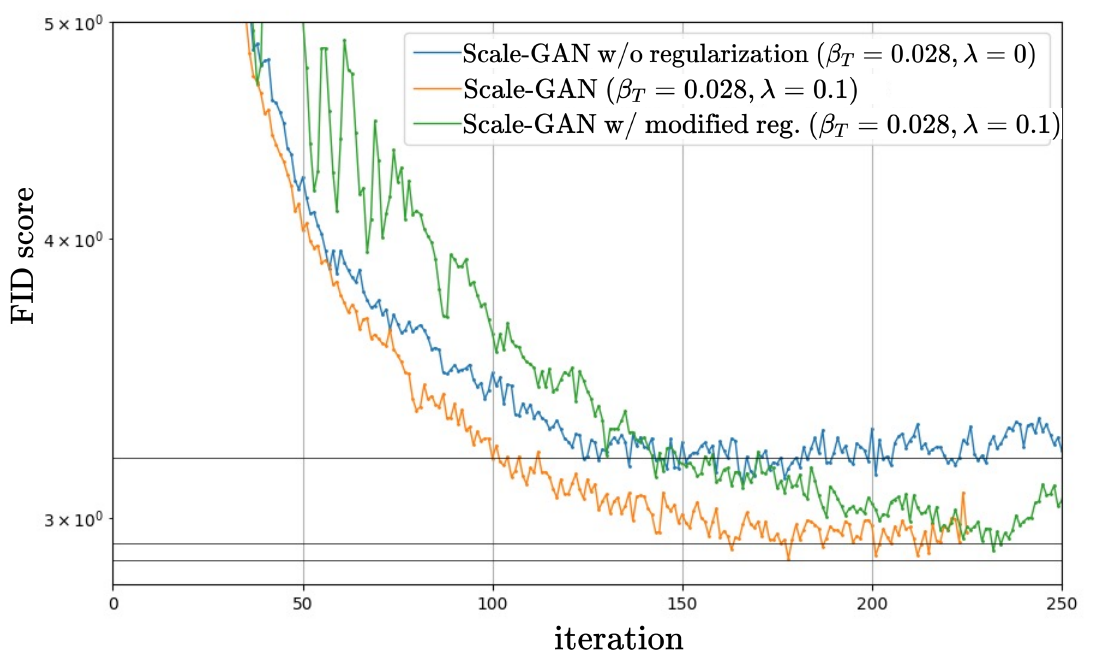}
 \caption{FID scores for Scale-GAN w/o regularization, Scale-GAN, and Scale-GAN with modified regularization 
 are depicted at each iteration for CIFAR-10 dataset. }
\label{appendix:fig:training_speed}
\end{figure}


\section{Supplementary of Proposed Framework}
\label{appendix:Proposed_Framework}

\renewcommand{\theequation}{D.\arabic{equation}}
\renewcommand{\thefigure}{D.\arabic{figure}}
\renewcommand{\thetable}{D.\arabic{table}}
\renewcommand{\thetheorem}{D.\arabic{theorem}}
\setcounter{equation}{0}
\setcounter{figure}{0}
\setcounter{table}{0}

For theoretical analysis, let us prepare the following notations. 
For $0\leq\delta<1/2$, let us define $\lambda_\delta=\delta^3/10$, and 
\begin{align*}
 \mathcal{U}_\delta&=\{f\in L^\infty\,|\,\delta<\inf f, \sup{f}<1-\delta\},\\
 \Lambda_\delta&=[0,\lambda_\delta]\subset\Rbb,\\
 \mathcal{Q}_{\delta,C}
 &=\left\{
 q\in L^1\cap L^\infty\,\bigg|\,\frac{p_0}{p_0+q}\in \mathcal{U}_\delta,\,
 \|q\|_{\mathrm{Lip}}\leq C,\, \int q\rmd{\mu}=1
 \right\}, \\
 \mathcal{D}_\delta
 &=\left\{D\in L^\infty\,\big|\,
 \delta-5\lambda_\delta \bigg(\frac{1-\delta}{\delta}\bigg)^2
 < D < 1-\delta+5\lambda_\delta \bigg(\frac{1-\delta}{\delta}\bigg)^2 \right\}. 
\end{align*}
For the probability density $p_0(\x)$, the function 
$q$ in $\mathcal{Q}_{\delta,C}$ should be non-negative, and thus 
$q$ is a probability density. 
When $p_0\in L^\infty$ and $\|p_0\|_{\mathrm{Lip}}\leq C$, we have $p_0\in \mathcal{Q}_{\delta,C}$. 
By the definition of $\lambda_\delta$, we see that $0<D<1$ holds for $D\in\mathcal{D}_\delta, 0\leq \delta<1/2$. 
Throughout the paper, all probability densities are strictly positive on their domain.

\subsection{Learning Algorithm: Scale-GAN}
\label{appendix:Learning_Algorithm_Scale-GAN}

The learning algorithm of Scale-GAN is similar to Diffusion-GAN~\citep{wang2022diffusion}. 
The main difference is the data augmentation. In Scale-GAN, only the data scaling is employed, while
Diffusion-GAN used the noise injection in addition to data scaling. 
For i.i.d. training data $\x_1,\ldots,\x_n\sim p_0$ and i.i.d. scales $t_1,\ldots,t_n\sim \pi$, 
the empirical approximation of \eqref{eqn:Scale-GAN-loss} is given by the min-max optimization problem, 
\begin{align}
 \min_{G\in\mathcal{G}}\max_{\widetilde{D}\in\widetilde{\mathcal{D}}}
 \frac{1}{n}\sum_{i}\log \widetilde{D}(s_{t_i}\x_i, t_i)
 +
 \frac{1}{n}\sum_{i}\log(1-\widetilde{D}(s_{t_i}G(\z_i), t_i))
 -\lambda \widetilde{\mathbb{V}}[\widetilde{D}]
 \label{appendix_eqn:empirical-Scale-GAN}, 
\end{align}
where 
\begin{align*}
\widetilde{\mathbb{V}}[\widetilde{D}]=\frac{1}{n}\sum_{j=1}^{n}
\left\{
 \frac{1}{n}\sum_{i}(\widetilde{D}(s_{t_j}\x_i,t_j)-\frac{1}{n}\sum_{i'}\widetilde{D}(s_{t_{j}}\x_{i'},t_{j}))^2
\right\}
\end{align*}
is an empirical approximation of $\mathbb{E}_{\pi}[\mathbb{V}_{p_0}[\widetilde{D}]]$. 
The pseudo code of Scale-GAN is shown in Algorithm~\ref{appendix:alg:Scale-GAN}. 

\begin{figure}[!t]
  \begin{algorithm}[H]
    \caption{{\bf :} Scale-GAN}
   \label{appendix:alg:Scale-GAN}
    \begin{algorithmic}
     \State {\bf Initialization}: DNNs $G_u, D_\theta$, distribution of latent variable $p_z$, 
     scaling function $s_t$ and \\
     $\phantom{Initialization: }$ intensity distribution $\pi(t)$ with 
     hyper-parameters $\beta_0, \beta_T, T_{\min}, T_{\max}, d_{\text{target}}$, etc. 
     \While{$i\leq $ number of training iterations}
     \State \hspace*{-3mm}$\#$\emph{Step I: Update Discriminator}
     \begin{itemize}
      \item Sample minibatch of $m$ noise samples $\z_1,\ldots,\z_m\sim_{i.i.d.} p_z$.
      \item Obtain generated sample $\y_i=G_u(\z_i),\, i=1,\ldots,m$. 
      \item Sample minibatch of $m$ data $\x_1,\ldots,\x_m$. 
      \item Sample scale intensity $t_1,\ldots,t_m\sim \pi$. 
      \item Update discriminator $D_\theta$ by maximizing 
	    the objective function of~\eqref{appendix_eqn:empirical-Scale-GAN} for the fixed $G_u$. 
     \end{itemize}
     \vspace*{3mm}
     \State \hspace*{-3mm}$\#$\emph{Step II: Update Generator}
     \begin{itemize}
      \item Sample minibatch of $m$ noise samples $\z_1,\ldots,\z_m\sim_{i.i.d.} p_z$.
      \item Obtain generated sample $\y_i=G_u(\z_i),\, i=1,\ldots,m$. 
      \item Sample scale intensity $t_1,\ldots,t_m\sim \pi$.       
      \item Update discriminator by minimizing the objective function of~\eqref{appendix_eqn:empirical-Scale-GAN} 
	    for the fixed $D_\theta$.
     \end{itemize}
     \vspace*{3mm}
     \State \hspace*{-3mm}$\#$\emph{Step III: Update scaling strategy}
       \If{ $i\mod 4==0$}
       \State \quad Update $T$ by \eqref{eqn:adaptive_strategy}. 
       \State \quad Update the distribution of the scale intensity $\pi(t)$ using updated $T$. 
       \EndIf
     \EndWhile 
    \end{algorithmic}
  \end{algorithm}
\end{figure}

%
%

\subsection{Proof of Theorem~\ref{thm:opt-discriminator}}
\label{appendix:Theorem_scale_invariance}

\subsubsection*{Problem Setting}
For the probability density $q$ on $\mathcal{X}$ and 
the discriminator $D:\mathcal{X}\rightarrow(0,1)$, let us define 
the regularized loss by 
\begin{align*}
 F_{p_0}(q,D;\lambda):=L_{p_0}(q,D)-\lambda \mathbb{V}_{p_0}[D], \ \ 
 L_{p_0}(q,D):=\mathbb{E}_{p_0}[\log D]+\mathbb{E}_{q}[\log(1-D)]. 
\end{align*}
For the probability density $p_0(\x),\, \x\in\mathcal{X}$, 
let $p_t(\y)$ be the probability density of the scaled variable $\y = s_t \x$ on 
$s_t\mathcal{X}=\{s_t\x\,|\,\x\in\mathcal{X}\}$, 
where $s_t\in(0,1]$ is the scaling function with $s_0=1$. 
More concretely, $p_t(\y)=p_0(\y/s_t)/s_t^d$ holds. 
Likewise, for the probability density $q$ of the generator, 
the probability density $q_t$ is defined by $q_t(\y)=q(\y/s_t)/s_t^d$. 
Note that $q_0=q$ holds. 
Using a pre-specified probability density $\pi(t)>0$ for $t\in[0,T]$, 
Scale-GAN with the variance regularization is formulated as the min-max problem, 
\begin{align}
 \label{appendix:reg-scale-GAN}
 \min_{\{q_t\}_{t\in[0,T]}}\max_{\{D_t\}_{t\in[0,T]}} \mathbb{E}_{t\sim\pi}
\big[ L_{p_t}(q_t,D_t) - \lambda\mathbb{V}_{p_t}[D_t]\big],
\end{align}
for $D_t(\cdot) = \widetilde{D}(\cdot,t)$. 
The objective function of \eqref{appendix:reg-scale-GAN} is expressed as
$\mathbb{E}_{t\sim\pi}[F_{p_t}(q_t,D_t;\lambda)]$.


\subsubsection*{Theoretical Analysis of Optimal Discriminator}
We consider the optimal solution of \eqref{appendix:reg-scale-GAN} under the condition 
$\|p_0\|_\infty<\infty, q\in \mathcal{Q}_{\delta,\infty}$ and $D\in \mathcal{U}_0$. 
Note that the assumption $\frac{p_0}{p_0+q}\in \mathcal{U}_\delta$ ensures 
$\frac{p_t}{p_t+q_t}\in \mathcal{U}_\delta$ for $t\geq0$. 
Under the assumption, $p_0$ and $q$ are positive almost everywhere with respect to $\mu$. 
Note that the set $\mathcal{U}_\delta$ is convex for $0\leq\delta<1/2$. 


We confirm the strict concavity of the functional $D\mapsto F_{p_0}(q,D;\lambda)$ for $D\in\mathcal{U}_0$. 
The concavity of the functional $D\rightarrow L_{p_0}(q,D)$ and 
the convexity of the variance $D\mapsto\mathbb{V}_{p_0}[D]$ are clear. 
For $D_1, D_2\in\mathcal{U}_0$ such that $\mu(D_1\neq D_2)>0$, 
the set 
\begin{align*}
\{\x\in\mathcal{X}\,|\,\log(\alpha D_1(\x)+(1-\alpha) D_2(\x)) > \alpha\log D_1(\x)+(1-\alpha)\log D_2(\x)\}
\end{align*}
has a positive measure for $0<\alpha<1$. The same property holds for the function $D\mapsto\log(1-D)$. 
Therefore, the strict concavity of $D\mapsto F_{p_0}(q,D;\lambda)$ holds when 
we ignore the difference of $D$ on a set of measure zero. 

Let $\rmd{F}_{p_0}(q,D;\lambda)(u)$ be the G\^{a}teaux differential~\citep{kurdila2005convex}
of $F_{p_0}(q,\cdot;\lambda)$ 
at $D$ to the direction $u\in L^\infty$, i.e., 
\begin{align*}
 \rmd{F}_{p_0}(q,D;\lambda)(u)
 &:=\lim_{\epsilon\rightarrow+0}\frac{F_{p_0}(q,D+\epsilon u;\lambda)-F_{p_0}(q,D;\lambda)}{\epsilon}. 
\end{align*}
Note that $D+\epsilon u\in \mathcal{U}_0$ holds for sufficiently small $\epsilon>0$. 
Then, Lebesgue's dominated convergence theorem ensures that 
\begin{align*}
 \rmd{F}_{p_0}(q,D;\lambda)(u)
 &\phantom{:}=
  \int u
 \left\{\left(  \frac{p_0}{D}-\frac{q}{1-D}  \right)-2\lambda p_0 (D-\mathbb{E}_{p_0}[D])  \right\}\rmd{\mu}. 
\end{align*}
If $D\in \mathcal{U}_0$ attains the maximum value of $\max_{D\in \mathcal{U}_0}F_{p_0}(q,D;\lambda)$, 
the equality $\rmd{F}_{p_0}(q,D_0;\lambda)(u)=0$ should hold for any $u\in L^\infty$. 
The assumption on $p_0$ and $q$ leads that the inside of $\{\cdots\}$ 
in the G\^{a}teaux differential is included in $L^\infty\cap L^1$. 
By setting $u=(\frac{p_0}{D}-\frac{q}{1-D})-2\lambda p_0(D-\mathbb{E}_{p_0}[D])\in L^\infty\cap L^1$, 
we see that the optimal $D$ should satisfy 
\begin{align}
&\phantom{\Longleftrightarrow}
 \left(\frac{p_0}{D}-\frac{q}{1-D}  \right)-2\lambda p_0 (D-\mathbb{E}_{p_0}[D])=0 \nonumber\\
\label{eqn:extremal-cond}
&\Longleftrightarrow\ 
2 \lambda D^3 -2\lambda  (\mathbb{E}_{p_0}[D]+1)D^2
 + \left(2\lambda\mathbb{E}_{p_0}[D]-1-\frac{q}{p_0}\right)D+1=0
\end{align}
almost everywhere. Due to the concavity, we see that the discriminator $D\in \mathcal{U}_0$ satisfying
\eqref{eqn:extremal-cond} attains the global maximum value of $\max_{D\in \mathcal{U}_0}F_{p_0}(q,D;\lambda)$. 
Furthermore, the strict concavity ensures the uniqueness of the optimal solution up to the difference on a set of measure zero. 
The same conclusion holds for the optimal solution of $\max_{D\in \mathcal{U}_0}F_{p_t}(q_t,D;\lambda)$. 

Let us prove the existence of the optimal solution. 
Since $L^\infty$ is not reflexive, the general existence theorem discussed by~\cite{kurdila2005convex} does not apply. 
Our proof utilizes some specific properties of the functional $F_{p_0}(q,D;\lambda)$. 
\begin{lemma}[Existence and uniqueness of optimal solution]
\label{appendix:lemma:Existence_Uniqueness_OptSol}
Assume $\|p_0\|_\infty<\infty$ and $q\in\mathcal{Q}_{\delta,\infty}$. 
 Then, the optimal solution of the optimization problem, 
$\max_{D\in \mathcal{U}_0}F_{p_0}(q,D;\lambda)$, uniquely exists for $\lambda\geq0$. 
\end{lemma}
\begin{proof}
For $\lambda=0$, the optimality condition \eqref{eqn:extremal-cond} leads that 
$D=\frac{p_0}{p_0+q}\in \mathcal{U}_\delta$. 
In the below, we assume $\lambda>0$. 
Define the real-valued function $f(z,c)$ for $z,c\in[0,1]$ by
\begin{align}
 \label{eqn:f(z,c)}
 f(z,c)=2\lambda z^3- 2 \lambda (c+1) z^2+(2\lambda c-1-r)z+1. 
\end{align}
Equation~\eqref{eqn:extremal-cond} is expressed as 
$f(D(\x),\mathbb{E}_{p_0}[D])=0$ with $r=q(\x)/p_0(\x)$ almost everywhere. 
For any fixed $\x$, we have $\frac{1}{1+r}\in (\delta,1-\delta)$. 
For a fixed constant $c\in[0,1]$, let us define $z_c\in[0,1]$ as a real number satisfying $f(z_c,c)=0$. 
As $f(0,c)=1$ and $f(1,c)=-r<0$, the existence of $z_c$ in $(0,1)$ is guaranteed by the intermediate value theorem. 
Since $\lim_{z\rightarrow-\infty}f(z,c)=-\infty$ and 
$\lim_{z\rightarrow\infty}f(z,c)=\infty$, we see that 
each solution of the cubic equation $f(z,c)=0$ for $z$ lies on each interval, $(-\infty,0), (0,1)$ and $(1,\infty)$. 
Hence, $z_c$ is the unique solution of $f(z,c)=0$ on the interval $(0,1)$. 
Suppose that $f(z_c(\x),c)=0$ holds for $r=q(\x)/p_0(\x)$. 
Since $z_c$ continuously depends on $r$, $z_c(\x)$ is a measurable function. 

Next, let us prove that there exists $c$ such that $c=\mathbb{E}_{p_0}[z_c(\x)]$ holds. 
If this equality holds, we see that there exists $D$ such that $f(D(\x),\mathbb{E}_{p_0}[D])=0$ holds 
for $r=q(\x)/p_0(\x)$. 
For the function $h(c)=\mathbb{E}_{p_0}[z_c(\x)]-c$, 
the continuity and differentiability of $h(c)$ follow Lebesgue's dominated convergence theorem. 
Indeed, the continuity of $\mathbb{E}_{p_0}[z_c]$ 
follows the continuity of $z_c$ for $c$ and the boundedness $0<z_c(\x)<1$. 
Hence, we have $h(0)>0$ and $h(1)<0$. 
The above argument guarantees the existence of the optimal $D\in\mathcal{U}_0$. 
If there are multiple roots for the equation $h(c)=0, c\in[0,1]$, 
there are multiple $D$s with different expectations
satisfying the optimality condition, $f(D(\x),\mathbb{E}_{p_0}[D])=0$ with $r=q(\x)/p_0(\x)$ (a.e.). 
This contradicts the strict concavity of $F_{p_0}(q,D;\lambda)$ in $D$. 
Therefore, the uniqueness is guaranteed. 
\end{proof}

Next, let us consider the optimal solution of $\min_{D\in \mathcal{U}_0}F_{p_t}(q_t,D;\lambda)$ for $t>0$. 
\begin{lemma}
\label{appendix:lemma:opt_Dt}
Assume $\|p_0\|_\infty<\infty$ and $q\in\mathcal{Q}_{\delta,\infty}$. 
Suppose that $D_t$ is the optimal solution of $\min_{D\in \mathcal{U}_0}F_{p_t}(q_t,D;\lambda)$
for $t\geq0$. Then, $D_t(s_t\x) = D_0(\x)$ (a.e.) holds. 
\end{lemma}
\begin{proof}
 The optimality condition is $\rmd{F}_{p_t}(q_t,D_t;\lambda)(u)=0$ for any $u\in L^\infty$, which leads to 
 \begin{align*}
  2\lambda D_t^3-2\lambda(\mathbb{E}_{p_t}[D_t]+1) D_t^2
  +
  \bigg(  2\lambda\mathbb{E}_{p_t}[D_t]-1-\frac{q_t}{p_t}  \bigg)D_t+1=0. 
 \end{align*}
 Note that $\frac{p_0}{p_0+q}\in \mathcal{U}_\delta$ leads to $\frac{p_t}{p_t+q_t}\in \mathcal{U}_\delta$. 
 The existence and uniqueness of the optimal solution $D_t\in \mathcal{U}_0$ is guaranteed 
 by the same argument in Lemma~\ref{appendix:lemma:Existence_Uniqueness_OptSol}. 
 Since $q_t(\y)/p_t(\y)=q(\x)/p_0(\x)$ holds for $\y=s_t\x$, 
 the above optimality condition is equivalent with 
\begin{align*}
 &\phantom{+} 2\lambda D_t(s_t\x)^3-2\lambda(\mathbb{E}_{\x\sim p_0}[D_t(s_t\x)]+1) D_t(s_t\x)^2\\
 &\quad +\bigg(  2\lambda\mathbb{E}_{\x\sim p_0}[D_t(s_t\x)]-1-\frac{q(\x)}{p_0(\x)}  \bigg)D_t(s_t\x)+1=0. 
\end{align*}
 Therefore, $D_t(s_t\x) = D_0(\x)$ holds almost everywhere. This means $\widetilde{D}(s_t\x,t)=\widetilde{D}(\x,0)$~(a.e.). 
\end{proof}

\subsection{Proof of Theorem~\ref{thm:bias_Scale-GAN}}
\label{appendix:Theorem_bias_Scale-GAN}

The formal statement of Theorem~\ref{thm:bias_Scale-GAN} is the following. 
The probability density $q_t$ defined from $q$ and the set $\mathcal{Q}_{\delta,C}$ are introduced 
in Section~\ref{appendix:Theorem_scale_invariance} and Section~\ref{appendix:Proposed_Framework}, respectively. 
\begin{theorem}
 \label{appendix:thm:bias_Scale-GAN}
 Suppose $\|p_0\|_\infty<\infty$ and $\|p_0\|_{\mathrm{Lip}}<C$. 
 Let $q_\lambda$ be the optimal solution of 
 \begin{align*}
  \min_{q\in\mathcal{Q}_{\delta,C}} \max_{\{D_t\}\subset\mathcal{U}_0} \mathbb{E}_{t\sim \pi}[F_{p_t}(q_t,D_t;\lambda)]
 \end{align*}
 for a fixed $\delta\in[0,1/2)$. Then, $\|q_\lambda-p_0\|_\infty=O(\lambda^{\frac{1}{d+3}})$ holds. 
\end{theorem}

Let us consider the solution of the generator. Lemma~\ref{appendix:lemma:opt_Dt} yields that 
\begin{align*}
 \max_{\{{{D}_t}\}_{t\geq0}\subset \mathcal{U}_0} 
 \mathbb{E}_{t\sim \pi}[L_{p_t}(q_t,{D}_t)-\lambda \mathbb{V}_{p_t}[{D}_t]\mid t ]
 = 
 \max_{D\in \mathcal{U}_0} L_{p_0}(q,D)-\lambda \mathbb{V}_{p_0}[D]. 
\end{align*}
Let us define $q_\lambda$ as the optimal solution of the min-max optimization problem, 
\begin{align*}
 \min_{q\in \mathcal{Q}_{\delta,C}}\max_{D\in \mathcal{U}_0} L_{p_0}(q,D)-\lambda \mathbb{V}_{p_0}[D].
\end{align*}
The min-max problem with $\lambda=0$ is nothing but the vanilla GAN problem. 
Hence, we have $q_0=p_0$. 

For $q\in \mathcal{Q}_{\delta,\infty}$, 
we consider the interval on which the optimal discriminator $D$ of 
$\max_{D\in \mathcal{U}_0}F_{p_0}(q,D;\lambda)$ exists. 
The derivatives of $f$ in \eqref{eqn:f(z,c)} are 
\begin{align*}
 \frac{\partial{f}}{\partial z}(0,c)&=-1-r+2c\lambda \in [-1-r,-1-r+2\lambda],\\
 \frac{\partial{f}}{\partial z}(1,c)&=-1 - r +2 \lambda(1 - c) \in [-1-r, -1-r+2\lambda],\\
 \frac{\partial^2{f}}{\partial z^2}(z,c) &=4\lambda(3z-1-c) \Longrightarrow -8\lambda \leq 
 \frac{\partial^2{f}}{\partial z^2}(z,c)\leq 8\lambda. 
\end{align*}
Therefore, for $0< z< 1$ and $0<c<1$, the cubic function $f(z,c)$ of $z$ is bounded from below and above by
the following quadratic functions of $z$, 
\begin{align*}
 -4\lambda z^2-(1+r)z+1 \leq f(z,c) \leq 4\lambda (z-1)^2-(r+1)(z-1)-r, 
\end{align*}
where the intercept of the quadratic functions are determined by $f(0,c)=1$ and $f(1,c)=-r$. 
By considering the zero of the quadratic functions, 
we see that 
$D(\x)\in(0,1)$ such that $f(D(\x),\mathbb{E}_{p_0}[D])=0$ for $r=q(\x)/p_0(\x)$ 
is bounded as follows, 
\begin{align*}
 0< 
 \frac{-(r+1)+\sqrt{(1+r)^2+16 \lambda }}{8 \lambda }
 \leq  D \leq 
 \frac{r+1+8 \lambda -\sqrt{(1+r)^2+16 \lambda  r}}{8 \lambda }
 <1. 
\end{align*}
Using the inequality $\sqrt{1+x}\geq 1+\frac{x}{2}-\frac{x^2}{8}$ for $0\leq x\leq 3$, 
we have 
\begin{align*}
 \frac{p_0}{p_0+q}-4\lambda 
 \leq
 D
 \leq 
 \frac{p_0}{p_0+q}
 +4\lambda\left(\frac{q}{p_0}\right)^2
 \leq 
 \frac{p_0}{p_0+q}
 +4\lambda\left(\frac{1-\delta}{\delta}\right)^2
\end{align*}
for $0\leq \lambda\leq \frac{3\delta}{16(1-\delta)}$. 
In this case, we have 
\begin{align}
\label{eqn:eval_D-D0}
\bigg\|D-\frac{p_0}{p_0+q}\bigg\|_\infty\leq 4\lambda\bigg(\frac{1-\delta}{\delta}\bigg)^2. 
\end{align}
The above discussion ensures the following lemma. 
\begin{lemma}
 \label{lemma:region_optD}
Assume $\|p_0\|_\infty<\infty$ and $q\in\mathcal{Q}_{\delta,\infty}$ with $\delta>0$. 
Then, the optimal discriminator $D$ of $\max_{D\in\mathcal{U}_0}F_{p_0}(q,D;\lambda)$
satisfies  $D\in\mathcal{D}_\delta$ for $\lambda\in \Lambda_\delta$. 
\end{lemma}
\begin{proof}
[Proof of Lemma~\ref{lemma:region_optD}]
For $\lambda\in \Lambda_\delta$, we have $\lambda<\frac{3\delta}{16(1-\delta)}$. 
From the definition of $\mathcal{D}_\delta$, 
we see that the discriminator $D$ satisfying 
\eqref{eqn:eval_D-D0} 
is included in $\mathcal{D}_\delta$. 
\end{proof}

We assume that $\|p_0\|_\infty<\infty$ and $\|p_0\|_{\mathrm{Lip}}<C$ for $0<C<\infty$. 
Lemma~\ref{lemma:region_optD} ensures that for $\lambda\in \Lambda_\delta$ and 
$q\in \mathcal{Q}_{\delta,C}$, it holds that 
\begin{align*}
 \max_{D\in \mathcal{U}_0}F_{p_0}(q,D;\lambda) = \max_{D\in \mathcal{D}_\delta}F_{p_0}(q,D;\lambda). 
\end{align*}
Suppose that the domain of the functional, $(q,D,\lambda)\mapsto F_{p_0}(q,D;\lambda)$, is given by 
$(q,D,\lambda)\in \mathcal{H}_{\delta,C}:=\mathcal{Q}_{\delta,C}\times \mathcal{D}_\delta\times\Lambda_\delta$. 
The norm on $\mathcal{H}_{\delta,C}$ is denoted by 
$\|(q,D,\lambda)\|_\infty=\max\{\|q\|_\infty, \|D\|_\infty, |\lambda|\}$. 
One can find that $F_{p_0}$ is Lipschitz continuous on $\mathcal{H}_{\delta,C}$, 
because $\log{D}$ and $\log(1-D)$ are uniformly bounded on $\mathcal{D}_\delta$, i.e., 
$\sup_{D\in\mathcal{D}_\delta}\max\{|\log{D}|, |\log(1-D)|\}<\infty$. 
Clearly, $\mathcal{H}_{\delta,C}$ is a convex set. 
We see that $F_{p_0}(q,D;\lambda)$ is convex in $(q,\lambda)\in \mathcal{Q}_{\delta,C}\times \Lambda_\delta$ 
and concave in $D\in \mathcal{D}_\delta$. 

Let us define $G(q,\lambda)$ by 
\begin{align*}
 G(q,\lambda)=\max_{D\in \mathcal{D}_{\delta}} F_{p_0}(q,D;\lambda). 
\end{align*}
Then, $G(q,\lambda)$ is convex and Lipschitz continuous on 
$\mathcal{Q}_{\delta,C}\times \Lambda_{\delta}$. 
The convexity of $G(q,\lambda)$ follows from the standard argument of convex analysis. 
The Lipschitz continuity of $F_{p_0}(q,D;\lambda)$ leads to the Lipschitz continuity of 
$G(q,\lambda)$ as follows. 
\begin{align*}
 \sup_{D}F_{p_0}(q,D;\lambda)-\sup_{D'}F_{p_0}(q',D';\lambda')
& \leq 
 \sup_{D}F_{p_0}(q,D;\lambda)-F_{p_0}(q',D;\lambda') \\
& \leq \|F_{p_0}\|_{\mathrm{Lip}}\|(q-q',0,\lambda-\lambda')\|_\infty. 
\end{align*}
Also, $\sup_{D'}F_{p_0}(q',D';\lambda')-\sup_{D}F_{p_0}(q,D;\lambda)$ has the same upper bound. 

The optimal solution $q_\lambda$ is given by solving $\min_{q\in \mathcal{Q}_{\delta,C}}G(q,\lambda)$. 
The direct calculation leads that $q_0$ is uniquely determined as $q_0=p_0$. 
We prove that $q_\lambda$ is close to $p_0$ for sufficiently small $\lambda>0$. 
Let us define 
\begin{align*}
\bar{G}_\epsilon:=\inf\{G(q,0)\,|\,q\in\mathcal{Q}_{\delta,C},\, \|q-q_0\|_\infty=\epsilon\}. 
\end{align*}
\begin{lemma}
 \label{lemma:G_ineq}
 Suppose that $\|p_0\|_\infty<\infty, \|p_0\|_{\mathrm{Lip}}<C$ and $\delta>0$. 
 Then, $G(q_0,0)<\bar{G}_\epsilon\leq G(q,0)$ holds for any small $\epsilon>0$ and any $q\in\mathcal{Q}_{\delta,C}$ 
 such that $\|q-q_0\|_\infty\geq \epsilon$.  
\end{lemma}
Lemma~\ref{lemma:G_ineq} ensures that $\bar{G}_\epsilon$ is a monotone function of $\epsilon$. 
\begin{proof}
[Proof of Lemma~\ref{lemma:G_ineq}]
We prove the first inequality. 
Assume that $\|p_0\|_{\mathrm{Lip}}<C, q\in \mathcal{Q}_{\delta,C}$ and $\|q-q_0\|=\epsilon$. 
As $\frac{p_0}{p_0+q}\in\mathcal{U}_{\delta}\subset\mathcal{D}_\delta$, we have 
\begin{align*}
 G(q,0)
 =\int(p_0+q)\left\{\frac{p_0}{p_0+q}\log\frac{p_0}{p_0+q}+\frac{q}{p_0+q}\log\frac{q}{p_0+q} \right\} \mathrm{d}\mu. 
\end{align*}
Since $\|q_0-q\|_\infty=\|p_0-q\|_\infty=\epsilon$, there exists $\x_0$ in the interior of $\mathcal{X}$ such that
$q(\x_0)$ is nearly $p_0(\x_0)+\epsilon$ or $p_0(\x_0)-\epsilon$. 
The assumption $\|q\|_{\mathrm{Lip}}\leq C$ leads that 
\begin{align*}
 \|\x-\x_0\|<\frac{\epsilon}{5C} 
 &\Longrightarrow q(\x)>p_0(\x)+\frac{\epsilon}{2}\ \ \left(\text{or}\ q(\x)<p_0(\x)-\frac{\epsilon}{2}\right)\\
 &\Longrightarrow \frac{q(\x)}{p_0(\x)}>1+\frac{\epsilon}{2\|p_0\|_\infty}
 \ \  \left(\text{or}\ 
 \frac{q(\x)}{p_0(\x)}<1-\frac{\epsilon}{2\|p_0\|_\infty}\right). 
 \end{align*}
 Note that $p_0/(p_0+q)\in\mathcal{U}_{\delta}$ with $\delta>0$ leads that $p_0>0$. 
 Hence, for any point $\x$ such that $\|\x-\x_0\|<\frac{\epsilon}{5C}$, the inequality 
\begin{align*}
 \frac{p_0(\x)}{p_0(\x)+q(\x)}\log\frac{p_0(\x)}{p_0(\x)+q(\x)}
 +\frac{q(\x)}{p_0(\x)+q(\x)}\log\frac{q(\x)}{p_0(\x)+q(\x)}  > -\log{2}+C_{\delta,\epsilon,\|p_0\|_\infty}
\end{align*}
holds, 
where $C_{\delta,\epsilon,\|p_0\|_\infty}$ is a positive constant depending only on $\delta, \epsilon$
and $\|p_0\|_\infty$. 
Define $\mathcal{X}_0=\{\x\in\mathcal{X}\,|\,\|\x-\x_0\|<\epsilon/5C\}$. Then, 
\begin{align*}
 G(q,0)
 &=
 \int_{\mathcal{X}\setminus\mathcal{X}_0}
 (p_0+q)\left\{\frac{p_0}{p_0+q}\log\frac{p_0}{p_0+q}+\frac{q}{p_0+q}\log\frac{q}{p_0+q} \right\} \mathrm{d}\mu\\
 &\phantom{=}
 +\int_{\mathcal{X}_0}
 (p_0+q)\left\{\frac{p_0}{p_0+q}\log\frac{p_0}{p_0+q}+\frac{q}{p_0+q}\log\frac{q}{p_0+q} \right\} \mathrm{d}\mu\\
 &\geq 
 \int_{\mathcal{X}\setminus\mathcal{X}_0}
 (p_0+q)(-\log2)\mathrm{d}\mu + 
 \int_{\mathcal{X}_0} (p_0+q)(-\log 2+C_{\delta,\epsilon,\|p_0\|_\infty})\mathrm{d}\mu\\
 &= -2\log{2}+ C_{\delta,\epsilon,\|p_0\|_\infty}\int_{\mathcal{X}_0} (p_0+q)\mathrm{d}\mu
 = G(q_0,0)+ C_{\delta,\epsilon,\|p_0\|_\infty}\int_{\mathcal{X}_0}(p_0+q)\,\mathrm{d}\mu.
\end{align*}
When $q> p_0+\epsilon/2$ on $\mathcal{X}_0$, we have 
$\int_{\mathcal{X}_0} (p_0+q)\mathrm{d}\mu \geq  \int_{\mathcal{X}_0}
 (2p_0+\frac{\epsilon}{2})\mathrm{d}\mu\geq 
\frac{\epsilon}{2}\mu(\mathcal{X}_0)$. 
When $q< p_0-\epsilon/2$ on $\mathcal{X}_0$, we have 
$\int_{\mathcal{X}_0} (p_0+q)\mathrm{d}\mu \geq  \int_{\mathcal{X}_0} (2q+\frac{\epsilon}{2})\mathrm{d}\mu\geq 
\frac{\epsilon}{2}\mu(\mathcal{X}_0)$. 
Let $\mathrm{Vol}(1)$ be the volume of the unit ball in the $d$-dimensional Euclidean space. 
Then, for a small $\epsilon$, 
$C_{\delta,\epsilon,\|p_0\|_\infty}\int_{\mathcal{X}_0} (p_0+q)\mathrm{d}\mu$
is bounded below by a positive constant 
$\frac{\epsilon}{2}C_{\delta,\epsilon,\|p_0\|_\infty}\mathrm{Vol}(1)(\epsilon/5C)^d$ 
for $\mathcal{X}=\Rbb^d$
or 
$\frac{\epsilon}{2}C_{\delta,\epsilon,\|p_0\|_\infty}\mathrm{Vol}(1)(\epsilon/10C)^d$ 
for $\mathcal{X}=[0,1]^d$. 
Since the lower bound is independent of $q$, we have $\bar{G}_\epsilon-G(q_0,0)>0$. 

Let us prove the second inequality. 
For any $q\in\mathcal{Q}_{\delta,C}$ such that $\|q-q_0\|_\infty\geq \epsilon$, let us define 
$q''=\alpha q_0+(1-\alpha)q\in \mathcal{Q}_{\delta,C}$ such that $\|q''-q_0\|=\epsilon$. 
Always such an $\alpha\in[0,1)$ exists. 
If $G(q,0)<\bar{G}_\epsilon$ holds, 
we have $G(q'',0)\leq \alpha G(q_0,0)+(1-\alpha) G(q,0)<\bar{G}_\epsilon$, which 
contradicts the definition of $\bar{G}_\epsilon$. 
\end{proof}

\begin{proof}
[Proof of Theorem~\ref{appendix:thm:bias_Scale-GAN}]
For $q\in\mathcal{Q}_{\delta,C}$ such that $\|q-q_0\|_\infty=\epsilon$, 
it holds that 
$G(q_0,0)<\bar{G}_\epsilon\leq G(q,0)\leq G(q_0,0)+\|G\|_{\mathrm{Lip}}\epsilon$. 
Thus, $\bar{G}_\epsilon$ monotonically converges to $G(q_0,0)$ as $\epsilon\searrow 0$. 
If $\|q-q_0\|_\infty\geq\epsilon$, we have 
$G(q,\lambda)\geq  G(q,0)-\|G\|_{\mathrm{Lip}}\lambda\geq \bar{G}_\epsilon-\|G\|_{\mathrm{Lip}}\lambda$. 
On the other hand, for $\|q'-q_0\|_\infty<\lambda$, we have $G(q',\lambda)\leq G(q_0,0)+\|G\|_{\mathrm{Lip}}\lambda$. 
Suppose $\lambda<\frac{\bar{G}_\epsilon-G(q_0,0)}{2\|G\|_{\mathrm{Lip}}}$. 
Then, for any $q$ and $q'$ such that $\|q-q_0\|_\infty\geq \epsilon$ and $\|q'-q_0\|_\infty<\lambda$, we have 
\begin{align*}
 G(q',\lambda)\leq G(q_0,0)+\|G\|_{\mathrm{Lip}}\lambda < \bar{G}_\epsilon-\|G\|_{\mathrm{Lip}}\lambda \leq G(q,\lambda). 
\end{align*}
The above argument means that the optimal solution of $\min_{q\in\mathcal{Q}_{\delta,C}} G(q,\lambda)$ exists in 
$\|q-q_0\|_\infty\leq \epsilon_\lambda$, where 
\begin{align*}
 \epsilon_\lambda 
 :=\inf\bigg\{\epsilon>0\,\big|\,\lambda<\frac{\bar{G}_\epsilon-G(q_0,0)}{2\|G\|_{\mathrm{Lip}}}\bigg\}. 
\end{align*}
We can see $\epsilon_\lambda\searrow 0$ as $\lambda\searrow 0$. 
Hence, $\|q_\lambda-p_0\|_\infty=\|q_\lambda-q_0\|_\infty\rightarrow0$ as $\lambda\searrow 0$. 

For a small $\epsilon$, the constant $C_{\delta,\epsilon,\|p_0\|_\infty}$ is of the order $O(\epsilon^2)$ as the function of $\epsilon$. 
This is because the function 
$$r\ \longmapsto\  \frac{1}{1+r}\log\frac{1}{1+r}+\frac{r}{1+r}\log\frac{r}{1+r}$$ used in the loss function of GAN takes minimum
value at $r=1$ and it is approximated by a quadratic function in the vicinity of $r=1$. 
Substituting this approximation into $\frac{\epsilon}{2}C_{\delta,\epsilon,\|p_0\|_\infty}\mathrm{Vol}(1)(\epsilon/5C)^d$, we see that 
$\bar{G}_\epsilon-G(q_0,0)\geq \text{Const}\cdot\epsilon^{d+3}$ holds for a small $\epsilon>0$. 
Hence, we obtain $\|q_\lambda-p_0\|_\infty<\epsilon_\lambda=O(\lambda^{\frac{1}{d+3}})$. 
\end{proof}

\subsection{Proof of Theorem~\ref{thm:estimation_error}}
\label{appendix:Theorem_estimation_error}

\subsubsection*{Smoothness classes}
Let us introduce basic smoothness classes. Details are shown in~\citep{puchkin24:_rates}. 
Let $[m]$ be $\{1,\ldots,m\}$ for $m\in\mathbb{N}$. 
For any $s\in\mathbb{N}$, let us define the function space $C^s$ as 
\begin{align*}
 C^s(\mathcal{X}) = \{f:\mathcal{X}\rightarrow\Rbb^m\,|\,\|f\|_{C^s}:=\max_{|\gamma|\leq s}\|D^\gamma{f}\|_\infty<\infty\}, 
\end{align*}
where the partial differential operator $D^\gamma$ for the multi-index $\gamma=(\gamma_1,\ldots,\gamma_d)\in\mathbb{N}_0^d$
is defined by 
\begin{align*}
 D^\gamma f_i = \frac{\partial^{|\gamma|}f_i}{\partial{x}_1^{\gamma_1}\cdots\partial{x}_d^{\gamma_d}},\ i\in [m], 
 \ \ \text{and}\ \ 
 \|D^\gamma f\|_\infty = \max_{i\in[m]}\|D^\gamma f_i\|_\infty. 
\end{align*}
For any positive number $0<\delta\leq1$, the H\"{o}lder constant of order $\delta$ is defined by 
\begin{align*}
 [f]_\delta = \max_{i\in[m]}\sup_{\substack{x,y\in\mathcal{X}\\ x\neq y}}\frac{|f_i(\x)-f_i(\y)|}{\min\{1,\|\x-\y\|\}^\delta}. 
\end{align*}
For any $\alpha>0$, let $\lfloor\alpha\rfloor=\max\{s\in\mathbb{N}\,|\,s<\alpha\}$. Then, the H\"{o}lder class 
$\mathcal{H}^\alpha(\mathcal{X})$ is defined by 
\begin{align*}
 \mathcal{H}^\alpha(\mathcal{X})=\{
 f\in C^{\lfloor\alpha\rfloor}(\mathcal{X})\,|\,\|f\|_{\mathcal{H}^\alpha}:=\max\{\|f\|_{C^{\lfloor\alpha\rfloor}},
 \max_{|\gamma|=\lfloor\alpha\rfloor} [D^\gamma f]_{\alpha-\lfloor\alpha\rfloor}\}<\infty \}, 
\end{align*}
and the H\"{o}lder ball $\mathcal{H}^\alpha(\mathcal{X},H)$ is defined by 
$\mathcal{H}^\alpha(\mathcal{X},H)=\{f\in\mathcal{H}^\alpha(\mathcal{X})\,|\,\|f\|_{\mathcal{H}^\alpha}\leq H\}$. 
Furthermore, let us define 
the class of $\Lambda$-regular functions $\mathcal{H}_\Lambda^\alpha(\mathcal{X},H), \Lambda>1$ as 
\begin{align*}
\mathcal{H}_\Lambda^\alpha(\mathcal{X},H)=
\left\{
 f\in\mathcal{H}^\alpha(\mathcal{X},H)\,|\,
 \Lambda^{-2}I_{d\times d} \preceq  \nabla f^T \nabla f \preceq \Lambda^{2}I_{d\times d}
 \right\}, 
\end{align*}
where $(\nabla f)_{ij}=\frac{\partial f_i}{\partial x_j}$ for $i\in[m],\,j\in[d]$.

\subsubsection*{Learning with Scale-GAN}

Let $p_0$ be the probability density of the training data on $\mathcal{X}$. 
In this section, we assume $\mathcal{X}=[0,1]^d$. 
Suppose that the distribution of the scale intensity, $\pi$, is the uniform distribution on $[0,1]$
and that the scaling function $s_t, t\in[0,1]$ is a non-increasing function with $s_0=1$ and $s_1>0$. 

 \paragraph{Model.}
 Let $\mathcal{G}$ be a set of the generator $G(\z)\in\mathcal{X}$ for $\z\sim p_z$, where 
 $p_{z}$ is the uniform distribution on $[0,1]^d$. 
 The corresponding set of probability densities is denoted by $\mathcal{Q}$, 
 i.e., the set of the push-forward distribution of $p_z$ by $G\in\mathcal{G}$. 
 The set of scaled generators 
 $(\z,t)\rightarrow (G(\z),t)\in\Rbb^{d+1},\, \z\sim p_z, t\sim \pi$
 is denoted by $\widetilde{\mathcal{G}}$. 
 Let us define $\mathcal{D}$ as a set of discriminators on $\mathcal{X}$ and 
 $\widetilde{\mathcal{D}}$ as a set of discriminators on $\mathcal{X}\times[0,1]$. 
 The relationship between $\mathcal{D}$ and $\widetilde{\mathcal{D}}$ is appropriately determined in the sequel sections.

We apply the statistical analysis of GAN introduced by \cite{puchkin24:_rates} to our problem. 
The estimator we consider is defined as follows. 
Suppose i.i.d. training data $\x_1,\ldots,\x_n\sim p_0$ and i.i.d. scale intensities $t_1,\ldots,t_n\sim \pi$ are observed. 
Let us define $L(q,\widetilde{D})$. 
For the scale function $s_t$, let the function $\kappa_s$ be define $\kappa_s(\x,t)=(s_t\x,t)$. 
for a probability density $q$ and a discriminator $\widetilde{D}:\mathcal{X}\times[0,1]\rightarrow(0,1)$ by 
\begin{align}
L(q,\widetilde{D})
 &:=\int_{[0,1]} \pi(t)\int_{[0,s_t]^d}\left\{s_t^{-d}p_0(s_t^{-1}\x)\log
 \widetilde{D}(\x,t)+s_t^{-d}q(s_t^{-1}\x)\log(1-\widetilde{D}(\x,t))\right\}\rmd{\mu}\rmd{t}\nonumber\\
 &\phantom{:}=\int_{[0,1]} \pi(t)\int_{[0,1]^d}\left\{ p_0(\x)\log
 \widetilde{D}\circ\kappa_s(\x,t)+q(\x)\log(1-\widetilde{D}\circ\kappa_s(\x,t))\right\}\rmd{\mu}\rmd{t}. \label{app:eqn:proper_population_loss}
\end{align}
The empirical approximation of $L(q,\widetilde{D})$ for $q\in\mathcal{Q}$ and $\widetilde{D}\in\widetilde{\mathcal{D}}$ is 
given by 
\begin{align*}
 \widetilde{L}(q,\widetilde{D}):=
  \frac{1}{n}\sum_{i}\log \widetilde{D}\circ\kappa_s(\x_i, t_i)
 +
 \frac{1}{n}\sum_{i}\log(1-\widetilde{D}\circ\kappa_s(G(\z_i), t_i)),\quad \z_1,\ldots,\z_n\sim_{i.i.d.}  p_z. 
\end{align*}
The empirical variance of $\widetilde{D}$ is denoted by 
\begin{align*}
 \widetilde{\mathbb{V}}[\widetilde{D}]
 = 
 \frac{1}{n}\sum_{j=1}^{n}\left[
 \frac{1}{n}\sum_{i}\bigg(\widetilde{D}(s_{t_j}\x_i,t_j)-\frac{1}{n}\sum_{i'}\widetilde{D}(s_{t_{j}}\x_{i'},t_{j})\bigg)^2
 \right] \leq 1. 
\end{align*}
The estimator $\widehat{q}_\lambda\in \mathcal{Q}$ is given by the optimal solution of 
\begin{align*}
 \min_{q\in\mathcal{Q}} \max_{\widetilde{D}\in\widetilde{\mathcal{D}}} 
 \widetilde{L}(q,\widetilde{D})-\lambda \widetilde{\mathbb{V}}[\widetilde{D}]. 
\end{align*}

\subsubsection*{Estimation Error Bound}
The Jensen-Shannon (JS) divergence between the probability density $p$ and $q$ is denoted by 
$\mathrm{JS}(p,q)$, which is defined by 
\begin{align*}
 \mathrm{JS}(p,q)
 =
 \frac{1}{2}\mathbb{E}_{p}\left[\log\frac{2p}{p+q}\right]
 +
 \frac{1}{2}\mathbb{E}_{q}\left[\log\frac{2q}{p+q}\right]. 
\end{align*}
For the true distribution $p_0$ and scaling function $s_t$, let us define 
$D_q(\x)$ and $\widetilde{D}_q(\x,t)$ by 
\begin{align*}
D_q(\x)=\frac{p_0(\x)}{p_0(\x)+q(\x)}, \quad \widetilde{D}_q(\x,t)=\frac{p_0(\x/s_t)}{p_0(\x/s_t)+q(\x/s_t)}. 
\end{align*}
Then, $\widetilde{D}_q(s_t\x,t)=\widetilde{D}_q(\x,0)=D_q(\x)$ holds. 
The following formulae are useful for our analysis, 
\begin{align*}
 \mathrm{JS}(p_0,q)=\mathrm{JS}(p_0\otimes\pi,q\otimes\pi)=L(q,\widetilde{D}_q)+\log2=\max_{\widetilde{D}\in L^\infty}L(q,\widetilde{D})+\log2. 
\end{align*}
\begin{definition}
 Define $\bar{q}, \check{D}_q, \widehat{D}_{q}, \widehat{q}_\lambda$ as follows. 
\begin{align*}
 \bar{q}&=\argmin_{q\in\mathcal{Q}}\mathrm{JS}(p_0,q), \quad
 \check{D}_{q}=\argmax_{\widetilde{D}\in\widetilde{\mathcal{D}}}L(q,\widetilde{D}), \quad 
 \widehat{D}_{q}=
 \argmax_{\widetilde{D}\in\widetilde{\mathcal{D}}}\widetilde{L}(q,\widetilde{D})-\lambda\widetilde{\mathbb{V}}[\widetilde{D}],\\
 \widehat{q}_\lambda &= 
 \argmin_{q\in\mathcal{Q}} \max_{\widetilde{D}\in\widetilde{\mathcal{D}}}\widetilde{L}(q,\widetilde{D})
 -\lambda\widetilde{\mathbb{V}}[\widetilde{D}]
 =\argmin_{q\in\mathcal{Q}}\widetilde{L}(q,\widehat{D}_{q})-\lambda\widetilde{\mathbb{V}}[\widehat{D}_{q}]. 
\end{align*}
\end{definition}
In the above definitions, the dependency of $\widehat{D}_{q}$ on $\lambda$ is omitted.

Let us introduce the following assumptions to derive the estimation error bound of GANs. 
As the discriminator of Scale-GAN has the argument for the scaling intensity, the assumption to $D(\x)$ in \citep{puchkin24:_rates} is replaced with 
that to $D(\x,t)$. 
\begin{assumption}
\label{appendix:assump_Lip}
 Let $u$ be the parameters of $\mathcal{G}$ and $\mathcal{Q}$, and 
 let $\theta$ be the parameter of $\widetilde{\mathcal{D}}$. 
 \begin{description}
  \item[(AG)] For $G_{u_1},G_{u_2}\in\mathcal{G}$, 
	     $\|G_{u_1}(\z)-G_{u_2}(\z)\|\leq L_{\mathcal{G}}\|u_1-u_2\|_\infty$ for any latent variable $\z$. 
	
  \item[(AD)] For the model $\widetilde{\mathcal{D}}$, 
	     \begin{align*}
	      &|\widetilde{D}_{\theta}(\x_1,t_1)-\widetilde{D}_{\theta}(\x_2,t_2)|
	      \leq L\|(\x_1,t_1)-(\x_2,t_2)\|,\\
	      &|\widetilde{D}_{\theta_1}(\x,t)-\widetilde{D}_{\theta_2}(\x,t)|\leq L_{\widetilde{\mathcal{D}}}\|\theta_1-\theta_2\|_\infty. 
	     \end{align*}
	     Furthermore, there exists $D_{\min}>0$ such that 
	     $\widetilde{D}\in [D_{\min}, 1-D_{\min}]$ for any $\widetilde{D}\in\widetilde{\mathcal{D}}$. 
	     
  \item[(Aq)] For $q_{u_1}, q_{u_2}\in\mathcal{Q}$, $|q_{u_1}(\x)-q_{u_2}(\x)|\leq L_{\mathcal{Q}}\|u_1-u_2\|_\infty$. 
 \end{description}
\end{assumption}
From Assumption~(AG), we have 
$\|(G_{u_1}(\z),t)-(G_{u_2}(\z),t)\|\leq L_{\mathcal{G}}\|{u_1}-{u_2}\|_\infty$
for $(G_{u_1}(\z),t),(G_{u_2}(\z),t)\in\widetilde{\mathcal{G}}$. 
Assumption~(AD) leads to 
$\|\widetilde{D}_{\theta_1}-\widetilde{D}_{\theta_2}\|_\infty\leq
L_{\widetilde{\mathcal{D}}}\|\theta_1-\theta_2\|_\infty$. 
We find that 
$\|\widetilde{D}\circ \kappa_s\|_{\mathrm{Lip}}\leq L(2+\|s\|_{\mathrm{Lip}}\sqrt{d})$ for $\widetilde{D}\in\widetilde{\mathcal{D}}$. 

Let us define 
\begin{align*}
\Delta_{\widetilde{\mathcal{G}}}=\min_{q\in\mathcal{Q}}\mathrm{JS}(p_0\otimes\pi,q\otimes\pi), \quad 
 \quad 
 \Delta_{\widetilde{\mathcal{D}}}=\max_{q\in\mathcal{Q}}\min_{\widetilde{D}\in\widetilde{\mathcal{D}}}
\{L(q,\widetilde{D}_q)-L(q,\widetilde{D})\}. 
\end{align*}
Then, we have
$\Delta_{\widetilde{\mathcal{G}}}=\Delta_{\mathcal{G}}=\min_{q\in\mathcal{Q}}\mathrm{JS}(p_0,q)$. 

The following theorem largely depends on Theorem~1 in~\citep{puchkin24:_rates}. 
\begin{theorem}
 \label{appendix:thm:general_bound}
 Assume (AG), (AD), and (Aq). Let 
 $u\in [-1,1]^{d_{\mathcal{G}}}$ and $\theta\in [-1,1]^{d_{\widetilde{\mathcal{D}}}}$. 
 Then, for any $\delta\in(0,1)$, with probability at least $1-\delta$, it holds that 
\begin{align}
\label{appendix:general-bound-JS_est}
 \mathrm{JS}(p_0,\widehat{q}_\lambda)
 &\lesssim 
\Delta_{\mathcal{G}}+\Delta_{\widetilde{\mathcal{D}}}+
 \frac{(d_{\mathcal{G}}+d_{\widetilde{\mathcal{D}}})
 \log(2n(L_{\mathcal{G}} L(2+\|s\|_{\mathrm{Lip}}\sqrt{d})\vee L_{\widetilde{\mathcal{D}}}\vee L_{\mathcal{Q}}\vee 1))}{n}
 +\frac{\log(8/\delta)}{n}
 +\lambda. 
\end{align}
\end{theorem}

\begin{proof}
The estimation accuracy of $\widehat{q}_\lambda$ is assessed by the Jensen-Shannon divergence, 
\begin{align*}
&\phantom{=} \mathrm{JS}(\widehat{q}_{\lambda},p_0)-\mathrm{JS}(\bar{q},p_0)\\
& =
 \mathrm{JS}(\widehat{q}_{\lambda}\otimes\pi,p_0\otimes\pi)-\mathrm{JS}(\bar{q}\otimes\pi,p_0\otimes\pi)\\
&\leq
 \Delta_{\widetilde{\mathcal{D}}}
 +L(\widehat{q}_{\lambda},\check{D}_{\widehat{q}_{\lambda}})
 -L(\bar{q},\check{D}_{\bar{q}})\\
&=
 \Delta_{\widetilde{\mathcal{D}}}
 + 
\underbrace{
 L(\widehat{q}_{\lambda},\check{D}_{\widehat{q}_{\lambda}})
 -
\widetilde{L}(\widehat{q}_{\lambda},\check{D}_{\widehat{q}_{\lambda}})}_{T_1}
 + 
 \underbrace{
 \widetilde{L}(\widehat{q}_{\lambda},\check{D}_{\widehat{q}_{\lambda}})
 -
 \widetilde{L}(\bar{q},\check{D}_{\bar{q}})
 }_{T_2} 
 + 
 \underbrace{
 \widetilde{L}(\bar{q},\check{D}_{\bar{q}})
 -
 L(\bar{q},\check{D}_{\bar{q}})}_{T_3}. 
\end{align*}
Theorem~1 in~\citep{puchkin24:_rates} is applied 
to the expression \eqref{app:eqn:proper_population_loss} 
with the model $\widetilde{\mathcal{G}}$ and 
$\widetilde{\mathcal{D}}\circ\kappa_s=\{\widetilde{D}\circ\kappa_s|\widetilde{D}\in\widetilde{\mathcal{D}}\}$. 
The upper bound of $T_1+T_3$ is given by (5.10) and (5.11) in~\citep{puchkin24:_rates} 
for $\widetilde{\mathcal{G}}$ and $\widetilde{\mathcal{D}}\circ\kappa_s$. 
From the definition of $\widehat{q}_\lambda$, we have 
\begin{align*}
 \widetilde{L}(\widehat{q}_\lambda,\check{D}_{\widehat{q}_\lambda}) -\lambda\widetilde{\mathbb{V}}[\check{D}_{\widehat{q}_\lambda}]
 \leq 
 \widetilde{L}(\widehat{q}_\lambda,\widehat{D}_{\widehat{q}_\lambda}) -\lambda\widetilde{\mathbb{V}}[\widehat{D}_{\widehat{q}_\lambda}]
 \leq 
 \widetilde{L}(\bar{q},\widehat{D}_{\bar{q}}) -\lambda\widetilde{\mathbb{V}}[\widehat{D}_{\bar{q}}]. 
\end{align*}
Hence, $T_2\leq  \widetilde{L}(\bar{q},\widehat{D}_{\bar{q}})-\widetilde{L}(\bar{q},\check{D}_{\bar{q}})+\lambda$. 
Here we use the fact that $\widetilde{\mathbb{V}}[D]\leq 1$ for $0<D<1$. 
The upper bound of 
$\widetilde{L}(\bar{q},\widehat{D}_{\bar{q}})-\widetilde{L}(\bar{q},\check{D}_{\bar{q}})$ 
is given by Eq~(5.12) in~\citep{puchkin24:_rates} for $\widetilde{\mathcal{G}}$ and $\widetilde{\mathcal{D}}\circ\kappa_s$. 
From the definition of the extended models, we have
$d_{\widetilde{\mathcal{G}}}=d_{\mathcal{G}}$ and $d_{\widetilde{\mathcal{D}}\circ\kappa_s}=d_{\widetilde{\mathcal{D}}}$. 
Also, the Lipschitz constants in Assumption~\ref{appendix:assump_Lip} are maintained for the extended models
except $L$. The Lipschitz constant $L$ is changed to $L(2+\sqrt{d}\|s\|_{\mathrm{Lip}})$. 
From the above argument, we obtain the upper bound of $\mathrm{JS}(p_0, \widehat{q}_\lambda)$.  
\end{proof}

The formal statement of Theorem~\ref{thm:estimation_error} is the following. 
\begin{theorem}
 \label{appendix:thm:estimation_error_bound}
 Suppose that Assumption~\ref{appendix:assump_Lip} holds. 
 Let $G_0:[0,1]^d\rightarrow\mathcal{X}$ 
 be the generator of the data distribution $p_0\in\mathcal{H}^{\beta}(\mathcal{X})$ for $\beta>2$. 
 Suppose $G_0\in\mathcal{H}_{\Lambda}^{1+\beta}([0,1]^d, H_0)$ for constants $H_0>0, \Lambda>1$. 
 For the scaling function $s_t$, we assume $1/s\in\mathcal{H}^\alpha([0,1])$, i.e., $\|1/s\|_{\mathcal{H}^\alpha}<\infty$ 
 for an $\alpha>2$. 
 Suppose that $L\|s\|_{\mathrm{Lip}}\lesssim n^{c'}$. 
 Then, there exist DNNs with ReQU activation function such that 
\begin{align*}
 \text{Generator}:\ &\mathcal{G}\subset \mathcal{H}_{\Lambda_{\mathcal{G}}}^{1+\beta}([0,1]^d,H_{\mathcal{G}}),\ 
 \ \text{with}\ H_{\mathcal{G}}\geq 2H_0,\ \Lambda_{\mathcal{G}}\geq 2\Lambda, \ \text{and}\\
\text{Discriminator}:\ 
&\widetilde{\mathcal{D}}\subset \mathcal{H}^{1}(\mathcal{X}\times[0,1],L)\cap\mathcal{U}_{D_{\min}}(\mathcal{X}\times[0,1]), \ 
 0<D_{\min}\leq \Lambda^{-d}/(\Lambda^{-d}+\Lambda^{d}), 
\end{align*}
such that for any $\delta\in(0,1)$, with probability at least $1-\delta$, $\widehat{q}_\lambda$ satisfies the inequality 
\begin{align*}
 \mathrm{JS}(p_0,\widehat{q}_\lambda)
 &\lesssim
 \left[ \max_{q\in\mathcal{Q}}\big\|\frac{p_0}{p_0+q}\big\|_{\mathcal{H}^\beta}-\frac{L}{8\sqrt{d}\|1/s\|_{\mathcal{H}^\alpha}} \right]_+^2
 + \left(\frac{L^2}{\|1/s\|_{\mathcal{H}^\alpha}^2}+c\right)
 \left(\frac{\log{n}}{n}\right)^{\frac{2\beta}{2\beta+d}}\\
 &\phantom{\lesssim} + \frac{\log(1/\delta)}{n} +\lambda, 
\end{align*}
 where 
 $c$ and $c'$ are positive constants depending on the constants $d,\beta,H_0$ and $\Lambda$. 
\end{theorem}
The upper bounds for $\Delta_{\mathcal{G}}$ and $\Delta_{\widetilde{\mathcal{D}}}$ in Theorem~\ref{appendix:thm:general_bound}
 are shown below. 
For $\Delta_{\mathcal{G}}$, we use theoretical analysis in \citep{puchkin24:_rates}. 
When we evaluate $\Delta_{\widetilde{\mathcal{D}}}$, 
naive application of the existing theorem~\citep{puchkin24:_rates} leads the sub-optimal order. 
This is because the input dimension to the discriminator is extended from $d$ to $d+1$. 
In our upper bound, $d+1$ is reduced to $d$. Though it is a minor change, 
but we can obtain the min-max optimal convergence rate for the variance term. 

\begin{remark}
In \citep{belomestny23:_simul}, 
the existence of 
$\mathcal{G}\subset \mathcal{H}_{\Lambda_{\mathcal{G}}}^{2}([0,1]^d,H_{\mathcal{G}})$
is guaranteed to approximate $G_0\in\mathcal{H}_{\Lambda_{\mathcal{G}}}^{1+\beta}([0,1]^d,H_{\mathcal{G}})$. 
As shown at Step~2 in the proof of Theorem~2 in \citep{puchkin24:_rates}, however, 
one can construct $\mathcal{G}$ included in 
$\mathcal{H}_{\Lambda_{\mathcal{G}}}^{1+\beta}([0,1]^d,H_{\mathcal{G}})$.  
Indeed, the DNN with ReQU realizing the normalized spline of degree $q > \beta+1$ meets the condition. 
As a result, $\frac{p_0}{p_0+q}\in\mathcal{H}^\beta(\mathcal{X})$ holds. 
\end{remark}

\begin{proof}
[Proof of Theorem~\ref{appendix:thm:estimation_error_bound}]
For $\Delta_{\mathcal{G}}$, the same argument as in Theorem~2 in~\citep{puchkin24:_rates} holds. 
Suppose that the true generator $G_0$ satisfies $G_0\in\mathcal{H}^{1+\beta}_{\Lambda}([0,1]^d,H_0)$
for $\beta>2, H_0>0$ and $\Lambda>1$. 
Then, there exists a DNN architecture 
$\mathcal{G}\subset\mathcal{H}_{\Lambda_{\mathcal{G}}}^2([0,1]^d,H_{\mathcal{G}}), 
H_{\mathcal{G}}\geq 2H_0, \Lambda_{\mathcal{G}}\geq 2\Lambda$ such that
\begin{align*}
 \Delta_{\mathcal{G}}\leq C_{d,\beta,H_0}\Lambda^{9d}\left(\frac{\log{n}}{n}\right)^{\frac{2\beta}{2\beta+d}},\quad
 d_{\mathcal{G}}\leq C_{d,\beta,H_0}\left(\frac{n}{\log{n}}\right)^{\frac{d}{2\beta+d}}. 
\end{align*}
Let us consider $\Delta_{\widetilde{\mathcal{D}}}$. 
The optimal discriminator $\widetilde{D}_q(\x,t)$ attains the maximum value of 
\begin{align*}
L(q,\widetilde{D}) =\int_{[0,1]^{d+1}} [p_0(\x)\log \widetilde{D}(s_t\x,t)+q(\x)\log(1-\widetilde{D}(s_t\x,t))]\rmd{\mu}\rmd{t}
\end{align*}
At each $(\x, t)$, $\widetilde{D}_q(s_t\x, t)=D_q(\x)$ attains the maximum of the integrand. 
The second order derivative of the concave function 
$v\mapsto p_0(\x)\log v+q(\x)\log(1-v)$ is 
$-\frac{1}{v^2}p_0(\x)-\frac{1}{(1-v)^2}q(\x)$, which is bounded above by $(\|p_0\|_\infty+\|q\|_\infty)/D_{\min}^2$. 
When the latent variable $\z$ for the generator $G\in\mathcal{G}$ is the uniform distribution on $[0,1]^d$, 
Lemma~8 in~\citep{puchkin24:_rates} guarantees that 
$\|q\|_\infty$ for $q\in\mathcal{Q}$ is uniformly bounded above by $\Lambda_{\mathcal{G}}^d$. 
Hence, for any $q\in\mathcal{Q}$, we have
\begin{align}
\label{appendix:bound:Delta_D}
 L(q,\widetilde{D}_q)-L(q,\widetilde{D}) \leq c \sup_{\x\in\mathcal{X},t\in[0,1]}|D_q(\x)-\widetilde{D}(s_t\x,t)|^2,
\end{align}
where $c$ depends only on $D_{\min}, \|p_0\|_\infty$ and $\Lambda^d$. 
To derive the upper bound of $\Delta_{\widetilde{\mathcal{D}}}$, 
we evaluate the right-hand side of the above inequality. 

Let $\mathcal{D}$ be a set of discriminators on $\mathcal{X}$. 
Suppose that $\widetilde{D}\in\widetilde{\mathcal{D}}$ is expressed as the composite function 
$\widetilde{D}(\x,t)=D(\phi_s(\x,t))$ for $D\in\mathcal{D}$ and $\phi_s(\x,t)=\x/s_t$. 
Then, $\widetilde{D}(s_t\x,t)=D(\x)$ holds for $\x\in\mathcal{X}, t\in[0,1]$. 
The function $\phi_s$ is constructed by DNN. 
Let us express $\phi_s$ as the composition of the following two transformations, 
\begin{align*}
 (\x,t) \stackrel{(a)}{\longmapsto} (\x,1/s_t) \stackrel{(b)}{\longmapsto} \x/s_t
\end{align*}
The first map (a) is approximated by DNN. 
Indeed, the identity function $\x\mapsto\x$ is realized by a two-layer NN with ReQU. 
Suppose that $1/s\in\mathcal{H}^\alpha([0,1],H_s)$ for $\alpha>2$. 
Then, due to Theorem~2 in~\citep{belomestny23:_simul}, 
there exists a DNN $\psi(t)$ such that $\psi\in\mathcal{H}^1([0,1],2H_s)$ and 
$\|1/s-\psi\|_\infty \lesssim H_s/K_\psi^\alpha$ for a large $K_\psi\in\mathbb{N}$, 
where $K_\psi$ is the parameter that determines the size of the width for the DNN. 
The second map (b), i.e., the product of two real values, is exactly realized by a two-layer NN with ReQU, 
in which all the parameters $\theta$ satisfy $\|\theta\|_\infty\leq 1$. 
Such DNN meets the condition of the one studied in \citep{belomestny23:_simul,puchkin24:_rates}. 
For the DNN defined by $\psi_{\mathrm{DNN}}(\x,t):=\x\psi(t)$, we have 
\begin{align*}
 \|\phi_s-\psi_{\mathrm{DNN}}\|_\infty = \max_{\x,t}\max_i|x_i/s_t-x_i\psi(t)| \leq
 \|1/s-\psi\|_\infty \lesssim \frac{H_s}{K_\psi^\alpha}. 
\end{align*}
and we can confirm that $\psi_{\mathrm{DNN}}\in\mathcal{H}^1(\mathcal{X}\times[0,1],4H_s)$. 
Since $\phi_s$ is determined from $s_t$, 
$\psi_{\mathrm{DNN}}$ does not need to have learnable parameters. 

Again, Theorem~2 in~\citep{belomestny23:_simul} guarantees that 
there exists a DNN model $\mathcal{D}\subset\mathcal{H}^1(\mathcal{X}, \frac{L}{4\sqrt{d} H_s})$
such that 
\begin{align*}
\inf_{D\in\mathcal{D}} \bigg\|\underbrace{
 \frac{L}{8\sqrt{d} H_s}\frac{D_q}{\|D_q\|_{\mathcal{H}^\beta}}}_{\bar{D}_q}
 -D\bigg\|_\infty
 \lesssim
 \frac{L}{H_s}\frac{1}{K^\beta}
\end{align*}
for a large $K$, since $\bar{D}_q\in\mathcal{H}^\beta(\mathcal{X},\frac{L}{8\sqrt{d}H_s})$. 
In the above, $d$ is absorbed as a constant. 
Let us define $\widetilde{\mathcal{D}}=\{D\circ\psi_{\mathrm{DNN}}\,|\,D\in\mathcal{D}\}$. 
Then, $\widetilde{\mathcal{D}}\subset\mathcal{H}^1(\mathcal{X}\times[0,1], L)$ holds. 
From the above inequalities, we have 
\begin{align*}
 \inf_{\widetilde{D}\in\widetilde{\mathcal{D}}}
 \big\| \bar{D}_q -\widetilde{D}\big\|_\infty
& \lesssim
 \inf_{D\in\mathcal{D}} \|\bar{D}_q-D\|_\infty+
 \|D-D\circ\psi_{\mathrm{DNN}}\|_\infty\\
&\lesssim
 \frac{L}{H_s}
 \frac{1}{K^\beta} 
 +\frac{L}{\sqrt{d}H_s}
 \underbrace{\sup_{\x,t}\max_i|x_i-\psi_{\mathrm{DNN},i}(s_t\x,t)|}_{< \text{\ arbitrary small}}
\lesssim
 \frac{L}{H_s} \frac{1}{K^\beta}. 
\end{align*}
Furthermore, we have 
\begin{align*}
 \big\|D_q-\bar{D}_q \big\|_{\infty}
 =
 \left[ \|D_q\|_{\mathcal{H}^\beta}-\frac{L}{8\sqrt{d}H_s} \right]_+. 
\end{align*}
As a result, we have 
\begin{align}
\label{appendix:bound_tildeD}
 \inf_{\widetilde{D}\in\widetilde{\mathcal{D}}}\|D_q-\widetilde{D}\|_\infty
 \lesssim 
 \left[ \|D_q\|_{\mathcal{H}^\beta}-\frac{L}{8\sqrt{d}H_s} \right]_+
 +\frac{L}{H_s} \frac{1}{K^\beta}.
\end{align}
Combining \eqref{appendix:bound:Delta_D} and \eqref{appendix:bound_tildeD}, we obtain
\begin{align*}
 \Delta_{\widetilde{\mathcal{D}}}
 \lesssim
 \left[ \max_{q\in\mathcal{Q}}\|D_q\|_{\mathcal{H}^\beta}-\frac{L}{8\sqrt{d}\|1/s\|_{\mathcal{H}^\alpha}} \right]_+^2
 +\frac{L^2}{\|1/s\|_{\mathcal{H}^\alpha}^2} \frac{1}{K^{2\beta}}, 
\end{align*}
where $H_s$ is replaced with $\|1/s\|_{\mathcal{H}^\alpha}$. 
Let $d_{\mathcal{D}}$ be the number of parameters for $\mathcal{D}$. Then, 
$d_{\mathcal{D}}=d_{\widetilde{\mathcal{D}}}$ holds from the definition of $\widetilde{\mathcal{D}}$. 
Again, 
we have $d_\mathcal{D}=O((n/\log{n})^{d/(2\beta+d)})$ for $K=(n/\log{n})^{1/(2\beta+d)}$, 
since the dimension of the input vector of $D\in\mathcal{D}$ is $d$. 
Note that $L_{\widetilde{\mathcal{D}}}=L_{\mathcal{D}}$ holds. Lemma~2 and Lemma~4 in~\citep{puchkin24:_rates} guarantees
\begin{align*}
 \log(2n(L_{\mathcal{G}}L(2+\sqrt{d}\|s\|_{{\mathrm{Lip}}})\vee L_{\mathcal{D}}\vee L_{\mathcal{Q}}\vee 1))
 & \lesssim
 \log{n}+\log{L\|s\|_{\mathcal{H}^1}} + \log{\Lambda K} \\
 &\lesssim 
 \log{n}
\end{align*}
under the assumption that $L\|s\|_{\mathcal{H}^1}\lesssim n^{c'}$. 
Substituting all the upper bounds for 
$\Delta_{\mathcal{G}}, \Delta_{\widetilde{\mathcal{D}}}, d_{\mathcal{G}}, d_{\widetilde{\mathcal{D}}}$
and $\log(2n(L_{\mathcal{G}}L(2+\sqrt{d}\|s\|_{\mathrm{Lip}})\vee L_{\mathcal{D}}\vee L_{\mathcal{Q}}\vee 1))$
into \eqref{appendix:general-bound-JS_est} in Theorem~\ref{appendix:thm:general_bound}, 
we obtain the result of Theorem~\ref{appendix:thm:estimation_error_bound}. 

\end{proof}

\subsection{Invertible Data Augmentation}
\label{appendix:Invertible_Data_Augmentation}
The theoretical analysis is similar to that in Section~\ref{appendix:Theorem_estimation_error}. 

\subsubsection*{Estimation Error Bound}
For the invertible data augmentation $\x\mapsto S_t\x$ with the parameter $t\in{T}$, 
we define the map $\kappa_S$ by $\kappa_S(\x,t)=(S_t\x,t)$. 
Let us assume 
$\|\kappa_S(\x,t)-\kappa_S(\x',t')\|\leq \|\kappa_S\|_{\mathrm{Lip}}\|(\x,t)-(\x',t')\|$. 
Then, we have 
$$\|\widetilde{D}\circ \kappa_S\|_{\mathrm{Lip}}\leq L\|\kappa_S\|_{\mathrm{Lip}}.$$
We define 
\begin{align*}
 \Delta_{\widetilde{\mathcal{G}}}=\min_{q\in\mathcal{Q}}\mathrm{JS}(p_0\otimes\pi,q\otimes\pi), \quad 
 \quad 
 \Delta_{\widetilde{\mathcal{D}}}=\max_{q\in\mathcal{Q}}\min_{\widetilde{D}\in\widetilde{\mathcal{D}}}
\{L(q,\widetilde{D}_q)-L(q,\widetilde{D})\}. 
\end{align*}
Then, we have
$\Delta_{\widetilde{\mathcal{G}}}=\Delta_{\mathcal{G}}=\min_{q\in\mathcal{Q}}\mathrm{JS}(p_0,q)$. 
\begin{theorem}
 Assume (AG), (AD), and (Aq) in Assumption~\ref{appendix:assump_Lip}. 
Let $u\in [-1,1]^{d_{\mathcal{G}}}$ and $\theta\in [-1,1]^{d_{\widetilde{\mathcal{D}}}}$. 
 Then, for any $\delta\in(0,1)$, with probability at least $1-\delta$, it holds that 
\begin{align*}
 \mathrm{JS}(p_0,\widehat{q}_\lambda)
 &\lesssim 
 \Delta_{\mathcal{G}}+\Delta_{\widetilde{\mathcal{D}}}+
 \frac{(d_{\mathcal{G}}+d_{\widetilde{\mathcal{D}}})
 \log(2n(L_{\mathcal{G}} L\|\kappa_S\|_{\mathrm{Lip}}\vee L_{\widetilde{\mathcal{D}}}\vee L_{\mathcal{Q}}\vee 1))}{n}
 +\frac{\log(8/\delta)}{n}
 +\lambda. 
\end{align*}
\end{theorem}
\begin{proof}
The estimation accuracy of $\widehat{q}_\lambda$ is assessed by the Jensen-Shannon divergence, 
\begin{align*}
&\phantom{=} \mathrm{JS}(\widehat{q}_{\lambda},p_0)-\mathrm{JS}(\bar{q},p_0)\\
& =
 \mathrm{JS}(\widehat{q}_{\lambda}\otimes\pi,p_0\otimes\pi)-\mathrm{JS}(\bar{q}\otimes\pi,p_0\otimes\pi)\\
&\leq
 \mathrm{JS}(\widehat{q}_{\lambda}\otimes\pi,p_0\otimes\pi)-\log2-L(\bar{q},\check{D}_{\bar{q}})\\
&=
 L(\widehat{q}_{\lambda},\widetilde{D}_{\widehat{q}})
 -L(\widehat{q}_{\lambda},\check{D}_{\widehat{q}_{\lambda}})
 +L(\widehat{q}_{\lambda},\check{D}_{\widehat{q}_{\lambda}}) -L(\bar{q},\check{D}_{\bar{q}})\\
&=
 \min_{\widetilde{D}\in\widetilde{\mathcal{D}}}
 \{L(\widehat{q}_{\lambda},\widetilde{D}_{\widehat{q}_\lambda}) -L(\widehat{q}_{\lambda},\widetilde{D})\}
 +L(\widehat{q}_{\lambda},\check{D}_{\widehat{q}_{\lambda}}) -L(\bar{q},\check{D}_{\bar{q}})\\
&\leq
 \Delta_{\widetilde{\mathcal{D}}}
 +L(\widehat{q}_{\lambda},\check{D}_{\widehat{q}_{\lambda}})
 -L(\bar{q},\check{D}_{\bar{q}})\\
&=
 \Delta_{\widetilde{\mathcal{D}}}
 + 
\underbrace{
 L(\widehat{q}_{\lambda},\check{D}_{\widehat{q}_{\lambda}})
 -
 L_n(\widehat{q}_{\lambda},\check{D}_{\widehat{q}_{\lambda}})}_{T_1}
 + 
 \underbrace{
 L_n(\widehat{q}_{\lambda},\check{D}_{\widehat{q}_{\lambda}})
 -
 L_n(\bar{q},\check{D}_{\bar{q}})
 }_{T_2} 
 + 
 \underbrace{
L_n(\bar{q},\check{D}_{\bar{q}})
 -
 L(\bar{q},\check{D}_{\bar{q}})}_{T_3}. 
\end{align*}
Theorem~1 in~\cite{puchkin24:_rates} is applied 
to the expression \eqref{app:eqn:proper_population_loss} 
with the model $\widetilde{\mathcal{G}}$ and 
$\widetilde{\mathcal{D}}\circ\kappa_S=\{\widetilde{D}\circ\kappa_S|\widetilde{D}\in\widetilde{\mathcal{D}}\}$. 
The upper bound of $T_1+T_3$ is given by (5.10) and (5.11) in~\cite{puchkin24:_rates} 
for $\widetilde{\mathcal{G}}$ and $\widetilde{\mathcal{D}}\circ\kappa_S$. 
From the definition of $\widehat{q}_\lambda$, we have 
\begin{align*}
 L_n(\widehat{q}_\lambda,\check{D}_{\widehat{q}_\lambda}) -\lambda{\mathbb{V}}_n[\check{D}_{\widehat{q}_\lambda}]
 \leq 
 L_n(\widehat{q}_\lambda,\widehat{D}_{\widehat{q}_\lambda}) -\lambda{\mathbb{V}}_n[\widehat{D}_{\widehat{q}_\lambda}]
 \leq 
 L_n(\bar{q},\widehat{D}_{\bar{q}}) -\lambda{\mathbb{V}}_n[\widehat{D}_{\bar{q}}]. 
\end{align*}
Hence, $T_2\leq  L_n(\bar{q},\widehat{D}_{\bar{q}})-L_n(\bar{q},\check{D}_{\bar{q}})+\lambda$. 
Here we use the fact that $\mathbb{V}_n[D]\leq1$ for $0<D<1$. 
The upper bound of $L_n(\bar{q},\widehat{D}_{\bar{q}})-L_n(\bar{q},\check{D}_{\bar{q}})$
is given by Eq~(5.12) in~\cite{puchkin24:_rates} for $\widetilde{\mathcal{G}}$ and
 $\widetilde{\mathcal{D}}\circ\kappa_S$. 
From the definition of the extended models, we have 
$d_{\widetilde{\mathcal{G}}}=d_{\mathcal{G}}$ and $d_{\widetilde{\mathcal{D}}\circ\kappa_S}=d_{\widetilde{\mathcal{D}}}$. 
Also, the Lipschitz constants in Assumption~\ref{appendix:assump_Lip} except $L$ are maintained for the
 extended models, while $L$ is replaced with $L\|\kappa_S\|_{\mathrm{Lip}}$.  
From the above argument, we obtain the upper bound of $\mathrm{JS}(p_0,\widehat{q}_\lambda)$. 
\end{proof}

\begin{theorem}
\label{theorem:convergence_rate_GAN_DA}
 Suppose that Assumption~\ref{appendix:assump_Lip} holds. 
 Let $G_0:[0,1]^d\rightarrow\mathcal{X}$ 
 be the generator of the data distribution $p_0\in\mathcal{H}^{\beta}(\mathcal{X})$ for a $\beta>2$. 
 Suppose $G_0\in\mathcal{H}_{\Lambda}^{1+\beta}([0,1]^d, H_0)$ for constants $H_0>0, \Lambda>1$. 
 For the DA $S_t, t\in{T}$, 
 we assume 
 that the function $\phi_{S^{-1}}(\x,t)=S_t^{-1}\x$ satisfies
 $\phi_{S^{-1}}\in\mathcal{H}^\alpha(\mathcal{X}\times T)$, i.e., $\|\phi_{S^{-1}}\|_{\mathcal{H}^\alpha}<\infty$ 
 for an $\alpha>2$. 
 Suppose that the parameter $t$ of the DA has the probability density $\pi(t)$ on a subset of Euclidean space $T$. 
 Suppose that the function $\kappa_{S}(\x,t)=(S_t\x,t)$ has a finite Lipschitz constant
 for the $\ell_2$-norm, i.e., 
 $\|\kappa_{S}(\x,t)-\kappa_{S}(\x',t')\|\leq \|\kappa_{S}\|_{\mathrm{Lip}}\|(\x,t)-(\x',t')\|$. 
 Furthermore, suppose that 
 a positive constant $L$ and $\kappa_S$ satisfies 
 $L\|\kappa_S\|_{\mathrm{Lip}}\lesssim n^{c}$, where $c$ is a constant. 
 Then, there exist DNNs, $\mathcal{G}$ and $\mathcal{D}$, with ReQU activation function such that 
\begin{align*}
\text{Generator}:\ &\mathcal{G}\subset \mathcal{H}_{\Lambda_{\mathcal{G}}}^{2}([0,1]^d,H_{\mathcal{G}}),\ 
 \ \text{with}\ H_{\mathcal{G}}\geq 2H_0,\ \Lambda_{\mathcal{G}}\geq 2\Lambda, \ \text{and}\\
\text{Discriminator}:\ 
&\widetilde{\mathcal{D}}\subset \mathcal{H}^{1}(\mathcal{X}\times[0,1],L)\cap\mathcal{U}_{D_{\min}}(\mathcal{X}\times[0,1]), \ 
 0<D_{\min}\leq \Lambda^{-d}/(\Lambda^{-d}+\Lambda^{d}), 
\end{align*}
such that for any $\delta\in(0,1)$, with probability at least $1-\delta$, $\widehat{q}_\lambda$ satisfies the inequality 
\begin{align*}
 \mathrm{JS}(p_0,\widehat{q}_\lambda)
& \lesssim
 \left[ \max_{q\in\mathcal{Q}}\|D_q\|_{\mathcal{H}^\beta}-\frac{L}{4\sqrt{d}(1+\|\phi_{S^{-1}}\|_{\mathcal{H}^\alpha})} \right]_+^2
 + \left\{\left(\frac{L}{1+\|\phi_{S^{-1}}\|_{\mathcal{H}^\alpha}}\right)^2+c_0+\frac{\log(L\|\kappa_S\|_{\mathrm{Lip}})}{\log{n}}\right\}
 \left(\frac{\log{n}}{n}\right)^{\frac{2\beta}{2\beta+d}}\\
&\phantom{\lesssim} + \frac{\log(1/\delta)}{n} +\lambda\\
& \lesssim
 \left[ \max_{q\in\mathcal{Q}}\|D_q\|_{\mathcal{H}^\beta}-\frac{L}{4\sqrt{d}(1+\|\phi_{S^{-1}}\|_{\mathcal{H}^\alpha})} \right]_+^2
 + \left\{\left(\frac{L}{1+\|\phi_{S^{-1}}\|_{\mathcal{H}^\alpha}}\right)^2+c_0\right\}
 \left(\frac{\log{n}}{n}\right)^{\frac{2\beta}{2\beta+d}} + \frac{\log(1/\delta)}{n} +\lambda, 
\end{align*}
where $c_0$ is a constant depending only on $d,\beta,H_0,\Lambda$, and $c$.  
\end{theorem}
The convergence rate of the variance term in $\mathrm{JS}(p_0,\widehat{q}_\lambda)$ 
is the same as the min-max optimal rate for $p_0\in\mathcal{H}^\beta(\mathcal{X},H)$. 
Note that we assume that the true distribution is realized by the generator $G_0$, meaning that
the class of true probability densities is slightly restricted 
in the set of probability densities in $p_0\in\mathcal{H}^\beta(\mathcal{X},H)$. 
\begin{remark}
 The H\"{o}lder norm of the map $\phi_{S^{-1}}(\x,t)=S_t^{-1}\x$ controls the bias-variance trade off. 
 \begin{itemize}
 \item $\|\phi_{S^{-1}}\|_{\mathcal{H}^\alpha}$ is large: the bias is large and variance is small. 
 \item $\|\phi_{S^{-1}}\|_{\mathcal{H}^\alpha}$ is small: the bias is small and variance is large.
\end{itemize}
\end{remark}
The upper bounds of $\Delta_{\mathcal{G}}$ and $\Delta_{\widetilde{\mathcal{D}}}$ are shown below. 
For $\Delta_{\mathcal{G}}$, we use theoretical analysis in \cite{puchkin24:_rates}. 
When we evaluate $\Delta_{\widetilde{\mathcal{D}}}$, 
naive application of the existing theorem~\cite{puchkin24:_rates} leads the sub-optimal order. 
This is because the input dimension to the discriminator is extended from $d$ to $d+1$. 
In our upper bound, $d+1$ is reduced to $d$. Though it is a minor change, 
but we can obtain a \emph{possible} min-max optimal convergence rate for the variance term. 

\begin{proof}
[Proof of Theorem~\ref{theorem:convergence_rate_GAN_DA}]
For $\Delta_{\mathcal{G}}$, the same argument as in Theorem~2 of~\cite{puchkin24:_rates} holds. 
Suppose that the true generator $G_0$ satisfies $G_0\in\mathcal{H}^{1+\beta}_{\Lambda}([0,1]^d,H_0)$
for $\beta>2, H_0>0$ and $\Lambda>1$. 
Then, there exists a DNN architecture $\mathcal{G}$ depending on $n,H_0,\beta,d$ such that
\begin{align*}
 \Delta_{\mathcal{G}}\leq C_{d,\beta,H_0}\Lambda^{9d}\left(\frac{\log{n}}{n}\right)^{\frac{2\beta}{2\beta+d}},\quad
 d_{\mathcal{G}}\leq C_{d,\beta,H_0}\left(\frac{n}{\log{n}}\right)^{\frac{d}{2\beta+d}}. 
\end{align*}

Let us consider $\Delta_{\widetilde{\mathcal{D}}}$. 
The optimal discriminator $\widetilde{D}_q(\x,t)$ attains the maximum value of 
\begin{align*}
L(q,\widetilde{D}) =
 \int_{\mathcal{X}\times{T}}
 \pi(t)
 [p_0(\x)\log \widetilde{D}\circ\kappa_S(\x,t)+q(\x)\log(1-\widetilde{D}\circ\kappa_S(\x,t))]\rmd{\mu}\rmd{t}
\end{align*}
At each $(\x, t)$, $\widetilde{D}_q\circ\kappa_S(\x,t)=D_q(\x)$ attains the maximum of the integrand. 
The second order derivative of the concave function 
$v\mapsto p_0(\x)\log v+q(\x)\log(1-v)$ is 
$-\frac{1}{v^2}p_0(\x)-\frac{1}{(1-v)^2}q(\x)$, which is bounded above by $(\|p_0\|_\infty+\|q\|_\infty)/D_{\min}^2$. 
When the latent variable $\z$ for the generator $G\in\mathcal{G}$ is the uniform distribution on $[0,1]^d$, 
Lemma~8 of~\cite{puchkin24:_rates} guarantees that 
$\|q\|_\infty$ for $q\in\mathcal{Q}$ is uniformly bounded above by $\Lambda_{\mathcal{G}}^d$. 
Hence, for any $q\in\mathcal{Q}$, we have
\begin{align*}
 L(q,\widetilde{D}_q)-L(q,\widetilde{D}) \leq c \sup_{\x\in\mathcal{X},t\in{T}}|D_q(\x)-\widetilde{D}\circ\kappa_S(\x,t)|^2,
\end{align*}
where $c$ depends on $D_{\min}, \|p_0\|_\infty, \Lambda^d$. 
To derive the upper bound of $\Delta_{\widetilde{\mathcal{D}}}$, 
we evaluate the right-hand side of the above inequality. 

Let $\mathcal{D}$ be a set of discriminators on $\mathcal{X}$. 
Suppose that $\widetilde{D}\in\widetilde{\mathcal{D}}$ is expressed as the composite function 
$\widetilde{D}(\x,t)=D(\phi_{S^{-1}}(\x,t))$ for $D\in\mathcal{D}$ and $\phi_{S^{-1}}(\x,t):=S_t^{-1}\x$. 
Then, $\widetilde{D}\circ\kappa_S(\x,t)=\widetilde{D}(S_t\x,t)=D(\x)$ holds for $\x\in\mathcal{X}, t\in{T}$. 
The function $\phi_{S^{-1}}$ is approximated by DNN. 
Suppose that $\phi_{S^{-1}}\in \mathcal{H}^{\alpha}(\mathcal{X}\times{T}, H_S)$ with $H_S\geq 1$. 
Then, 
due to Theorem~2 in~\cite{belomestny23:_simul}, 
there exists a DNN $\psi_{\mathrm{DNN}}$ such that $\psi_{\mathrm{DNN}}\in\mathcal{H}^1(\mathcal{X}\times{T},2H_S)$ and 
\begin{align*}
 \max_{i\in[d]}\|\phi_{S^{-1},i}-\psi_{\mathrm{DNN},i}\|_\infty \lesssim H_S/K_0^\alpha
\end{align*}
for a large $K_0\in\mathbb{N}$, which determines the width of DNN to be $O((K_0+\alpha)^d)$. 
Since $\phi_{S^{-1}}$ is determined from the prespecified DA, 
the DNN, $\psi_{\mathrm{DNN}}$, does not need to have learnable parameters. 

Note that for 
$D\in\mathcal{H}^1(\mathcal{X}, \frac{L}{2\sqrt{d}H_S})$ and 
$\psi_{\mathrm{DNN}} \in \mathcal{H}^1(\mathcal{X}\times{T}, 2H_S)$, we have
$D\circ\psi_{\mathrm{DNN}}\in\mathcal{H}^1(\mathcal{X}\times{T}, L)$. 
Indeed, 
\begin{align*}
 |D\circ\psi_{\mathrm{DNN}}(\x,t)-D\circ\psi_{\mathrm{DNN}}(\x',t')|
&\leq  
 \frac{L}{2\sqrt{d}H_S}\cdot 1\wedge\|\psi_{\mathrm{DNN}}(\x,t)-\psi_{\mathrm{DNN}}(\x',t')\|\\
&\leq  
 \frac{L}{2H_S}\cdot 1\wedge\|\psi_{\mathrm{DNN}}(\x,t)-\psi_{\mathrm{DNN}}(\x',t')\|_\infty\\
&\leq   
 \frac{L}{2H_S}\cdot 1\wedge 2H_S(1\wedge \|(\x,t)-(\x',t')\|)\\
&\leq   
 \frac{L}{2H_S}\cdot (2H_S)\wedge 2H_S(1\wedge \|(\x,t)-(\x',t')\|)\\
&= L \cdot1\wedge\|(\x,t)-(\x',t')\|. 
\end{align*}
Furthermore, $\|D\circ\psi_{\mathrm{DNN}}\|_\infty\leq \|D\|_{\infty}\leq \frac{L}{2\sqrt{d}H_S}\leq L$ holds. 
Hence, we have $D\circ\psi_{\mathrm{DNN}}\in\mathcal{H}^1(\mathcal{X}\times{T}, L)$. 
Again, Theorem~2 in~\cite{belomestny23:_simul} guarantees that 
there exists a DNN model $\mathcal{D}\subset\mathcal{H}^1(\mathcal{X}, \frac{L}{2\sqrt{d} H_S})$ such that 
\begin{align*}
\inf_{D\in\mathcal{D}} \bigg\|\underbrace{
 \frac{L}{4\sqrt{d} H_S}\frac{D_q}{\|D_q\|_{\mathcal{H}^\beta}}}_{\bar{D}_q}
 -D\bigg\|_\infty
 \lesssim
 \frac{L}{H_S}\frac{1}{K^\beta}
\end{align*}
for a large $K$, since $\bar{D}_q\in\mathcal{H}^\beta(\mathcal{X},\frac{L}{4\sqrt{d}H_S})$. 
In the above, $\frac{1}{4\sqrt{d}}$ is absorbed as a constant in the inequality. 
Let us define $\widetilde{\mathcal{D}}=\{D\circ\psi_{\mathrm{DNN}}\,|\,D\in\mathcal{D}\}$. 
Then, $\widetilde{\mathcal{D}}\subset\mathcal{H}^1(\mathcal{X}\times{T}, L)$ holds. 
From the above inequalities, we have 
\begin{align*}
 \inf_{\widetilde{D}\in\widetilde{\mathcal{D}}}
 \big\| \bar{D}_q -\widetilde{D}\big\|_\infty
& \lesssim
 \inf_{D\in\mathcal{D}} \|\bar{D}_q-D\|_\infty
 +
  \|D-D\circ\psi_{\mathrm{DNN}}\|_\infty\\
&\lesssim
 \frac{L}{H_S}
 \frac{1}{K^\beta} 
 +\frac{L}{\sqrt{d}H_S}
 \max_{i\in[d]}\sup_{\x,t}
 \underbrace{
 |\phi_{S^{-1},i}(S_t\x,t)-\psi_{\mathrm{DNN},i}(S_t\x,t)|}_{< \text{\ arbitrary small}
 }
\lesssim
 \frac{L}{H_S} \frac{1}{K^\beta}. 
\end{align*}
Furthermore, we have 
\begin{align*}
 \big\|D_q-\bar{D}_q \big\|_{\infty}
 =
 \left[ \|D_q\|_{\mathcal{H}^\beta}-\frac{L}{4\sqrt{d}H_S} \right]_+. 
\end{align*}
As a result, we have 
\begin{align*}
 \inf_{\widetilde{D}\in\widetilde{\mathcal{D}}}\|D_q-\widetilde{D}\|_\infty
 \lesssim 
 \left[ \|D_q\|_{\mathcal{H}^\beta}-\frac{L}{4\sqrt{d}H_S} \right]_+
 +\frac{L}{H_S} \frac{1}{K^\beta}.
\end{align*}
Combining all the inequalities above, we obtain
\begin{align*}
 \Delta_{\widetilde{\mathcal{D}}}
 \lesssim
 \left[ \max_{q\in\mathcal{Q}}\|D_q\|_{\mathcal{H}^\beta}-\frac{L}{4\sqrt{d}(1+\|\phi_{S^{-1}}\|_{\mathcal{H}^\alpha})} \right]_+^2
 +\frac{L^2}{(1+\|\phi_{S^{-1}}\|_{\mathcal{H}^\alpha})^2} \frac{1}{K^{2\beta}}, 
\end{align*}
where $H_S$ is replaced with $1+\|\phi_{S^{-1}}\|_{\mathcal{H}^\alpha}$. 
Let $d_{\mathcal{D}}$ be the number of parameters for $\mathcal{D}$. Then, 
$d_{\mathcal{D}}=d_{\widetilde{\mathcal{D}}}$ holds from the definition of $\widetilde{\mathcal{D}}$. 
Again, 
we have $d_\mathcal{D}=O((n/\log{n})^{d/(2\beta+d)})$ for $K=(n/\log{n})^{1/(2\beta+d)}$, 
since the dimension of the input vector of $D\in\mathcal{D}$ is $d$. 
Note that $L_{\widetilde{\mathcal{D}}}=L_{\mathcal{D}}$ holds. Lemma~2 and Lemma~4 of~\cite{puchkin24:_rates} guarantees
\begin{align*}
 \log(2n(L_{\mathcal{G}}L\|\kappa_S\|_{\mathrm{Lip}}\vee L_{\mathcal{D}}\vee L_{\mathcal{Q}}\vee 1))
 \lesssim \log\Lambda{n}
 \lesssim \log\Lambda\log{n}. 
\end{align*}
Substituting all ingredients, we obtain the inequality of the theorem. 
\end{proof}

\subsubsection*{Composition of Multiple DAs}
Next, let us consider the usage of multiple DAs. 
Let $S_{i,t_i}:\mathcal{X}\rightarrow \mathcal{X}, t_i\in{T_i}, i\in[B]$ be $B$ DAs.
We assume that $S_{i,t_i}\mathcal{X}\subset\mathcal{X}$ and 
$S_{i,t_i}::\mathcal{X}\rightarrow S_{i,t_i}\mathcal{X}$ is invertible for any $i\in[B]$ and $t_i\in{T_i}$. 
Let $\e_i, i\in[B]$ be the canonical basis in $\Rbb^B$, e.g., $\e_1=(1,0\ldots,0)\in\Rbb^B$, etc, 
and $E=\{\e_1,\ldots,\e_B\}\subset\Rbb^B$. 
Let $\Delta_B$ be $\{\bm{a}=(a_1,\ldots,a_B)\in\Rbb^B\,|\,\bm{0}\leq\bm{a},\sum_{i=1}^{B}a_i\leq 1\}$. 
For the collection of invertible DAs, $S=\{S_i\}_{i\in[B]}$, let us define 
the map $\kappa_S$ for $(\x,\bm{t},\bm{a})\in\mathcal{X}\times\prod_{i=1}^B{T_i}\times\Delta_B$ 
for $\t=(t_1,\ldots,t_B)\in{\prod_{i=1}^B{T_i}}$ and $\bm{a}=(a_1,\ldots,a_B)$ by 
\begin{align*}
\kappa_S(\x,\bm{t},\bm{a})=(\sum_{i}a_iS_{i,t_i}\x,\bm{t},\bm{a})
\in \mathcal{X}\times\prod_{i=1}^B{T_i}\times\Delta_B. 
\end{align*}
Furthermore, the map $\phi_{S^{-1}}$ on $\mathcal{X}\times\prod_{i\in[B]}T_i\times\Delta_B$ is defined by 
\begin{align*}
 \phi_{S^{-1}}(\x,\bm{t}, \bm{a})=\sum_{i}a_i S_{i,t_i}^{-1}\x. 
\end{align*}
From the definition, $\phi_{S^{-1}}\circ \kappa_S(\x,\bm{t},\e)=\x$ holds for any
$(\x,\bm{t},\e)\in\mathcal{X}\times\prod_{i}T_i\times{E}$. 
For the extended sample $\widetilde{\x}=\kappa_S(\x,\bm{t},\e_i)$, one can identify the DA, $S_{i,t}$, operated to $\x$. 

We construct a DNN that approximates $\phi_{S^{-1}}$. 
The Lipschitz constant of $\widetilde{D}\circ\kappa_S(\x,\bm{t},\bm{a})$ for the extended discriminator
$\widetilde{D}$ such that $\|\widetilde{D}\|_{\mathrm{Lip}}\leq L$
is given as follows: 
\begin{align*}
 \|\kappa_S(\x,\bm{t},\bm{a})-\kappa_S(\x',\bm{t}',\bm{a}')\|
& \leq
 \|\sum_{i}a_iS_{i,t_i}\x-\sum_{i}a_i' S_{i,t_i'}\x'\|+\|(\t,\bm{a})-(\t',\bm{a}')\|\\
& \leq
 \|\sum_{i}(a_i-a_i')S_{i,t_i}\x\|+\|\sum_{i}a_i' (S_{i,t_i'}\x'-S_{i,t_i}\x)\|+\|(\t,\bm{a})-(\t',\bm{a}')\|\\
& \leq
 \underbrace{\sum_{i}|a_i-a_i'|}_{\leq \sqrt{B}\|\bm{a}-\bm{a}'\|}\underbrace{\|S_{i,t_i}\x\|}_{\leq\sqrt{d}}+\sum_{i}|a_i'|\|S_{i,t_i'}\x'-S_{i,t_i}\x\|+\|(\t,\bm{a})-(\t',\bm{a}')\|\\
& \leq
 \sqrt{Bd}\|\bm{a}-\bm{a}'\|+\sum_{i}|a_i'|\|\kappa_{S_i}\|_{\mathrm{Lip}}\|(\x,\t)-(\x',\t')\|+\|(\t,\bm{a})-(\t',\bm{a}')\|\\
& \leq
 \sqrt{Bd}\|\bm{a}-\bm{a}'\|+\max_i\|\kappa_{S_i}\|_{\mathrm{Lip}}\|(\x,\t)-(\x',\t')\|+\|(\t,\bm{a})-(\t',\bm{a}')\|\\
& \leq
 (\sqrt{Bd}+1 + \max_i\|\kappa_{S_i}\|_{\mathrm{Lip}})\|(\x,\t,\bm{a})-(\x',\t',\bm{a}')\|. 
\end{align*}
Note that for $\mathcal{X}=[0,1]^d$, 
j$\|S_{i,t_i}\x\|\leq \sqrt{d}$\footnote{$\sqrt{d}$ in the inequality can be replaced with the radius of $\mathcal{X}$ in $\ell_2$-norm.}. 
Hence, we have $\|\widetilde{D}\circ\kappa_S\|_{\mathrm{Lip}}\leq L(\sqrt{Bd}+1+\max_{i\in[B]}\|\kappa_{S_i}\|_{\mathrm{Lip}})$. 

Suppose $\phi_{S_i^{-1}}\in \mathcal{H}^{\alpha_i}(\mathcal{X}\times T_i, H_{i})$ for $i\in[B]$. 
Each $\phi_{S_i^{-1}}$ is approximated by a DNN $\psi_i\in\mathcal{H}^1(\mathcal{X}\times T_i,2H_i)$
in the sense that $\|\phi_{S_i^{-1}}-\psi_i\|_\infty\lesssim H_i/K_i^{\alpha_i}$. 
Let us consider the DNN $\psi$ defined by 
\begin{align*}
\psi:(\x,\t,\bm{a})\mapsto (\psi_1(\x,t_1),\ldots,\psi_B(\x,t_B),\bm{a}) \mapsto \sum_{i=1}^{B}a_i\psi_i(\x,t_i). 
\end{align*}
Since the addition and multiplication are exactly expressed by ReQU-based DNN, the second operation is exactly realized by DNN with ReQU. 
\begin{align*}
 \|\phi_{S^{-1}}(\x,\t,\bm{a})-\psi(\x,\t,\bm{a})\|_\infty
 \leq \sum_{i}|a_i|\|\phi_{S_i^{-1}}-\psi_i\|_\infty \lesssim \max_i \frac{H_i}{K_i^{\alpha_i}}. 
\end{align*}
The above upper bound can be arbitrary small when $\min_i K_i$ is sufficiently large. 
The $\mathcal{H}^1$ norm of $\psi$ is evaluated as follows: 
\begin{align*}
 \|\psi(\x,\t,\bm{a})-\psi(\x',\t',\bm{a}')\|_\infty
 &\leq
 \sum_{i}|a_i-a_i'|\!\!\!\!\!\!\!\!\!\!\!\!\!\!\!\!
 \underbrace{\|\psi_i(\x,t_i)\|_\infty}_{\qquad\qquad \leq 1\ \text{since $\psi_i(\x,t_i)\in\mathcal{X}$}}
 \!\!\!\!\!\!\!\!\!\!\!\!\!\!\!\!\!\!
 +
 \sum_{i}|a_i'|\|\psi_i(\x,t_i)-\psi_i(\x',t_i')\|_\infty\\
 &\leq
 \|\bm{a}-\bm{a}'\|_1
 +
 \sum_{i}|a_i'|\|\psi_i(\x,t_i)-\psi_i(\x',t_i')\|_\infty\\
 &\leq
 2\sqrt{B}\cdot 1\wedge\|\bm{a}-\bm{a}'\| + 2\max_i H_i \cdot 1\wedge\|(\x,t_i)-(\x',t_i')\|\\
 &\leq
2(\sqrt{B} + \max_iH_i)\cdot 1\wedge\|(\x,\t,\bm{a})-(\x',\t',\bm{a}')\|.
\end{align*}
In the above we use the inequality 
$\|\bm{a}-\bm{a}'\|_1\leq 2(1\wedge \|\bm{a}-\bm{a}'\|_1)\leq 2(1\wedge\sqrt{B}\|\bm{a}-\bm{a}'\|)\leq 2\sqrt{B}(1\wedge\|\bm{a}-\bm{a}'\|)$ 
for $\bm{a},\bm{a}'\in\{\bm{b}:\bm{b}\geq0,\sum_ib_i\leq 1\}$. 
Hence, $\psi\in\mathcal{H}^1(2H_S)$ with $H_S=\sqrt{B}+\max_{i\in[B]}H_i$. 
One can set $H_i=\|\phi_{S_i^{-1}}\|_{\mathcal{H}^{\alpha_i}}$. 

For $\widetilde{D}(\x,\t,\bm{a})=D\circ\psi(\x,\t,\bm{a})$ with 
$D\in\mathcal{H}^{1}(\mathcal{X},\frac{L}{2\sqrt{d}(\sqrt{B}+\max_i H_i)})$ agrees to 
$\widetilde{D}\in\mathcal{H}^{1}(\mathcal{X}\times\prod_i T_i\times \Delta_B, L)$
as proved in Theorem~\ref{theorem:convergence_rate_GAN_DA}.

Eventually, for $L(\sqrt{Bd}+1+\max_{i\in[B]}\|\kappa_{S_i}\|_{\mathrm{Lip}})\lesssim n^c$ we obtain
\begin{align*}
 \mathrm{JS}(p_0,\widehat{q}_\lambda)
& \lesssim
 \left[ \max_{q\in\mathcal{Q}}\|D_q\|_{\mathcal{H}^\beta}-\frac{L}{4\sqrt{d}(\sqrt{B}+\max_{i\in[B]}\|\phi_{S_i^{-1}}\|_{\mathcal{H}^{\alpha_i}})} \right]_+^2\\
 &+ \left\{\frac{L^2}{
(\sqrt{B}+\max_{i\in[B]}\|\phi_{S_i^{-1}}\|_{\mathcal{H}^{\alpha_i}})^2
 }+c+\frac{\log(L(\sqrt{Bd}+1+\max_{{i\in[B]}}\|\kappa_{S_i}\|_{\mathrm{Lip}})) }{\log{n}}\right\}
 \left(\frac{\log{n}}{n}\right)^{\frac{2\beta}{2\beta+d}}\\
&\phantom{\lesssim} + \frac{\log(1/\delta)}{n} +\lambda\\
& \lesssim
 \left[ \max_{q\in\mathcal{Q}}\|D_q\|_{\mathcal{H}^\beta}-\frac{L}{4\sqrt{d}(\sqrt{B}+\max_{i\in[B]}\|\phi_{S_i^{-1}}\|_{\mathcal{H}^{\alpha_i}})} \right]_+^2\\
 &+ \left\{\frac{L^2}{
(\sqrt{B}+\max_{i\in[B]}\|\phi_{S_i^{-1}}\|_{\mathcal{H}^{\alpha_i}})^2}+c\right\}
 \left(\frac{\log{n}}{n}\right)^{\frac{2\beta}{2\beta+d}}
 + \frac{\log(1/\delta)}{n} +\lambda
\end{align*}

\bibliographystyle{apalike}

\end{document}